\documentclass[11pt]{article}
\usepackage[cm]{fullpage}
\usepackage{natbib}
\usepackage{amsmath,amssymb,amsthm,mathtools}
\usepackage{booktabs,tabularx,array,graphicx,float,placeins,microtype,xcolor,enumitem,longtable}
\usepackage{tikz}
\usetikzlibrary{arrows.meta,positioning,calc}
\usepackage{hyperref,url}
\definecolor{readblue}{HTML}{2459A9}
\definecolor{readteal}{HTML}{16877A}
\definecolor{softblue}{HTML}{EAF0FA}
\definecolor{softteal}{HTML}{E8F5F2}
\definecolor{softgray}{HTML}{F2F3F5}
\hypersetup{colorlinks=true,linkcolor=readblue,citecolor=readblue,urlcolor=readblue,
pdftitle={ReadKV: Query-Adaptive Quantization for KV Caches},
pdfauthor={Siddharth Bhandari, Lucas Gretta, Krishna Balasubramanian,
Shiva Kasiviswanathan}}
\newcommand{\ReadKV}{\textsc{ReadKV}}
\newcommand{\siblings}{\mathcal S}
\newtheorem{theorem}{Theorem}
\newtheorem{corollary}[theorem]{Corollary}
\newtheorem{proposition}[theorem]{Proposition}
\newtheorem{lemma}[theorem]{Lemma}
\floatstyle{ruled}
\newfloat{algorithm}{tbp}{loa}
\floatname{algorithm}{Algorithm}
\title{Read What Matters: Query-Adaptive Quantization for KV Caches}
\author{
Siddharth Bhandari\textsuperscript{1} \quad
Lucas Gretta\textsuperscript{2} \quad
Krishna Balasubramanian\textsuperscript{1} \quad
Shiva Kasiviswanathan\textsuperscript{1}\\[0.4em]
\textsuperscript{1}Amazon \qquad
\textsuperscript{2}University of California, Berkeley
}
\date{}
\begin{document}
\maketitle
\begin{abstract}
KV-cache entries are stored before their future queries are known, but each
decoding query needs precision in different places. We study this mismatch
using separate budgets for retained bits and bits fetched per query.
\ReadKV{} stores each key and value in a progressive code whose prefixes
support different reconstruction precisions. For each query, it allocates
key-channel prefixes using the query, computes attention from the
reconstructed keys, and then allocates value-token prefixes using that
attention. Stored entries remain unchanged. Each stage optimizes a calibrated
distortion objective under a fixed budget; we prove exact allocation under
diminishing refinement gains and relate these objectives to attention-output
error. We also exhibit a finite-dimensional attention family where query-dependent
access strictly outperforms every query-independent reader at the same read
budget, even with unrestricted competing encoders and decoders.

Across six base models, reading four bits on average from an eight-bit cache
increases C4 perplexity by at most 0.66\%, using about one quarter of the
logical reads and half the retained capacity of a 16-bit cache. It is
consistently more accurate than storing and fully reading four bits at the
same payload-read budget. Retaining more bits than each query fetches is
aimed at long-context decoding, where the cache bytes moved per step, rather
than the weights, dominate cost. Long-context question answering and retrieval on two
instruction-tuned models provide additional quality evidence.
On the tested 8K-token, batch-one, single-layer workload on an NVIDIA A10G,
a restricted eight-bit \ReadKV{} reader with a two-bit mean payload-read
budget has 39\% lower latency than the tested TurboQuant codec.
\end{abstract}

\section{Introduction}
Long-context autoregressive language models maintain a key--value (KV) cache
to avoid recomputing representations of earlier tokens. During decoding,
each new query accesses this cache, and repeated reads can incur substantial
memory traffic as the context grows. However, a query may depend primarily
on a subset of key dimensions, and its attention may concentrate on a small
fraction of value tokens. These observations motivate query-dependent reads
that fetch the cache information needed at each decoding step.

We propose \ReadKV{}, which separates stored bits per scalar, $W$, from
average payload bits read per scalar, $R$. Its \emph{progressive
representation} lets each query choose how much precision to fetch while
preserving unread detail for later queries. This design is motivated by the
write-once, read-many structure of the KV cache. Retaining additional
precision increases storage and insertion cost, but does not require every
later query to fetch those extra bits. For an 8,192-token prompt followed by
128 decoding steps, storing eight bits while reading four, rather than storing
and reading four, adds write traffic equal to just $0.79\%$ of cumulative
reads in the payload-only model of Appendix~\ref{app:traffic}.
Stored precision and read volume
therefore impose distinct costs. Decoding is often bound by memory traffic: each step moves the model weights once and
each cached entry once per request. The bits fetched per entry thus set the
cache traffic of a step, while the bits retained per entry set how many
requests fit in memory and are served per step. A fixed-rate code couples the two,
because every stored bit is fetched on every step; a progressive code
decouples them, and at $(8,2)$ \ReadKV{}'s logical reads are about one
quarter of its retained bits. The design therefore targets long contexts and
large batches, where the cache rather than the weights dominates the bytes
moved per step (Appendix~\ref{app:throughput_model}). Device latency also
depends on selection, reconstruction, and temporary memory traffic.

Existing work reduces retained KV precision, varies stored precision using
importance estimates, or selects cache features and tokens for the current
query. \ReadKV{} instead makes scalar reconstruction depth a query-time
access decision while leaving the stored record unchanged. The closest
systems precedent already combines precision tiers with selective bit-plane
retrieval \citep{xie2025reimagining}. Section~\ref{sec:related_work}
distinguishes these approaches and Appendix~\ref{app:related} gives a more
detailed comparison.

\begin{figure}[t]
\centering
\begingroup
\definecolor{motivationink}{HTML}{243449}
\definecolor{motivationmuted}{HTML}{536174}
\definecolor{motivationrule}{HTML}{C4CEDA}
\definecolor{motivationbit}{HTML}{DDE4EC}
\begin{tikzpicture}[
  x=1cm,y=1cm,
  font=\sffamily\fontsize{8.5}{10}\selectfont,
  text=motivationink,
  >={Latex[length=1.8mm,width=1.35mm]},
  every node/.style={inner sep=0pt,outer sep=0pt},
  heading/.style={font=\sffamily\bfseries\fontsize{9.5}{11}\selectfont},
  note/.style={font=\sffamily\fontsize{8}{10}\selectfont,
               text=motivationmuted,align=center},
  entry/.style={draw=motivationrule,fill=softgray,rounded corners=2pt,
                minimum height=.58cm,align=center},
  bit/.style={draw=motivationrule,fill=motivationbit,line width=.3pt},
  flow/.style={line width=.85pt,rounded corners=2pt},
  panel/.style={rounded corners=3pt,line width=.55pt}
]
\path[use as bounding box] (0,.30) rectangle (13.70,5.35);

\draw[panel,draw=motivationrule,fill=softgray!60]
  (0,2.57) rectangle (13.70,4.35);
\draw[panel,draw=readblue!25,fill=softblue!45]
  (0,.35) rectangle (13.70,2.30);

\node[entry,minimum width=2.20cm,minimum height=.62cm,
      font=\sffamily\fontsize{8}{9}\selectfont]
  (token) at (1.20,4.70) {KV entry for token \(i\)\\\((k_i,v_i)\)};
\draw[->,flow,motivationink,shorten <=2pt,shorten >=2pt]
  (token.east) -- (3.12,4.70);
\node[heading] at (4.02,4.70) {Written once};
\node[note] at (9.43,5.16) {Up to \(T-i\) later reads};
\node[heading] at (7.46,4.70) {Read at \(q_{i+1}\)};
\node[heading] at (11.41,4.70) {Read at \(q_T\)};

\node[heading,anchor=west] at (.22,3.66) {Fixed code};
\node[note,anchor=west,align=left,text width=1.55cm] at (.22,3.04)
  {same bits at\\every query};
\node[heading,anchor=west,text=readblue] at (.22,1.68) {\ReadKV{}};
\node[note,anchor=west,align=left,text=readblue,text width=1.55cm]
  at (.22,.93) {query-adaptive\\reads};

\foreach \x/\y in {2.75/4.14,2.75/1.92}{
  \foreach \c in {0,...,5}{
    \foreach \r in {0,...,3}{
      \draw[bit] (\x+\c*.28,\y-\r*.28) rectangle +(.22,-.22);
    }
  }
}

\foreach \x in {6.75,10.70}{
  \foreach \c in {0,...,5}{
    \foreach \r in {0,...,3}{
      \path[fill=motivationink!78,draw=motivationink!78,line width=.3pt]
        (\x+\c*.28,4.14-\r*.28) rectangle +(.22,-.22);
    }
  }
}

\foreach \x in {6.75,10.70}{
  \foreach \c in {0,...,5}{
    \foreach \r in {0,...,3}{
      \draw[bit] (\x+\c*.28,1.92-\r*.28) rectangle +(.22,-.22);
    }
  }
}

\foreach \c/\depth in {0/1,1/4,2/2,3/1,4/3,5/1}{
  \foreach \r in {1,...,\depth}{
      \path[fill=readblue,draw=readblue,line width=.3pt]
      (6.75+\c*.28,2.20-\r*.28) rectangle +(.22,-.22);
  }
}
\foreach \c/\depth in {0/3,1/1,2/1,3/4,4/1,5/2}{
  \foreach \r in {1,...,\depth}{
      \path[fill=readteal,draw=readteal,line width=.3pt]
      (10.70+\c*.28,2.20-\r*.28) rectangle +(.22,-.22);
  }
}

\draw[->,flow,motivationink,shorten <=2pt,shorten >=2pt]
  (4.53,3.71) -- node[note,above=2pt] {read all bits} (6.64,3.71);
\draw[->,flow,readblue,shorten <=2pt,shorten >=2pt]
  (4.53,1.49) -- node[note,text=readblue,above=2pt] {read selectively}
  (6.64,1.49);
\node[font=\fontsize{15}{17}\selectfont,text=motivationmuted]
  at (9.57,3.71) {\(\cdots\)};
\node[font=\fontsize{15}{17}\selectfont,text=motivationmuted]
  at (9.57,1.49) {\(\cdots\)};
\draw[->,flow,motivationink,shorten <=2pt,shorten >=2pt]
  (9.92,3.71) -- (10.59,3.71);
\draw[->,flow,readteal,shorten <=2pt,shorten >=2pt]
  (9.92,1.49) -- (10.59,1.49);

\node[note] at (3.46,2.82) {\(W\) bits/scalar stored};
\node[note] at (7.46,2.82) {\(W\) bits/scalar read};
\node[note] at (11.41,2.82) {\(W\) bits/scalar read};
\node[note,text=readblue] at (3.46,.58) {\(W\) bits/scalar stored};
\node[note,text=readblue] at (7.46,.58) {\(R<W\) bits/scalar read};
\node[note,text=readteal!85!black] at (11.41,.58)
  {\(R<W\) bits/scalar read};
\end{tikzpicture}
\endgroup
\caption{Write-once, read-many KV-cache access. Token \(i\)'s key and value
vectors are stored once and may be read by up to \(T-i\) later queries.
Fixed-precision quantization fetches the same full \(W\)-bit code each time.
\ReadKV{} stores a progressive \(W\)-bit code and lets each query fetch
different prefixes under an average read budget \(R<W\); unread refinements
remain available.}
\label{fig:kv-partial-read-problem}
\vspace{0em}
\end{figure}
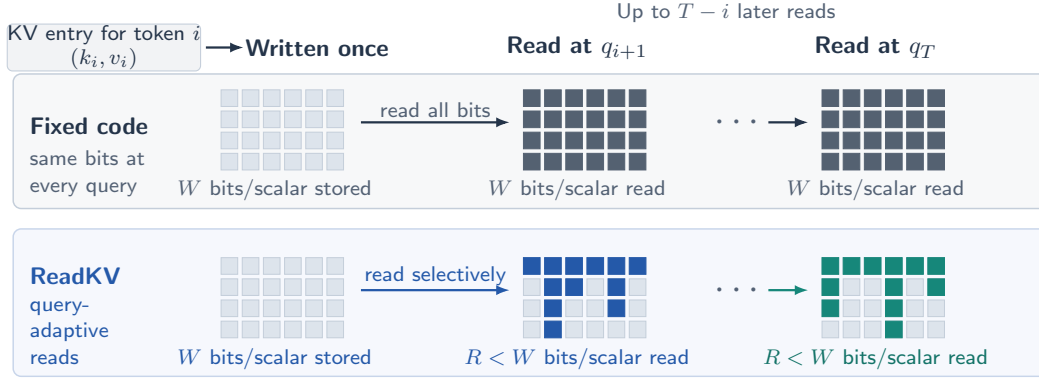

We formalize this distinction through separate storage and read budgets,
as illustrated in
Figure~\ref{fig:kv-partial-read-problem}. We focus on the partial-read regime
$R<W$, where access decisions determine which information is fetched;
$R=W$ permits full-record decoding. Our main contributions are as follows.
\begin{itemize}[leftmargin=0.15in,itemsep=0pt,parsep=0pt,
                topsep=2pt,partopsep=0pt]
\item We formulate KV-cache compression with separate storage and read
  budgets. A cache entry is encoded before its future queries are known,
  and each query may fetch only a budgeted subset of the stored bits.
  This distinguishes retained capacity from transferred payload and enables
  direct comparisons between query-independent and query-dependent access.

\item We introduce \ReadKV{} (Algorithm~\ref{alg:readkv}), which stores
  progressive key and value codes. The current query selects key prefixes;
  attention computed from their reconstructions selects value prefixes.
  Different queries can therefore use different precisions from the same
  stored entries without rewriting them.

\item We analyze why query-dependent access can help and how \ReadKV{}
  allocates its read budget (Section~\ref{sec:analysis}). On two
  finite-dimensional families, \ReadKV{} achieves strictly lower worst-case
  key-score and attention-output MSE, respectively, than every
  query-independent branching reader with the same read budget
  (Theorems~\ref{thm:finite_readkv_separation}
  and~\ref{thm:finite_two_stage_readkv_separation}). We then analyze
  \ReadKV{}'s allocation rule, which repeatedly assigns the next prefix bit
  to the key channel or value token for which it gives the largest predicted
  error reduction. Theorem~\ref{thm:readkv_attention_certificate} identifies
  when this rule is optimal for each stage and relates the resulting key and
  value objectives to the final attention-output error. Together, these
  results explain both why query-dependent reading can help and how
  \ReadKV{} uses its read budget.

\item Across six base models, \ReadKV{} $(W=8,R=4)$ stays within $0.66\%$
  of uncompressed perplexity and has lower perplexity than fully reading
  a four-bit cache at the same payload-read budget. Question answering and retrieval
  on two instruction-tuned models show further quality--read tradeoffs.
  In a separate single-layer benchmark on an NVIDIA A10G, \ReadKV{}
  $(W=8,R=2)$ has $39\%$ lower latency and $62$--$65\%$ fewer DRAM read
  bytes than a fused quantized codec.
\end{itemize}

\section{Partial-Read Cache Model}
\label{sec:general_problem}
A cached key must be stored before the queries that will later use it are
known. At each decoding step, the reader sees a new query and must decide
which parts of that stored key to fetch. We ask how accurately it can recover
the query--key score when it may read fewer bits than the encoder retained.
The following model separates these two budgets without committing to a
particular quantizer.

Let $\mathbb B_2^d=\{u\in\mathbb R^d:\|u\|_2\leq1\}$ be the Euclidean unit
ball. The encoder receives a vector $x\in\mathbb B_2^d$, and a later query
supplies $y\in\mathbb B_2^d$. The target is the inner product $x^\top y$.
There are three roles:
the \emph{encoder} observes $x$, but not $y$, and writes a bit string $Z$;
the \emph{reader} receives $y$ and requests selected bits of $Z$; and the
\emph{decoder} uses $y$ and the returned bits to estimate $x^\top y$.
The reader has no other access to $x$.

Protocols may be randomized. The three roles may share public randomness
drawn independently of $(x,y)$. Private randomness used by the encoder can
reach later stages only through $Z$, and randomness drawn after the query
arrives is also independent of $(x,y)$. These conventions allow randomized
codes and access rules without giving the protocol side information about
the particular input pair.

We assign separate rates to storage and reading. For $W,R\geq0$, the encoder
stores at most $w=\lfloor Wd\rfloor$ bits, and each query may fetch at most
$r=\lfloor Rd\rfloor$ bits. For example, $W=8$ and $R=2$ allow an $8d$-bit
record but only $2d$ fetched bits per query. These are global budgets rather
than per-coordinate quotas. The encoder may mix information across
coordinates, and the reader may spend more probes on some coordinates than
others. In the attention interpretation, $x$ is a cached key, $y$ is the
current query, and $x^\top y$ is their unscaled attention score. This
single-score model isolates the access question. Section~\ref{sec:readkv}
applies the same storage--read distinction to the full KV cache and its
attention output.

We count the bits stored for each record and the bits fetched for each
query. Computation is unrestricted, as in the one-bit-cell version of the
static table-probe model \citep{yao1981tables}. Record-specific metadata
counts toward the corresponding budget. Code descriptions and other state
fixed independently of $(x,y)$ are shared across records rather than charged
once per record. Appendix~\ref{app:traffic} accounts separately for their
retained memory and reload traffic in the implementation.

We distinguish readers by the information they may use to select which
stored bits to fetch. An \emph{address} identifies one bit of the record
$Z$, while a \emph{batch} is a set of addresses chosen before any bit in
that set is returned. The reader classes below differ in whether these
choices may depend on the current query, previously returned bits, or both.

Table~\ref{tab:access_models} orders five reader classes by the information
available when addresses are selected. A query-oblivious fixed pattern uses
neither the query nor returned bits. A query-independent branching reader
may use each returned bit to choose its next address, but its access path and
stopping rule cannot use $y$. Since the returned bits encode $x$, this path
may still depend indirectly on the stored vector. A $\mathsf{staged}(1)$
reader uses $y$ to choose all addresses in one batch. More generally, a
$\mathsf{staged}(J)$ reader, for $J\geq2$, uses at most $J$ batches and may
choose each later batch using $y$ and the values returned by earlier batches.
A fully branching reader may select each new address after observing every
previous return. Every class fetches at most $r$ bits, and its final decoder
may use $y$.

\begin{table}[!ht]
\centering
\small
\caption{Reader classes distinguished by whether address selection may use
the current query or previously returned bits. All classes use the same
storage and read budgets, and every final decoder may use the query.}
\label{tab:access_models}
\begin{tabularx}{\linewidth}{@{}l c X@{}}
\toprule
Reader Class & Uses Query? & Use of Returned Bits in Address Selection \\
\midrule
Query-oblivious fixed pattern & No & None; every address is fixed in advance \\
Query-independent branching & No & Each new address may use all previously returned bits \\
$\mathsf{staged}(1)$: one batch & Yes & None; the complete batch is selected before reading \\
$\mathsf{staged}(J)$, $J\geq2$: at most $J$ batches & Yes & A later batch may use bits returned by earlier batches \\
Fully branching & Yes & Each new address may use all previously returned bits \\
\bottomrule
\end{tabularx}
\end{table}

For an access class $\mathsf A$, let $\Pi$ range over complete
encoder--reader--decoder protocols that obey its access rules and the
budgets $(W,R)$. If $\widehat s_\Pi(x,y)$ is the resulting estimate, the
full-unit-ball minimax risk is
\begin{equation}
  \mathcal D^{\mathrm{wc}}_{d,\mathsf A}(W,R)
  =
  \inf_{\Pi\in\mathsf A}
  \sup_{x,y\in\mathbb B_2^d}
  \mathbb E_\Pi\!\left[
    \bigl(x^\top y-\widehat s_\Pi(x,y)\bigr)^2
  \right].
  \label{eq:worst_case_risk}
\end{equation}
The input pair is fixed before the random coins are drawn, so the expectation
is only over protocol randomness. Equation~\eqref{eq:worst_case_risk} is used
to compare access classes and state the boundary cases. The separation
results in Section~\ref{sec:analysis} use the same loss on explicit restricted
input families. Appendix~\ref{app:formal_access} gives the complete access
recursions, hierarchy, and boundary proofs.

Section~\ref{sec:readkv} develops \ReadKV{} as a
$\mathsf{staged}(2)$ reader: the query selects key prefixes in the first
batch, and attention computed from the reconstructed keys selects value
prefixes in the second. Section~\ref{sec:analysis} analyzes its allocation
rule and gives separations on specified key-score and attention families.

\section{Query-Adaptive Cache Design}
\label{sec:readkv}
This section defines the \ReadKV{} encoder, its $\mathsf{staged}(2)$ reader,
and decoder (Algorithm~\ref{alg:readkv}). The encoder stores each key and
value in a progressive code whose longer prefixes have lower calibrated
distortion. At query time,
the reader issues two read batches from this unchanged cache. The current
query determines the key prefixes fetched in the first batch. An intermediate
decoder reconstructs the keys and computes attention, which determines the
value prefixes fetched in the second batch. The final decoder reconstructs
the values and returns the attention output.
Figure~\ref{fig:readkv-two-reads} summarizes this query-time flow.

\begin{figure}[!ht]
\vspace{0.4em}
\centering
\caption{Query-time flow of \ReadKV{}. The query selects key-prefix depths
in the first batch. After the intermediate decoder reconstructs the keys,
the resulting attention weights select value-prefix depths in the second
batch. Both batches read from the same unchanged progressive $W$-bit cache,
and each batch's addresses are fixed before any of its bits are returned.}
\label{fig:readkv-two-reads}
\begin{tikzpicture}[
  font=\sffamily\fontsize{8.5}{10}\selectfont,
  node distance=4.7mm,
  >={Latex[length=1.7mm,width=1.3mm]},
  box/.style={draw=readblue!35,fill=softblue,rounded corners=2pt,
              minimum height=1.30cm,align=center,inner sep=4pt},
  valuebox/.style={box,draw=readteal!45,fill=softteal},
  flow/.style={->,line width=.8pt,draw=readblue!80!black},
  note/.style={font=\sffamily\fontsize{8}{9.5}\selectfont,align=center}
]
\node[box,text width=1.00cm] (q) {Current\\query};
\node[box,text width=2.25cm,right=of q] (k)
  {Choose key depths\\[-1pt]\(\downarrow\)\\[-1pt]Fetch key prefixes};
\node[box,text width=2.55cm,right=of k] (a)
  {Reconstruct keys\\[-1pt]\(\downarrow\)\\[-1pt]Compute\\attention weights};
\node[valuebox,text width=2.25cm,right=of a] (v)
  {Choose value depths\\[-1pt]\(\downarrow\)\\[-1pt]Fetch value prefixes};
\node[valuebox,text width=2.55cm,right=of v] (o)
  {Reconstruct values\\[-1pt]\(\downarrow\)\\[-1pt]Compute output};
\draw[flow] (q) -- (k);
\draw[flow] (k) -- (a);
\draw[flow,draw=readteal!80!black] (a) -- (v);
\draw[flow,draw=readteal!80!black] (v) -- (o);
\node[note,text=readblue,above=2.5mm of k] {Key read batch};
\node[note,text=readteal!80!black,above=2.5mm of v] {Value read batch};
\end{tikzpicture}
\end{figure}
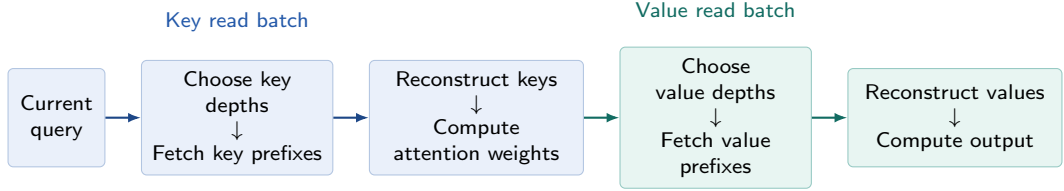

\subsection{Progressive Cache Representation}

A conventional fixed-rate scalar code is designed to be decoded as a whole.
Instead, \ReadKV{} stores a path in a binary quantization tree. If the path
has eight bits, its first two bits already specify a reconstruction;
reading a third bit selects a finer reconstruction, and so on up to eight.
Thus a later query can read a different prefix of the same unchanged code.
We denote the stored depth by the integer $W$ and use a Gaussian Lloyd tree
\citep{lloyd1982quantization}; Appendix~\ref{app:codec} gives its construction.

Consider one attention layer with $T$ cached tokens and vector dimension
$d_h$. It has $G$ physical KV heads: each is one cached stream of keys
$k_{g,j}\in\mathbb R^{d_h}$ and values $v_{g,j}\in\mathbb R^{d_h}$, indexed
by head $g$ and token $j$. The layer also has $H$ query heads, with
$q_h\in\mathbb R^{d_h}$. Several query heads may share the same KV stream.
We write $g(h)$ for the KV head used by $q_h$ and
$\siblings(g)=\{h:g(h)=g\}$ for the query heads sharing it. A fetched cache
bit is counted once even when several query heads share it, so its importance
combines their contributions.

Before encoding, we choose fixed orthogonal Hadamard--sign transforms
$U_K,U_V\in\mathbb R^{d_h\times d_h}$. The signs are drawn once before
encoding, and the same transforms are reused for all cache entries and
queries. These transforms mix coordinates before scalar quantization.
Applying $U_K$ to both a key and its query preserves their dot product;
$U_V$ is inverted after computing the value-weighted output. A channel below
means one coordinate in these transformed vectors.
For each transformed scalar $x'$, the encoder uses a calibrated channel mean
$\mu$ and scale $\sigma$ to form $z=(x'-\mu)/\sigma$, and stores its
$W$-bit tree path. A decoded prefix gives $\widehat z$ and hence
$\mu+\sigma\widehat z$. Reading zero bits gives $\widehat z=0$, the
calibrated mean. The transforms and calibration remain fixed as queries
change.

\subsection{Two-Stage Query-Adaptive Reading}

\noindent\textbf{Key-Channel Read.}
A query weights key channels through a dot product. A reconstruction error
in a channel with a larger query coefficient has a larger effect on the
score. Write $q'_h=U_Kq_h$, and let $\sigma^K_{g,c}$ be the calibrated scale
of transformed key channel $c$ in KV head $g$. Squaring the coefficient of
a standardized channel error and adding contributions from the sharing
query heads gives the importance
\begin{equation}
 a^K_{g,c}=\sum_{h\in\siblings(g)}(q'_{h,c}\sigma^K_{g,c})^2.
 \label{eq:key_importance}
\end{equation}
The reader assigns a prefix depth $t^K_{g,c}\in\{0,\ldots,W\}$ to each
head--channel pair. That depth is used across all $T$ cached tokens, because
the same query coefficient multiplies that channel in every key.
All key-prefix addresses are chosen from the query and fixed calibration
before any key bits are fetched.

The intermediate decoder reconstructs
$\widehat K_{g,j}\in\mathbb R^{d_h}$ from the returned prefixes. It then
forms the attention vector
$\widehat\alpha_h=(\widehat\alpha_{h,1},\ldots,\widehat\alpha_{h,T})
\in\mathbb R^T$ by applying softmax across tokens to the reconstructed logits
$q_h'^\top\widehat K_{g(h),j}/\sqrt{d_h}$. These probabilities determine
the second read batch; neither the decoder nor the reader uses full-precision
keys or their unread refinements.

\noindent\textbf{Value-Token Read.}
The attention output is a weighted sum of values. An error in a value token
with larger attention weight therefore matters more to that output.
\ReadKV{} uses the squared reconstructed probabilities to form
\begin{equation}
 a^V_{g,j}(\widehat\alpha)
 =s_g\sum_{h\in\siblings(g)}\widehat\alpha_{h,j}^{\,2}.
 \label{eq:value_importance}
\end{equation}
Here $s_g\geq0$ sets the value scale for a physical head. The reported
experiments use $s_g=1$, measuring value distortion in standardized units;
Appendix~\ref{app:codec} describes the physical-scale alternative.
The reader chooses one depth $t^V_{g,j}\in\{0,\ldots,W\}$ per value token,
shared across its $d_h$ channels. All value-prefix addresses are fixed
before this second batch is fetched. The final decoder reconstructs
$\widehat V_{g,j}\in\mathbb R^{d_h}$ from the returned prefixes and returns
$U_V^\top\sum_j\widehat\alpha_{h,j}\widehat V_{g(h),j}$ for query head $h$.
There is one adaptive transition between the batches: returned key bits
change attention, which changes the value read.

\subsection{Prefix Allocation Under a Read Budget}

An importance score alone does not specify a depth: an extra bit can help a
coarse reconstruction more than an already precise one. Let $D(t)$ be the
stored, calibrated distortion curve predicting standardized scalar
mean-squared reconstruction error after $t$ bits. For importance weights $a_i$, the
allocator uses the objective $\sum_i a_iD(t_i)$. Each term combines how much
an item's error matters with how much error its chosen prefix leaves.
These separable objectives omit cross-terms between reconstruction errors;
they do not assume those interactions vanish in real caches.

The two stages have fixed nonnegative integer budgets $B_K$ and $B_V$.
These count group refinements, not individual payload bits: increasing
one key-channel depth by one fetches $T$ bits, and increasing one value-token
depth by one fetches $d_h$ bits. Within either stage, with budget $B$, the
allocation problem is
\[
 \min_{\substack{t_i\in\{0,\ldots,W\}\\ \sum_i t_i\leq B}}
 \sum_i a_iD(t_i).
\]

Define $\operatorname{ChooseDepths}(a,D,W,B)$ by starting at $t_i=0$ and
repeatedly increasing an eligible depth $t_i<W$ with the largest positive
weighted gain $a_i[D(t_i)-D(t_i+1)]$. Stop after $B$ increments or when no
positive gain remains, breaking ties in a fixed order.
Section~\ref{sec:analysis} analyzes this rule, including its exact stagewise
optimality under diminishing refinement gains and its relation to
attention-output error.

The key stage calls this rule with $(a^K,B_K)$, and the value stage with
$(a^V(\widehat\alpha),B_V)$. Their total payload is at most
$TB_K+d_hB_V$ bits. Writing the budgeted mean depths as
$R_K=B_K/(Gd_h)$ and $R_V=B_V/(GT)$ gives the mean payload read budget
\[
 R=\frac{TB_K+d_hB_V}{2GTd_h}=\frac12(R_K+R_V).
\]
The K/V split is an input to the allocator; it optimizes the allocation
within each stage, not that split. Actual reads may be smaller if no
positive refinement gains remain.

\begin{algorithm}[!ht]
\caption{\ReadKV{} encoder, two-stage reader, and decoder.}
\label{alg:readkv}
\small
\noindent\emph{Public inputs.}
Stored depth $W$; transforms $U_K,U_V$; calibrated means and scales;
distortion curve $D(t)$; and integer refinement budgets $B_K,B_V$.

\smallskip\noindent
\emph{Encoder.}
For each new token, transform and standardize every key and value scalar.
Encode each scalar as a $W$-bit path through the progressive quantization
tree and append the resulting bit planes to the cache.

\smallskip\noindent
\emph{Prefix allocator
$\operatorname{ChooseDepths}(a,D,W,B)$.}
An item is a transformed key channel in the first stage or a cached value
token in the second. Its weight $a_i$ measures its importance to the current
query, and $t_i$ is its selected prefix depth.
\begin{enumerate}[leftmargin=1.5em,itemsep=1pt,topsep=1pt]
\item Initialize every depth to zero.
\item At each of at most $B$ steps, compute
$g_i=a_i[D(t_i)-D(t_i+1)]$ for every item with $t_i<W$.
\item Stop if every available gain is nonpositive. Otherwise increment an
item with the largest gain, breaking ties by a fixed public order.
\item Return the depth vector $t$.
\end{enumerate}

\smallskip\noindent
\emph{Key reader and intermediate decoder.}
Transform each query as $q'_h=U_Kq_h$ and compute the key-channel weights
$a^K$ using Equation~\eqref{eq:key_importance}. Set
$t^K=\operatorname{ChooseDepths}(a^K,D,W,B_K)$ and fetch the corresponding
key prefixes. Reconstruct $\widehat K$, compute the masked logits, and apply
softmax to obtain $\widehat\alpha$.

\smallskip\noindent
\emph{Value reader and final decoder.}
Compute the value-token weights $a^V(\widehat\alpha)$ using
Equation~\eqref{eq:value_importance}. Set
$t^V=\operatorname{ChooseDepths}(a^V(\widehat\alpha),D,W,B_V)$ and fetch the
corresponding value prefixes. Reconstruct $\widehat V$ and return
$U_V^\top\sum_j\widehat\alpha_{h,j}\widehat V_{g(h),j}$ for each query head
$h$.
\end{algorithm}

Algorithm~\ref{alg:readkv} summarizes the procedure. With $N$ items, a
max-heap implementation uses $O(N)$ memory and $O(N+B\log N)$ time; the key
and value stages have $N=Gd_h$ and $N=GT$, respectively.
Appendix~\ref{app:rate_accounting} specifies integer-budget conversion and
tie handling.

\section{Performance Guarantees for Partial Reads}
\label{sec:analysis}

This section addresses three questions about query-adaptive reading. First,
can a query-dependent key reader outperform every reader whose access path is
independent of the query? Second, does this advantage persist through the
complete key-then-value attention computation? Third, on an arbitrary fixed
cache and query, what does the allocation rule in
Algorithm~\ref{alg:readkv} optimize, and how do its objectives relate to the
final attention output? The first two results give explicit finite
separations. The third gives an exact stagewise allocation guarantee and a
deterministic attention-error certificate.
Appendix~\ref{app:readkv_analysis_proofs} contains the complete proofs and
supporting lemmas.

\subsection{Key-Score Separation}
\label{sec:readkv_separation}

We first ask whether the query can improve accuracy simply by choosing which
parts of a stored key to read. The comparison is with a
\emph{query-independent branching} reader. Such a reader may use returned
bits to choose later addresses, but neither its addresses nor its stopping
rule may depend on the query. Its encoder and decoder may otherwise be
arbitrary, and the decoder may use the query together with the complete read
transcript. We write $\widehat s_\Pi(x,y)$ for its estimate.

Fix a public orthogonal transform $U\in\mathbb R^{d\times d}$ and an integer
$1\leq k\leq d$. Consider the source and query families
\[
 \mathcal X_U=\{U^\top z/\sqrt d:z\in[-1,1]^d\},
 \qquad
 \mathcal Y_{U,k}
 =\{U^\top q:\|q\|_2\leq1,\ \|q\|_0\leq k\}.
\]
Both families are contained in $\mathbb B_2^d$. The query is sparse only
after applying the public transform. For example, with a Hadamard--sign
transform, a sparse transformed query is generally dense in the original
coordinates.

Let $Q_L$ be the depth-$L$ reconstruction of a nested binary scalar code
whose first $L$ refinement gains are positive, and define
$\epsilon_L=\sup_{|u|\leq1}|u-Q_L(u)|$. For
$x=U^\top z/\sqrt d$, the encoder stores the $L$-bit code of each coordinate
of $z$. Given $y=U^\top q$, the \ReadKV{} key reader assigns coordinate $i$
weight $q_i^2/d$. With a budget of $Lk$ refinements, it reads all $L$ bits
of every coordinate in $\operatorname{supp}(q)$ and nothing elsewhere. The
resulting estimate is
$\widehat s_{\ReadKV}(x,y)=d^{-1/2}\sum_iq_iQ_L(z_i)$.
All addresses are determined before any bit is returned, so this construction
uses one query-dependent batch.

\begin{theorem}[Finite \ReadKV{} Separation for Transformed-Sparse Queries]
\label{thm:finite_readkv_separation}
Let $d\geq1$ and $1\leq k\leq d$. Let $L,r,w$ be positive integers satisfying
$Ld\leq w$ and $Lk\leq r$. The construction above stores $Ld$ informative
bits in a $w$-bit record, padding unused cells if necessary, and uses at most
$Lk$ probes. Its worst-case error satisfies
\[
 \sup_{x\in\mathcal X_U,\ y\in\mathcal Y_{U,k}}
 \bigl(x^\top y-\widehat s_{\ReadKV}(x,y)\bigr)^2
 \leq \frac{k\epsilon_L^2}{d}.
\]
In contrast, every randomized protocol $\Pi$ whose reader is
query-independent branching and uses at most $r$ probes satisfies
\[
 \sup_{x\in\mathcal X_U,\ y\in\mathcal Y_{U,k}}
 \mathbb E_\Pi\!\left[
   \bigl(x^\top y-\widehat s_\Pi(x,y)\bigr)^2
 \right]
 \geq \frac{2^{\,1-2r/d}}{\pi e\,d}.
\]
The competing protocol may use an arbitrary randomized encoder and decoder
and a finite binary record of arbitrary length. Therefore, the \ReadKV{}
construction has strictly smaller worst-case MSE whenever
$k\epsilon_L^2<2^{\,1-2r/d}/(\pi e)$.
\end{theorem}

The lower bound places no restriction on the competing write budget, record
construction, or decoder. It even permits the reader to branch on returned
payloads. The only restriction is that the query cannot influence the access
path. The separation therefore identifies the value of query-dependent
addressing rather than an advantage of one particular encoder or decoder.

For the deployed depth-eight code, take $w=8d$ and $r=2d$. The theorem then
applies when $k\leq\lfloor d/4\rfloor$, and the ratio between the converse
and the \ReadKV{} upper bound exceeds $155/k$. At $d=128$ and $k=32$, the
certified factor exceeds $4.85$. The transformed-sparsity condition is
important for this statement. Appendix~\ref{app:key_stage_obstruction} gives
an explicit full-unit-ball example on which this deployed deterministic code
does not dominate every query-oblivious protocol.

The proof has two parts. The upper bound follows directly from the pointwise
scalar error:
$|x^\top y-\widehat s_{\ReadKV}|\leq
\epsilon_L\sqrt{k/d}$. For the converse, condition on all randomness visible
to the reader. An $r$-probe query-independent reader is then a binary decision
tree with at most $2^r$ leaves, so its transcript reveals at most $r$ bits
about the source vector. An entropy--MMSE inequality converts this information
limit into a lower bound on vector reconstruction error. Finally, a uniformly
random signed coordinate query converts that reconstruction error into
inner-product MSE. Appendix~\ref{app:key_stage_separation} gives the complete
argument and the constant calculation.

\subsection{Two-Stage Attention Separation}
\label{sec:two_stage_separation}

The preceding theorem isolates the key stage. We next analyze the complete
causal transition used by \ReadKV{}: key reads determine reconstructed
attention, and that attention determines the value reads.

Fix an integer $m\geq1$. Let $e_1\in\mathbb R^m$ be the first standard basis
vector, let $Q_8$ be the deployed depth-eight reconstruction, and let $c_1$
be its positive depth-one reconstruction value. Choose
$b,\eta\in\{-1,1\}$ and arbitrary values $v_1,v_2\in[-1,1]^m$. Define
\[
 k_1=bc_1e_1,\qquad k_2=-bc_1e_1,\qquad
 q_\eta=\eta(5\sqrt m/c_1)e_1.
\]
After the standard $1/\sqrt m$ attention normalization, the logits are
$(5\eta b,-5\eta b)$. Put
$p=(1+e^{-10})^{-1}$ and $\rho=1-p$. Let $j_\star=1$ when
$\eta b=1$ and $j_\star=2$ otherwise. The exact attention output is
$o=pv_{j_\star}+\rho v_{3-j_\star}\in\mathbb R^m$. The encoder observes
$(b,v_1,v_2)$ before the query sign $\eta$ is known. This construction uses
the displayed query norm and is separate from the unit-ball family in the
previous theorem.

The encoder applies the depth-eight progressive code to both keys and both
values. The first batch reads one key bit from each token, reconstructing
both keys exactly and identifying $j_\star$. Write
$\Delta_t=D(t-1)-D(t)$ for the deployed distortion gain at depth $t$. The
gain table satisfies $p^2\Delta_8>\rho^2\Delta_1$, so every refinement of the
high-attention value has larger weighted gain than the first refinement of
the other value; Lemma~\ref{lem:deployed_codec_bounds} verifies the required
gain bounds. The second batch therefore reads $v_{j_\star}$ to depth eight
and leaves $v_{3-j_\star}$ unread. The decoder returns
$\widehat o_{\ReadKV}=pQ_8(v_{j_\star})$, where $Q_8$ is applied
coordinatewise. For a competing protocol $\Pi$, write $\widehat o_\Pi$ for
its decoder output.

The following theorem shows that the complete two-stage procedure, rather
than only its key reader, has a strict worst-case advantage. The comparison
allows an arbitrary competing encoder and decoder, but requires its read
path to be independent of the current query.

\begin{theorem}[Finite Two-Stage \ReadKV{} Separation]
\label{thm:finite_two_stage_readkv_separation}
For every $m\geq1$, the construction above is a $\mathsf{staged}(2)$
\ReadKV{} protocol. It stores $w=32m$ bits and reads $r=8m+2$ bits. Let
$\epsilon_8=\sup_{|u|\leq1}|u-Q_8(u)|<0.0097$. Its worst-case
per-coordinate attention-output error satisfies
\[
 \sup_{b,v_1,v_2,\eta}
 \frac1m\|o-\widehat o_{\ReadKV}\|_2^2
 \leq(p\epsilon_8+\rho)^2
 <9.50\mathbin{\times}10^{-5},
\]
where the supremum is over
$b,\eta\in\{-1,1\}$ and $v_1,v_2\in[-1,1]^m$. Every randomized protocol
$\Pi$ whose encoder is fixed before $\eta$ is revealed and whose reader is
query-independent branching with respect to $\eta$, using at most $r$ probes,
satisfies
\[
 \sup_{b,v_1,v_2,\eta}
 \frac1m\mathbb E_\Pi
 \!\left[\|o-\widehat o_\Pi\|_2^2\right]
 \geq
 \frac{\tanh^2(5)}{\pi e}\,2^{-7-2/m}
 >2.285\mathbin{\times}10^{-4}.
\]
The competing protocol may use an arbitrary randomized encoder and decoder
and a finite binary record of arbitrary length. Thus, the certified
separation factor exceeds $2.40$ for every $m\geq1$ and exceeds $9.52$ at
$m=128$. The latter setting has $W=8$ stored bits per cached scalar and
realized read rate
$R=(8m+2)/(4m)=2+1/(2m)\approx2.004$.
\end{theorem}

The proof follows the two read stages. The key probes first recover the exact
attention weights $p$ and $\rho$. The gain inequality then forces all eight
value refinements onto the high-attention token, yielding the upper bound
$(p\epsilon_8+\rho)^2$. For the converse, a query-independent transcript
contains at most $r$ bits about the two value vectors but must support either
query sign. Averaging over those signs exposes a direction with coefficient
$\tanh^2(5)/2$, after which the entropy--MMSE bound gives the stated lower
bound.
Appendix~\ref{app:two_stage_separation} proves each step and verifies the
numerical constants.

\subsection{Allocation and Attention-Error Guarantee}
\label{sec:allocation_attention}

The separation theorems show that query-dependent access can be strictly
useful. We now analyze the allocation made by \ReadKV{} on an arbitrary fixed
cache and query. The aim is to establish two properties. Each stage should
choose the best legal prefix depths for the calibrated objective defined in
Section~\ref{sec:readkv}. Those optimized objectives should then control the
attention-output error through explicit, inspectable quantities.

A \emph{depth schedule} assigns an integer prefix depth to each allocation
item. Key items are physical-head and transformed-channel pairs, while value
items are physical-head and token pairs. Their feasible sets are
\[
 \mathcal T_K(B_K)=\{t^K\in\{0,\ldots,W\}^{G\times d_h}:
 \textstyle\sum_{g,c}t^K_{g,c}\leq B_K\},\qquad
 \mathcal T_V(B_V)=\{t^V\in\{0,\ldots,W\}^{G\times T}:
 \textstyle\sum_{g,j}t^V_{g,j}\leq B_V\}.
\]
Using the importance scores in Equations~\eqref{eq:key_importance}
and~\eqref{eq:value_importance}, define
\[
 S_K(t^K)=\frac{T}{d_h}\sum_{g,c}a^K_{g,c}D(t^K_{g,c}),
 \qquad
 S_V(t^V;\widehat\alpha)
 =\sum_{g,j}a^V_{g,j}(\widehat\alpha)D(t^V_{g,j}).
\]
The key objective is the calibrated diagonal approximation to squared logit
error, aggregated over query heads that share a physical KV head. The value
objective is the corresponding approximation to squared
attention-weighted value error. The factor $T/d_h$ is constant across key
schedules, so it does not affect which schedule minimizes $S_K$.

The exact output error need not equal these diagonal objectives. Calibration
may misestimate the realized residual energy, and residuals from different
channels or tokens may interact. For query head $h$, the factors
$\gamma_{K,h}$ and $\gamma_{V,h}$ measure the ratio between realized
diagonal residual energy and its calibrated prediction. The factors
$\kappa_{K,h}$ and $\kappa_{V,h}$ are the largest eigenvalues of normalized
residual Gram matrices and measure amplification from cross terms.
Appendix~\ref{app:correlation_attention_bound} gives the exact definitions,
including the zero-residual conventions.

Because $U_V$ is orthogonal, output error can be measured in transformed
value coordinates. Let $V_{g(h)}\in\mathbb R^{T\times d_h}$ contain the
full-precision transformed values for the physical KV head used by query
head $h$, and let $\mathbf1_T\in\mathbb R^T$ be the all-ones vector. Since
both exact and reconstructed attention weights sum to one, subtracting any
center $c_h\in\mathbb R^{d_h}$ from every value row does not change the error
caused by perturbing the attention probabilities. We write
$\|\cdot\|_{2\to2}$ for the matrix operator norm.

\begin{theorem}[\ReadKV{} Allocation and Attention-Error Certificate]
\label{thm:readkv_attention_certificate}
Fix an attention layer and query. Let $W\geq1$ and $B_K,B_V\geq0$ be
integers. Let $D(t)\geq0$ for $t=0,\ldots,W$, and suppose the scalar
distortion gains $D(t-1)-D(t)$ are nonnegative and nonincreasing for
$t=1,\ldots,W$. Let $t^K_{\ReadKV}$ be the key schedule returned by
Algorithm~\ref{alg:readkv}, let $\widehat\alpha_{\ReadKV}$ be the attention
computed from the resulting key reconstruction, and let
$t^V_{\ReadKV}$ be the value schedule returned in the second stage. Then
\[
 S_K(t^K_{\ReadKV})
 =\min_{t^K\in\mathcal T_K(B_K)}S_K(t^K),
 \qquad
 S_V(t^V_{\ReadKV};\widehat\alpha_{\ReadKV})
 =\min_{t^V\in\mathcal T_V(B_V)}
   S_V(t^V;\widehat\alpha_{\ReadKV}).
\]

Let $o_h,\widehat o_h\in\mathbb R^{d_h}$ be the exact and reconstructed
outputs for query head $h$. Choose any center $c_h\in\mathbb R^{d_h}$, and
evaluate the residual-correlation factors
$\kappa_{K,h},\kappa_{V,h}$ and calibration factors
$\gamma_{K,h},\gamma_{V,h}$ at the \ReadKV{} reconstruction. Define
\[
 \overline\beta_K
 =\max_h\frac12
   \|V_{g(h)}-\mathbf1_Tc_h^\top\|_{2\to2}^2
   \kappa_{K,h}\gamma_{K,h},
 \qquad
 \overline\beta_V
 =\max_h 2\kappa_{V,h}\gamma_{V,h}.
\]
For every fixed cache and query for which these factors are finite,
\[
 \sum_h\|o_h-\widehat o_h\|_2^2
 \leq
 \overline\beta_K
 \min_{t^K\in\mathcal T_K(B_K)}S_K(t^K)
 +
 \overline\beta_V
 \min_{t^V\in\mathcal T_V(B_V)}
 S_V(t^V;\widehat\alpha_{\ReadKV}).
\]
The bound is deterministic and permits arbitrary dependence among queries,
keys, values, and quantization residuals.
\end{theorem}

The theorem has two messages. First, the allocation rule is exact for the
calibrated objectives used by \ReadKV{}. It finds the best legal integer
prefix schedule for the key stage and, once the reconstructed keys determine
the attention probabilities, for the value stage. This is a stagewise
statement, not a claim that the procedure jointly minimizes the true
attention-output error. The deployed depth-eight code satisfies the required
diminishing-gain condition by Lemma~\ref{lem:deployed_codec_bounds} in
Appendix~\ref{app:codec}.

Second, the theorem identifies when small calibrated objectives imply small
attention-output error. The conversion factors account for calibration
mismatch, interactions among reconstruction residuals, and the scale of the
values. No independence assumption is needed. Since these factors use the
realized residuals, the result is an offline certificate for a completed
reconstruction. Corollary~\ref{cor:calibrated_coherence} gives a simpler
sufficient condition based on calibration envelopes and pairwise residual
coherence.

The proof uses two main ideas. For allocation, each item's successive prefix
bits form a list of nonincreasing gains. Choosing the largest available gain
therefore merges these lists in order, and an exchange argument proves that
the resulting schedule is optimal. For the attention bound, the output error
is split into the change caused by inaccurate attention weights and the
error caused by reconstructing the values. Softmax stability controls the
first part, while residual Gram matrices control the correlations omitted by
the calibrated objectives. Summing the resulting bounds over query heads
gives the displayed certificate.
Appendix~\ref{app:allocation_attention_guarantee} provides the full proof and
all zero-residual cases.

\section{Language-Model Evaluation}
\label{sec:quality}

We test whether \ReadKV{} reduces KV-cache reads while preserving quality
on C4 language modeling \citep{raffel2020exploring}, LongBench question
answering \citep{bai2024longbench}, and passkey retrieval
\citep{mohtashami2023landmark}. We use three nominal storage/read settings:
$(4,2)$ uses a smaller cache and low read budget; $(8,2)$ retains more
detail at the same read budget; $(8,4)$ raises the read budget at fixed
capacity. The flexible reader's integer K/V split is selected on development
data for each model and rate, then frozen before test. The models in
Tables~\ref{tab:polished-prediction} and~\ref{tab:polished-qa} cover Qwen2.5
\citep{qwen2025qwen25}, Yi-1.5 \citep{ai2024yi}, DeepSeek-LLM
\citep{deepseek2024llm}, and Mistral \citep{jiang2023mistral}.
Appendix~\ref{app:quality_protocol} specifies precision, transform seeds,
hardware and calibration cohorts, and uncertainty estimation.
 Single-layer GPU latency and physical-traffic measurements
use an NVIDIA A10G and the restricted reader described in
Section~\ref{sec:gpu_eval}; these are separate from the flexible-reader quality
experiments.

\noindent\textbf{Baseline Settings.}
\textsc{Dense} uses an uncompressed 16-bit KV cache and supplies the quality
and resource reference. Our internal \textsc{Full Reader} $(4,4)$ stores
and always reads four bits per scalar from \ReadKV{}'s nested code. It
separates two effects: against $(8,4)$ it holds the code and the read budget
fixed and varies only the stored depth, and with it whether the query can
choose which bits to fetch; against $(4,2)$ it holds the code and the storage
fixed and varies only the read budget. KIVI and TurboQuant are fixed-rate
codes: every query reads all retained payload bits.
KIVI (2-bit K/V) \citep{liu2024kivi} uses two payload bits per key and value
scalar, plus groupwise scales and offsets and a 128-token uncompressed
residual buffer. TurboQuant K4/V4 \citep{zandieh2026turboquant} uses
four-bit key and value codes, with the first and last two layers
uncompressed. We give \ReadKV{} the same exemptions in rows marked
$\dagger$; their $(W,R)$ budgets apply to compressed layers. SparQ
$r32/k128$ \citep{ribar2024sparq} retains a 16-bit cache, probes the
history using $r=32$ query-selected key channels, and reads at most
$k=128$ selected past positions. Appendix~\ref{app:competitor_details}
gives implementation and configuration-selection details.

\noindent\textbf{Reading the Results.}
For each model and dataset, we report the percentage perplexity increase
$\Delta\mathrm{PPL}=100(\mathrm{PPL}_{\mathrm{method}}/
\mathrm{PPL}_{\text{\textsc{Dense}}}-1)$ or the F1-point loss
$\Delta\mathrm{F1}=\mathrm{F1}_{\text{\textsc{Dense}}}-
\mathrm{F1}_{\mathrm{method}}$. Positive values are worse; negative values
are sample improvements. Logical reads and retained KV capacity are
percentages of \textsc{Dense}. Logical reads include payload, metadata,
probes, rereads, and dense layers; they exclude model weights, cache writes,
and workspace, and are not physical DRAM traffic. Entries are encoded once
and read repeatedly: extra retained refinements can improve quality without
increasing reads, although their capacity still limits context length and
concurrency.

\noindent\textbf{C4 Language Modeling.}
Each model is evaluated on 32 held-out 2,048-token documents, with separate
calibration and development data. Tokenizers differ, so comparisons are
within each model.

\begin{table}[!ht]
\centering
\caption{C4: 32 held-out 2,048-token documents per model; three transform seeds where available. \ReadKV{} $(8,4)$ stays within 0.7\% of \textsc{Dense} on all six models, using about a quarter of its logical reads and half its KV storage. $\dagger$: first/last two layers dense. Color marks the observed quality--read frontier: at shown precision, no other method row in the column improves loss or reads without worsening the other. Reads recur per query; retained capacity is shown but not ranked.}
\label{tab:polished-prediction}
\label{tab:unified-prediction-a}
\label{tab:unified-prediction-b}
\providecolor{tradeaccent}{HTML}{B34E28}
\small
\setlength{\tabcolsep}{5pt}
\renewcommand{\arraystretch}{1.10}
\begin{tabularx}{\linewidth}{@{}>{\raggedright\arraybackslash}X*{3}{>{\centering\arraybackslash}p{.19\linewidth}}@{}}
\toprule
\multicolumn{4}{@{}l}{\footnotesize Each method cell reports $\Delta\mathrm{PPL}$ (\%) on top and logical reads / retained KV capacity (\% of \textsc{Dense}) below.}\\[4pt]
\multicolumn{4}{@{}l}{\textit{Panel A: Qwen2.5 Models}}\\[2pt]
Method $(W,R)$ & \shortstack{Qwen2.5\\3B} & \shortstack{Qwen2.5\\7B} & \shortstack{Qwen2.5\\14B} \\
{\footnotesize\textcolor{black!65}{\textsc{Dense} PPL}} & {\footnotesize\textcolor{black!65}{11.36}} & {\footnotesize\textcolor{black!65}{10.13}} & {\footnotesize\textcolor{black!65}{8.74}} \\[2pt]
\midrule
\textsc{Full Reader} $(4,4)$ & \shortstack{+1.59\\[1.5pt]{\footnotesize 25.6\,/\,\textcolor{black!65}{25.1}}} & \shortstack{+1.34\\[1.5pt]{\footnotesize 25.4\,/\,\textcolor{black!65}{25.1}}} & \shortstack{+1.68\\[1.5pt]{\footnotesize 25.3\,/\,\textcolor{black!65}{25.1}}} \\[2pt]
\shortstack[l]{KIVI (2-bit K/V)\\[1pt]{\footnotesize\hypersetup{citecolor=black!60}\citealp{liu2024kivi}}} & \shortstack{+4.84\\[1.5pt]{\footnotesize 26.2\,/\,\textcolor{black!65}{23.8}}} & \shortstack{+2.79\\[1.5pt]{\footnotesize 26.2\,/\,\textcolor{black!65}{23.8}}} & \shortstack{+2.21\\[1.5pt]{\footnotesize 26.2\,/\,\textcolor{black!65}{23.8}}} \\[2pt]
\shortstack[l]{SparQ $r32/k128$\\[1pt]{\footnotesize\hypersetup{citecolor=black!60}\citealp{ribar2024sparq}}} & \textcolor{tradeaccent}{\shortstack{+1.27\\[1.5pt]{\footnotesize 24.7\,/\,100.0}}} & \textcolor{tradeaccent}{\shortstack{+1.44\\[1.5pt]{\footnotesize 24.7\,/\,100.0}}} & \textcolor{tradeaccent}{\shortstack{+0.81\\[1.5pt]{\footnotesize 24.7\,/\,100.0}}} \\[2pt]
\addlinespace[3pt]
\ReadKV{} $(4,2)$ & \shortstack{+4.44\\[1.5pt]{\footnotesize 13.1\,/\,\textcolor{black!65}{25.1}}} & \shortstack{+3.41\\[1.5pt]{\footnotesize 12.9\,/\,\textcolor{black!65}{25.1}}} & \shortstack{+3.75\\[1.5pt]{\footnotesize 12.8\,/\,\textcolor{black!65}{25.1}}} \\[2pt]
\ReadKV{} $(8,2)$ & \textcolor{tradeaccent}{\shortstack{+3.85\\[1.5pt]{\footnotesize 13.1\,/\,50.1}}} & \textcolor{tradeaccent}{\shortstack{+2.99\\[1.5pt]{\footnotesize 12.9\,/\,50.1}}} & \textcolor{tradeaccent}{\shortstack{+2.94\\[1.5pt]{\footnotesize 12.8\,/\,50.1}}} \\[2pt]
\ReadKV{} $(8,4)$ & \textcolor{tradeaccent}{\shortstack{+0.19\\[1.5pt]{\footnotesize 25.6\,/\,50.1}}} & \textcolor{tradeaccent}{\shortstack{+0.39\\[1.5pt]{\footnotesize 25.4\,/\,50.1}}} & \textcolor{tradeaccent}{\shortstack{+0.36\\[1.5pt]{\footnotesize 25.3\,/\,50.1}}} \\[2pt]
\shortstack[l]{TurboQuant K4/V4$^{\dagger}$\\[1pt]{\footnotesize\hypersetup{citecolor=black!60}\citealp{zandieh2026turboquant}}} & \shortstack{+1.74\\[1.5pt]{\footnotesize 34.4\,/\,\textcolor{black!65}{34.4}}} & \shortstack{+1.02\\[1.5pt]{\footnotesize 36.7\,/\,\textcolor{black!65}{36.7}}} & \shortstack{+0.72\\[1.5pt]{\footnotesize 32.3\,/\,\textcolor{black!65}{32.3}}} \\[2pt]
\ReadKV{} $(8,4)^{\dagger}$ & \textcolor{tradeaccent}{\shortstack{+0.14\\[1.5pt]{\footnotesize 33.9\,/\,55.6}}} & \textcolor{tradeaccent}{\shortstack{+0.09\\[1.5pt]{\footnotesize 36.0\,/\,57.2}}} & \textcolor{tradeaccent}{\shortstack{+0.13\\[1.5pt]{\footnotesize 31.5\,/\,54.3}}} \\[2pt]
\midrule
\multicolumn{4}{@{}l}{\textit{Panel B: Yi, DeepSeek, and Mistral Models}}\\[2pt]
Method $(W,R)$ & \shortstack{Yi-1.5\\6B} & \shortstack{DeepSeek\\LLM-7B} & \shortstack{Mistral\\7B-v0.3} \\
{\footnotesize\textcolor{black!65}{\textsc{Dense} PPL}} & {\footnotesize\textcolor{black!65}{8.96}} & {\footnotesize\textcolor{black!65}{8.49}} & {\footnotesize\textcolor{black!65}{7.62}} \\[2pt]
\midrule
\textsc{Full Reader} $(4,4)$ & \shortstack{+3.02\\[1.5pt]{\footnotesize 25.4\,/\,\textcolor{black!65}{25.1}}} & \shortstack{+3.44\\[1.5pt]{\footnotesize 25.2\,/\,\textcolor{black!65}{25.1}}} & \shortstack{+2.17\\[1.5pt]{\footnotesize 25.3\,/\,\textcolor{black!65}{25.1}}} \\[2pt]
\shortstack[l]{KIVI (2-bit K/V)\\[1pt]{\footnotesize\hypersetup{citecolor=black!60}\citealp{liu2024kivi}}} & \shortstack{+2.82\\[1.5pt]{\footnotesize 26.2\,/\,\textcolor{black!65}{23.8}}} & \shortstack{+3.54\\[1.5pt]{\footnotesize 26.2\,/\,\textcolor{black!65}{23.8}}} & \shortstack{+2.11\\[1.5pt]{\footnotesize 26.2\,/\,\textcolor{black!65}{23.8}}} \\[2pt]
\shortstack[l]{SparQ $r32/k128$\\[1pt]{\footnotesize\hypersetup{citecolor=black!60}\citealp{ribar2024sparq}}} & \shortstack{+92.44\\[1.5pt]{\footnotesize 24.7\,/\,\textcolor{black!65}{100.0}}} & \textcolor{tradeaccent}{\shortstack{+0.86\\[1.5pt]{\footnotesize 24.7\,/\,100.0}}} & \shortstack{+53.95\\[1.5pt]{\footnotesize 24.7\,/\,\textcolor{black!65}{100.0}}} \\[2pt]
\addlinespace[3pt]
\ReadKV{} $(4,2)$ & \shortstack{+12.13\\[1.5pt]{\footnotesize 12.9\,/\,\textcolor{black!65}{25.1}}} & \shortstack{+9.74\\[1.5pt]{\footnotesize 12.7\,/\,\textcolor{black!65}{25.1}}} & \textcolor{tradeaccent}{\shortstack{+11.00\\[1.5pt]{\footnotesize 12.8\,/\,25.1}}} \\[2pt]
\ReadKV{} $(8,2)$ & \textcolor{tradeaccent}{\shortstack{+10.16\\[1.5pt]{\footnotesize 12.9\,/\,50.1}}} & \textcolor{tradeaccent}{\shortstack{+6.27\\[1.5pt]{\footnotesize 12.7\,/\,50.1}}} & \shortstack{+11.24\\[1.5pt]{\footnotesize 12.8\,/\,\textcolor{black!65}{50.1}}} \\[2pt]
\ReadKV{} $(8,4)$ & \textcolor{tradeaccent}{\shortstack{+0.20\\[1.5pt]{\footnotesize 25.4\,/\,50.1}}} & \textcolor{tradeaccent}{\shortstack{+0.66\\[1.5pt]{\footnotesize 25.2\,/\,50.1}}} & \textcolor{tradeaccent}{\shortstack{+0.49\\[1.5pt]{\footnotesize 25.3\,/\,50.1}}} \\[2pt]
\shortstack[l]{TurboQuant K4/V4$^{\dagger}$\\[1pt]{\footnotesize\hypersetup{citecolor=black!60}\citealp{zandieh2026turboquant}}} & \shortstack{+0.99\\[1.5pt]{\footnotesize 35.4\,/\,\textcolor{black!65}{35.4}}} & \shortstack{+0.45\\[1.5pt]{\footnotesize 36.0\,/\,\textcolor{black!65}{36.0}}} & \shortstack{+0.45\\[1.5pt]{\footnotesize 35.4\,/\,\textcolor{black!65}{35.4}}} \\[2pt]
\ReadKV{} $(8,4)^{\dagger}$ & \shortstack{+0.22\\[1.5pt]{\footnotesize 34.7\,/\,\textcolor{black!65}{56.3}}} & \textcolor{tradeaccent}{\shortstack{+0.11\\[1.5pt]{\footnotesize 35.2\,/\,56.8}}} & \textcolor{tradeaccent}{\shortstack{+0.11\\[1.5pt]{\footnotesize 34.6\,/\,56.3}}} \\[2pt]
\bottomrule
\end{tabularx}
\end{table}

With the first and last two layers uncompressed, \ReadKV{} $(8,4)^\dagger$
achieves lower perplexity and slightly fewer logical reads than TurboQuant
K4/V4 on all six models, while retaining more KV storage.
Figure~\ref{fig:competitor-comparison}(a) displays this comparison with the
same layer exemptions. Without exemptions, \ReadKV{} $(8,4)$ stays within
0.66\% of \textsc{Dense}; it improves perplexity over the evaluated SparQ
configuration on all six models and retains roughly half the capacity, with a
slightly higher read rate.
It also improves on the \textsc{Full Reader}'s 1.34--3.44\% PPL increase at
the same four-bit payload-read budget, illustrating the value of retained
refinements.

The $R=2$ settings reduce reads further but are less consistent across
families: $(8,2)$ stays within 3.9\% PPL increase on Qwen, versus
6.27--11.24\% on Yi, DeepSeek, and Mistral. The smaller $(4,2)$ cache
retains the \textsc{Full Reader}'s quarter-capacity footprint with about
49\% fewer logical reads. Query-dependent allocation also matters: in a
paired Qwen3B control with the same code and K/V split, calibration-only key
allocation raises PPL from 11.789 to 12.011; 31 of 32 documents favor
current-query allocation (Appendix~\ref{app:placement}).
Appendices~\ref{app:absolute_quality}, \ref{app:depth_support},
and~\ref{app:surrogate_experiment} give additional baselines, intermediate
prefix-depth controls, and the surrogate/error comparison.

\noindent\textbf{LongBench Question Answering.}
We evaluate Qwen2.5-7B-Instruct and Mistral-7B-Instruct-v0.3 on 100 questions
per dataset from Qasper, HotpotQA, and 2WikiMultihopQA, using the same
identifiers at 3,500- and 7,500-token prompt limits. All methods use
\textsc{Dense} prompt prefill, which selects the first answer token;
compression affects generation from the second token.

\begin{table}[!ht]
\centering
\caption{LongBench QA at a 7,500-token prompt limit; 100 questions per condition. \ReadKV{} $(8,2)$ loses at most 1.46 F1 points across the six conditions while using about 13\% of \textsc{Dense} logical reads. $\dagger$: first/last two layers dense. Color marks the observed quality--read frontier: at shown precision, no other method row in the column improves loss or reads without worsening the other. Reads recur per query; retained capacity is shown but not ranked.}
\label{tab:polished-qa}
\label{tab:unified-qa-qasper-7500}
\label{tab:unified-qa-hotpotqa-7500}
\label{tab:unified-qa-2wikimqa-7500}
\providecolor{tradeaccent}{HTML}{B34E28}
\small
\setlength{\tabcolsep}{5pt}
\renewcommand{\arraystretch}{1.10}
\begin{tabularx}{\linewidth}{@{}>{\raggedright\arraybackslash}X*{3}{>{\centering\arraybackslash}p{.19\linewidth}}@{}}
\toprule
\multicolumn{4}{@{}l}{\footnotesize Each method cell reports $\Delta\mathrm{F1}$ (points) on top and logical reads / retained KV capacity (\% of \textsc{Dense}) below.}\\[4pt]
\multicolumn{4}{@{}l}{\textit{Panel A: Qwen2.5-7B-Instruct}}\\[2pt]
Method $(W,R)$ & \shortstack{Qasper} & \shortstack{HotpotQA} & \shortstack{2Wiki\\MultihopQA} \\
{\footnotesize\textcolor{black!65}{\textsc{Dense} F1}} & {\footnotesize\textcolor{black!65}{41.13}} & {\footnotesize\textcolor{black!65}{48.34}} & {\footnotesize\textcolor{black!65}{45.30}} \\[2pt]
\midrule
\shortstack[l]{KIVI (2-bit K/V)\\[1pt]{\footnotesize\hypersetup{citecolor=black!60}\citealp{liu2024kivi}}} & \shortstack{+1.56\\[1.5pt]{\footnotesize 20.4\,/\,\textcolor{black!65}{19.9}}} & \shortstack{+0.08\\[1.5pt]{\footnotesize 20.0\,/\,\textcolor{black!65}{19.9}}} & \shortstack{+2.41\\[1.5pt]{\footnotesize 20.1\,/\,\textcolor{black!65}{19.9}}} \\[2pt]
\shortstack[l]{TurboQuant K4/V4$^{\dagger}$\\[1pt]{\footnotesize\hypersetup{citecolor=black!60}\citealp{zandieh2026turboquant}}} & \textcolor{tradeaccent}{\shortstack{-0.06\\[1.5pt]{\footnotesize 36.7\,/\,36.7}}} & \shortstack{+0.42\\[1.5pt]{\footnotesize 36.7\,/\,\textcolor{black!65}{36.7}}} & \textcolor{tradeaccent}{\shortstack{-0.55\\[1.5pt]{\footnotesize 36.7\,/\,36.7}}} \\[2pt]
\shortstack[l]{SparQ $r32/k128$\\[1pt]{\footnotesize\hypersetup{citecolor=black!60}\citealp{ribar2024sparq}}} & \textcolor{tradeaccent}{\shortstack{+1.40\\[1.5pt]{\footnotesize 15.3\,/\,100.0}}} & \shortstack{-0.44\\[1.5pt]{\footnotesize 14.2\,/\,\textcolor{black!65}{100.0}}} & \textcolor{tradeaccent}{\shortstack{+0.69\\[1.5pt]{\footnotesize 14.5\,/\,100.0}}} \\[2pt]
\addlinespace[3pt]
\ReadKV{} $(4,2)$ & \shortstack{+1.68\\[1.5pt]{\footnotesize 12.6\,/\,\textcolor{black!65}{25.0}}} & \textcolor{tradeaccent}{\shortstack{-0.60\\[1.5pt]{\footnotesize 12.6\,/\,25.0}}} & \shortstack{+5.43\\[1.5pt]{\footnotesize 12.6\,/\,\textcolor{black!65}{25.0}}} \\[2pt]
\ReadKV{} $(8,2)$ & \textcolor{tradeaccent}{\shortstack{+1.46\\[1.5pt]{\footnotesize 12.6\,/\,50.0}}} & \shortstack{-0.35\\[1.5pt]{\footnotesize 12.6\,/\,\textcolor{black!65}{50.0}}} & \textcolor{tradeaccent}{\shortstack{+1.09\\[1.5pt]{\footnotesize 12.6\,/\,50.0}}} \\[2pt]
\ReadKV{} $(8,4)$ & \textcolor{tradeaccent}{\shortstack{+0.98\\[1.5pt]{\footnotesize 25.1\,/\,50.0}}} & \shortstack{+0.25\\[1.5pt]{\footnotesize 25.1\,/\,\textcolor{black!65}{50.0}}} & \textcolor{tradeaccent}{\shortstack{+0.29\\[1.5pt]{\footnotesize 25.1\,/\,50.0}}} \\[2pt]
\ReadKV{} $(8,2)^{\dagger}$ & \shortstack{+2.34\\[1.5pt]{\footnotesize 25.1\,/\,\textcolor{black!65}{57.2}}} & \shortstack{-0.12\\[1.5pt]{\footnotesize 25.0\,/\,\textcolor{black!65}{57.2}}} & \shortstack{+2.28\\[1.5pt]{\footnotesize 25.1\,/\,\textcolor{black!65}{57.2}}} \\[2pt]
\ReadKV{} $(8,4)^{\dagger}$ & \shortstack{+1.01\\[1.5pt]{\footnotesize 35.8\,/\,\textcolor{black!65}{57.2}}} & \shortstack{-0.03\\[1.5pt]{\footnotesize 35.8\,/\,\textcolor{black!65}{57.2}}} & \shortstack{+0.96\\[1.5pt]{\footnotesize 35.8\,/\,\textcolor{black!65}{57.2}}} \\[2pt]
\midrule
\multicolumn{4}{@{}l}{\textit{Panel B: Mistral-7B-Instruct-v0.3}}\\[2pt]
Method $(W,R)$ & \shortstack{Qasper} & \shortstack{HotpotQA} & \shortstack{2Wiki\\MultihopQA} \\
{\footnotesize\textcolor{black!65}{\textsc{Dense} F1}} & {\footnotesize\textcolor{black!65}{37.41}} & {\footnotesize\textcolor{black!65}{43.37}} & {\footnotesize\textcolor{black!65}{42.71}} \\[2pt]
\midrule
\shortstack[l]{KIVI (2-bit K/V)\\[1pt]{\footnotesize\hypersetup{citecolor=black!60}\citealp{liu2024kivi}}} & \shortstack{+0.49\\[1.5pt]{\footnotesize 20.3\,/\,\textcolor{black!65}{19.9}}} & \shortstack{+1.22\\[1.5pt]{\footnotesize 20.0\,/\,\textcolor{black!65}{19.9}}} & \shortstack{+1.80\\[1.5pt]{\footnotesize 20.0\,/\,\textcolor{black!65}{19.9}}} \\[2pt]
\shortstack[l]{TurboQuant K4/V4$^{\dagger}$\\[1pt]{\footnotesize\hypersetup{citecolor=black!60}\citealp{zandieh2026turboquant}}} & \textcolor{tradeaccent}{\shortstack{-0.79\\[1.5pt]{\footnotesize 35.4\,/\,35.4}}} & \shortstack{+0.01\\[1.5pt]{\footnotesize 35.4\,/\,\textcolor{black!65}{35.4}}} & \textcolor{tradeaccent}{\shortstack{-0.31\\[1.5pt]{\footnotesize 35.4\,/\,35.4}}} \\[2pt]
\shortstack[l]{SparQ $r32/k128$\\[1pt]{\footnotesize\hypersetup{citecolor=black!60}\citealp{ribar2024sparq}}} & \shortstack{+1.26\\[1.5pt]{\footnotesize 15.0\,/\,\textcolor{black!65}{100.0}}} & \shortstack{+2.65\\[1.5pt]{\footnotesize 14.3\,/\,\textcolor{black!65}{100.0}}} & \textcolor{tradeaccent}{\shortstack{+0.47\\[1.5pt]{\footnotesize 14.4\,/\,100.0}}} \\[2pt]
\addlinespace[3pt]
\ReadKV{} $(4,2)$ & \shortstack{+0.69\\[1.5pt]{\footnotesize 12.6\,/\,\textcolor{black!65}{25.0}}} & \textcolor{tradeaccent}{\shortstack{-0.42\\[1.5pt]{\footnotesize 12.5\,/\,25.0}}} & \shortstack{+1.08\\[1.5pt]{\footnotesize 12.5\,/\,\textcolor{black!65}{25.0}}} \\[2pt]
\ReadKV{} $(8,2)$ & \textcolor{tradeaccent}{\shortstack{-0.66\\[1.5pt]{\footnotesize 12.6\,/\,50.0}}} & \shortstack{-0.23\\[1.5pt]{\footnotesize 12.5\,/\,\textcolor{black!65}{50.0}}} & \textcolor{tradeaccent}{\shortstack{+0.52\\[1.5pt]{\footnotesize 12.5\,/\,50.0}}} \\[2pt]
\ReadKV{} $(8,4)$ & \shortstack{+0.98\\[1.5pt]{\footnotesize 25.1\,/\,\textcolor{black!65}{50.0}}} & \textcolor{tradeaccent}{\shortstack{-0.44\\[1.5pt]{\footnotesize 25.0\,/\,50.0}}} & \textcolor{tradeaccent}{\shortstack{+0.00\\[1.5pt]{\footnotesize 25.0\,/\,50.0}}} \\[2pt]
\ReadKV{} $(8,2)^{\dagger}$ & \shortstack{+0.58\\[1.5pt]{\footnotesize 23.5\,/\,\textcolor{black!65}{56.3}}} & \shortstack{-0.23\\[1.5pt]{\footnotesize 23.5\,/\,\textcolor{black!65}{56.3}}} & \shortstack{+0.87\\[1.5pt]{\footnotesize 23.5\,/\,\textcolor{black!65}{56.3}}} \\[2pt]
\ReadKV{} $(8,4)^{\dagger}$ & \shortstack{+0.60\\[1.5pt]{\footnotesize 34.4\,/\,\textcolor{black!65}{56.3}}} & \textcolor{tradeaccent}{\shortstack{-0.67\\[1.5pt]{\footnotesize 34.4\,/\,56.3}}} & \shortstack{+0.00\\[1.5pt]{\footnotesize 34.4\,/\,\textcolor{black!65}{56.3}}} \\[2pt]
\bottomrule
\end{tabularx}
\end{table}

At the 7,500-token prompt limit, \ReadKV{} $(8,2)$ has F1 changes between $-0.66$ and
$1.46$ points. It uses 37--38\% fewer logical reads than KIVI (2-bit K/V),
65--66\% fewer than protected TurboQuant, and 12--18\% fewer than SparQ.
Figure~\ref{fig:competitor-comparison}(b) compares it with the two
all-layer competitors: KIVI retains a smaller cache, whereas SparQ retains
twice its capacity. \ReadKV{} scores higher than KIVI in all six conditions;
against SparQ it scores higher in two conditions and loses at most 0.41 F1
points in the other four. Every condition is shown.

\begin{figure}[!t]
\centering
\includegraphics[width=\linewidth]{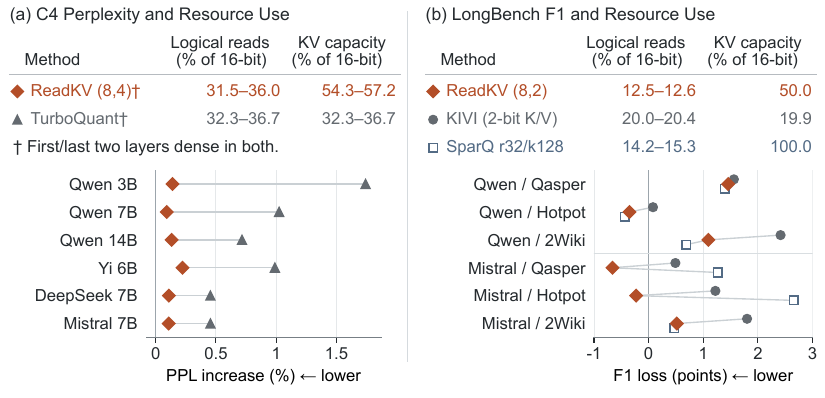}
\setlength{\abovecaptionskip}{3pt}
\caption{Comparison with representative baselines using
Tables~\ref{tab:polished-prediction} and~\ref{tab:polished-qa}.
\textbf{(a)} On C4, \ReadKV{} $(8,4)^\dagger$ has lower PPL and slightly
fewer logical reads than TurboQuant K4/V4 on all six models when the first
and last two layers remain dense, while retaining more KV capacity.
\textbf{(b)} On LongBench at a 7,500-token prompt limit, \ReadKV{} $(8,2)$
uses fewer logical reads than KIVI and SparQ. It has higher F1 than KIVI in
all six conditions and mixed results against SparQ. Resource values are
percentages of the 16-bit cache; lower quality loss is better.}
\label{fig:competitor-comparison}
\end{figure}

Table~\ref{tab:polished-qa} also shows how the read and storage budgets
affect quality. The all-layer $(8,4)$ setting has a largest loss of 0.98 F1
points. The quarter-capacity $(4,2)$ setting stays within two points in
five conditions, but loses 5.43 points on Qwen 2WikiMultihopQA; its Qwen
hardware and calibration cohort differs from $(8,2)$
(Appendix~\ref{app:quality_protocol}). Across both prompt limits, all-layer
$(8,2)$ and $(8,4)$ losses remain below 1.78 and 1.38 points, respectively.
The protected QA rows do not show a uniform quality advantage over
TurboQuant. Appendix~\ref{app:absolute_quality} gives the paired
3,500-token results and protected-row comparisons, including their
cross-cohort qualification.

\noindent\textbf{Passkey Retrieval.}
Appendix~\ref{app:retrieval} reports this long-context retention check;
several methods recover every tested key, so it supplements the C4 and QA
comparisons.

\subsection{GPU Evaluation}
\label{sec:gpu_eval}
To test whether lower logical reads reduce device cost, the timed reader uses
the same progressive code with hardware-oriented constraints: per-head
budgets, equal key and value rates, minimum prefix depth one, and two-token
value packets. Appendix~\ref{app:policy_quality} reports their quality effect
and defines the full device policy.

Table~\ref{tab:gpu-summary} summarizes the principal 8K-token comparison.

\begin{table}[!ht]
\centering\small\setlength{\tabcolsep}{6pt}
\caption{Single-layer decoding on an NVIDIA A10G with 8,192 cached tokens
and batch size one. Retained KV payload excludes temporary workspace.
Latency is the median CUDA Graph replay time. DRAM columns report
warm-cache / cache-evicted MiB from separately profiled eager calls and
include the complete attention operation.}
\label{tab:gpu-summary}
\begin{tabular}{lrrrr}
\toprule
Method & \shortstack{KV payload\\(MiB)} & \shortstack{Latency\\(ms)} &
\shortstack{DRAM reads\\(MiB)} & \shortstack{Total traffic\\(MiB)} \\
\midrule
\textsc{Dense} & 16.00 & 0.051 & 17.73 / 17.86 & 20.85 / 21.06 \\
TurboQuant K4/V4 \citep{zandieh2026turboquant} & 4.19 & 0.154 & 7.28 / 7.92 & 9.38 / 10.84 \\
\ReadKV{} $(8,2)$ & 8.00 & 0.093 & 2.80 / 2.74 & 7.47 / 7.47 \\
\bottomrule
\end{tabular}
\end{table}

On an NVIDIA A10G with 8,192 cached tokens and batch size one, the optimized
\ReadKV{} $(8,2)$ kernel takes 0.093\,ms under graph replay versus 0.154\,ms
for the TurboQuant codec and 0.051\,ms for \textsc{Dense}. Under the same
launcher and timing boundary, this is 39\% lower latency than the tested
quantized codec, while \textsc{Dense} is 45\% lower than \ReadKV{}. Across
warm and cache-evicted conditions, profiling shows 62--65\% fewer DRAM read
bytes and 20--31\% less total traffic than TurboQuant. The DRAM reads are
also smaller than the retained payload: in the warm condition, \ReadKV{}
$(8,2)$ reads 2.80\,MiB against an 8.00\,MiB payload, whereas TurboQuant and
\textsc{Dense} read 7.28 and 17.73\,MiB against 4.19 and 16.00\,MiB, more
than their entire caches. These single-layer measurements start from an existing cache and isolate
decode-time codec costs. Appendix~\ref{app:gpu_measurements} gives the full
timing grid, measurement boundaries, and retained-cache, workspace, and
physical-traffic accounting.

\section{Related Work}
\label{sec:related_work}

Fixed-rate KV-cache quantizers reduce the representation retained for every
future query. KIVI uses different grouping choices for keys and values,
KVQuant combines calibrated nonuniform quantization with outlier handling,
and TurboQuant uses rotations and vector-aware scalar quantization
\citep{liu2024kivi,hooper2024kvquant,zandieh2026turboquant}. Other methods
vary the precision resident in the cache using attention history or
quantization sensitivity. This includes QAQ, MiKV, ZipCache, PM-KVQ, MixKVQ,
and RDKV \citep{dong2024qaq,yang2024mikv,zhang2024zipcache,liu2026pmkvq,
zhang2026mixkvq,zhang2026rdkv}. These methods change what is retained;
\ReadKV{} keeps one record unchanged and varies what each query reads.

Query-adaptive retrieval instead chooses cache content using the current
query. SparQ selects key features before retrieving value tokens, Quest ranks
cache pages, and Loki uses low-dimensional key information to select tokens
\citep{ribar2024sparq,tang2024quest,singhania2024loki}. ThinK also uses query
information, but prunes the retained key cache after an observation window
\citep{xu2025think}. \ReadKV{} differs by assigning several possible
reconstruction depths to every stored scalar rather than making only a
feature, token, or page inclusion decision.

Progressive descriptions and successive refinement provide the coding
background for a representation that remains decodable at several rates
\citep{equitz1991successive,rimoldi1994successive}. Any-Precision LLM and
MatQuant apply related nested or sliced representations to model weights
\citep{park2024anyprecision,nair2025matquant}. The closest systems precedent
for our access pattern is \citet{xie2025reimagining}, who retain several KV
precision tiers and selectively retrieve floating-point bit planes for
ranked pages. Their work establishes that stored and fetched KV precision
can differ. \ReadKV{} uses a nested nonuniform scalar code with every prefix
depth available and allocates an explicit per-query read budget first across
key channels and then across value tokens using reconstructed attention.

The abstract model also connects to static data structures, locally
decodable source coding, and approximate inner-product sketches
\citep{dewolf2008structures,makhdoumi2013local,alon2016compression}. The
marginal-gain allocator and its continuous relaxation use classical discrete
allocation and reverse water-filling tools
\citep{fox1966discrete,shoham1988efficient,segall1976allocation,
cover2006elements}. Our contribution is not these optimization principles in
isolation, but their use in a causal partial-read KV protocol together with
finite separation and attention-error guarantees. Appendix~\ref{app:related}
gives a more detailed method-by-method comparison.

\section{Discussion}
\label{sec:discussion}
KV-cache entries are written once but read repeatedly under changing queries.
This asymmetry motivates separating retained from fetched precision: a cache
can preserve detail for future queries without moving all of it on every
decode step. \ReadKV{} implements this idea with a progressive code and a
causal two-stage reader. The query allocates key precision, and reconstructed
attention allocates value precision. We prove finite settings where this
access strictly outperforms query-independent reading, characterize the
stagewise allocator, and relate its objectives to attention-output error.
Across six base models, \ReadKV{} $(W=8,R=4)$ remains within $0.66\%$ of
uncompressed C4 perplexity and improves over fully reading a four-bit cache at
the same payload-read budget. Question answering and retrieval provide further
evidence of useful quality--read tradeoffs.

The theoretical separations concern explicit input families, and the
attention certificate depends on calibration and residual-correlation factors.
Finally, while the NVIDIA A10G study shows lower
single-layer latency and DRAM traffic than the tested fused quantized codec,
it does not yet establish an end-to-end serving advantage. Fusing allocation,
reconstruction, and attention, while supporting more flexible K/V budgets, is
the main path toward closing this gap. Retaining more than is read pays off
where the cache, not the weights, dominates the bytes moved per step: long
contexts and large batches, and possibly tiered memory that keeps the less
frequently fetched refinement levels in slower, larger memory.
Appendix~\ref{app:throughput_model} gives a first-order model of this regime;
complete-model evaluation under long contexts and memory-limited batching is
its direct test.

\section*{AI Use Statement}
Generative AI tools (ChatGPT/Codex) assisted with theoretical analysis and
proof checking, experiment design and implementation, result interpretation,
literature searches, and manuscript preparation. This included resource
accounting, table generation and consistency checks. Automated checks
compared reported summaries with saved experimental records. The authors
are responsible for the final text, claims, code and other artifacts.

\label{page:references}
\bibliography{references}
\bibliographystyle{plainnat}
\appendix

\section*{Appendix Organization}
The appendices provide definitions, proofs, reproduction details, and
additional experiments. The table maps each main technical topic to its
supporting material.

\noindent\begin{tabularx}{\linewidth}{@{}X l@{}}
\toprule
Topic & Location \\
\midrule
Storage and read accounting & Appendix~\ref{app:traffic} \\
Cache encoding and reading & Algorithm~\ref{alg:readkv} and Appendix~\ref{app:codec} \\
Proofs of query-dependent separations & Appendix~\ref{app:readkv_separation} \\
Integer prefix-allocation optimality & Appendix~\ref{app:allocation} \\
Proof of the allocation and attention-error certificate & Appendix~\ref{app:allocation_attention_guarantee} \\
Quality evaluation and mechanism controls & Appendix~\ref{app:quality} \\
Additional C4 controls and shorter-context QA & Tables~\ref{tab:polished-prediction-grid} and~\ref{tab:polished-qa-short} \\
GPU timing and traffic measurements & Appendix~\ref{app:gpu_measurements} \\
\bottomrule
\end{tabularx}
\medskip

\section{Storage and Traffic Accounting}
\label{app:traffic}

A cache entry occupies capacity while resident, incurs a write on insertion,
and incurs reads when reused. Reuse can amortize the insertion cost, but
does not reduce the retained capacity. We account for these costs separately.

\subsection{Task-Table Cost Accounting}
The quality tables keep the algorithm settings $W,R$ in the method name.
Their resource fields instead report percentages of \textsc{Dense}, the
matched uncompressed 16-bit KV representation, including each method's
required auxiliary information and any full-precision portions.
Figure~\ref{fig:competitor-comparison} uses these same resource quantities. Weight
precision and attention arithmetic are separate protocol choices; an FP32
attention reference does not imply a 32-bit accounted cache.

At a declared cache state $t$, let $C_t$ be the method's accounted
persistent KV-state bytes and $C_t^0$ those of \textsc{Dense}. Then
\[
 \mathrm{Cache}\%(t)=100\frac{C_t}{C_t^0}.
\]
The numerator includes required retained supporting information and copies;
any device-only subtotal is identified separately from host state.
It excludes model weights and temporary workspace: ``modeled cache'' means
the persistent KV representation, not a model's total allocated memory.
For matched calls, let $I_t$ and $I_t^0$ be the logical cache information
requested, in bits. Reading cost is
\[
 \mathrm{KV\ reads}\%=100\frac{\sum_t I_t}{\sum_t I_t^0}.
\]
Logical reads count payload, probes, rereads and fetched supporting state;
the denominator covers all original KV scalars that \textsc{Dense} would
inspect at those states, including entries skipped by a selective method.
A physical KV head is counted once in the underlying cache even when several
query heads share it.

The task tables report capacity at a fixed reference state: 2,048 cached
tokens for prediction, 3,500 or 7,500 for the corresponding QA panel, and
4,096 or 8,192 for passkey retrieval. The QA states are the prompt caps,
not claims that every example has that length. This choice keeps cache
capacity independent of answer length. Endpoint and trace-weighted occupancy
views remain in the machine-readable accounting record; neither is a
cumulative write count.

Logical reads use each method's recorded attention calls and normalize
against \textsc{Dense} at those same states. Under the dense-prefill
generation protocol, a $g$-token answer after an $n_0$-token prompt causes
$g-1$ compressed calls
at cache lengths $n_0+1,\ldots,n_0+g-1$: \textsc{Dense} prefill already
chose the first answer token, and the final emitted token is not forwarded.
These percentages compare information per recorded attention work; differing
stopping lengths do not establish equal-work total byte savings.

Token-associated payload and side information are supplemented by an
explicit conservative allowance for shared decoder data. Each layer call
is charged one load of that layer's calibration arrays and the retained
decoder table needed by the method. Tables shared across layers are counted
once in persistent capacity but once per compressed-layer call in this
read allowance. The accounting record separates these components. This
full-load convention avoids assuming free hardware residency; it is an
analytic allowance, not a measurement of device traffic or instruction-level
loads. It is applied to the competing codecs as well as to ReadKV.

Per-token or per-layer metadata, residual buffers and protected layers
are charged at their declared scope. Layout duplication is reported for
the corresponding representation; packet-tail padding is included in
logical reads where it is exported in the read counters. Shared
calibration and codebooks are charged once per retained instance. A quality
simulation's analytical count is not a measurement of GPU allocation or
physical traffic. Hardware transaction rounding belongs to physical
traffic, whereas complete packets requested by the algorithm belong to its
logical count.

Neither task-table field includes temporary selection, sorting,
reconstruction or attention buffers, even when scratch space is preallocated
and reused. Their simultaneously live storage belongs in peak workspace;
their accesses belong in physical traffic when they cross the measured
memory boundary. Allocation schedules are produced by the reader rather
than fetched from the retained cache, so they receive the same treatment.
The task-table costs include fetched calibration and decoder data but exclude
generated schedules. Including workspace accesses is necessary for a
device-traffic or latency claim, but adding them to retained KV reads would
change the quantity being compared.
Actual peak memory must count simultaneously live KV state, workspace,
weights and other allocations, rather than sum unrelated individual peaks.

\subsection{Payload-Only Cache-Reuse Model}

The following model explains the write/read example in the introduction.
It counts an idealized sequence of decode calls; the preceding subsection
gives the fuller accounting and dense-prefill boundary used by the task
tables.

We count logical payload bits for one cached key or value scalar. The
whole-cache count is obtained by multiplying by the number of key and value
scalars stored per token. A physical KV head is counted once even when several
query heads share it. Unlike the task-table accounting above, this model
omits metadata and implementation overhead to isolate the cost of reuse.

Let $N\geq0$ be the prompt length before decoding begins, and let $L\geq1$ be
the number of decode steps. At step $t\in\{0,\ldots,L-1\}$, the historical
cache contains $N+t$ tokens. We charge one insertion for each prompt token and
each generated token. This includes the token produced at the final decode
step, even though it is not read again within the accounting horizon. The read
count includes only accesses to historical cache entries during decoding. It
excludes prefill attention, current-token self-attention, metadata, and other
memory operations.

The total number of historical scalar reads is
\[
 S(N,L)=\sum_{t=0}^{L-1}(N+t)=LN+\frac{L(L-1)}2.
\]
For a fixed $W_0$-bit codec and a progressive codec that stores $W$ bits and
reads $R$ bits, the corresponding logical payload traffic is
\begin{align*}
 T_{W_0}(N,L)&=W_0(N+L)+W_0 S(N,L),\\
 T_{W,R}(N,L)&=W(N+L)+R S(N,L).
\end{align*}
The first term in each expression is write traffic, and the second is read
traffic. If keys and values use different depths, the formula is applied to
each and the two counts are added. If an implementation omits the final
insertion, $N+L$ is replaced by $N+L-1$ in both write terms.

\subsection{Equal Read Budgets}
Suppose $R=W_0<W$. The two codecs then fetch exactly the same number of
historical payload bits. The progressive codec additionally writes
$(W-W_0)(N+L)$ bits as entries are inserted. It cannot reduce total
logical traffic in this equal-read comparison, but the extra write traffic
may be small relative to the reads accumulated over the sequence. When
$S(N,L)>0$, that fraction is
\[
 \eta(N,L)=\frac{(W-W_0)(N+L)}
 {W_0[ LN+L(L-1)/2]}.
\]
For $W=8$, $W_0=R=4$, $N=8192$, and $L=128$, we obtain
$\eta(N,L)=0.0078735$. Thus, the added write traffic is approximately
$0.79\%$ of the common payload-read traffic. More generally, the extra-write
fraction is at most $\varepsilon>0$ exactly when
\[
 (W-W_0)(N+L)\leq \varepsilon W_0[ LN+L(L-1)/2].
\]

For $N\gg L$, this fraction approaches $(W-W_0)/(W_0L)$. If the sequence
starts from an empty cache, so that $N=0$ and $L>1$, it equals
$2(W-W_0)/[W_0(L-1)]$. These expressions show that repeated reuse, rather
than prompt length by itself, amortizes the additional writes. For example,
when $W=8$ and $W_0=4$, 201 decode steps make the fraction at most one percent
for every $N\geq0$. This statement is uniform over the prompt length. For
fixed $L$, the fraction decreases with $N$ because the derivative of
$(N+L)/[LN+L(L-1)/2]$ has numerator $-L(L+1)/2$. Its largest value therefore
occurs at $N=0$, where $2/(L-1)\leq0.01$ exactly when $L\geq201$.

\subsection{Reduced Read Budgets}
Suppose $W>W_0$ and $R<W_0$. The savings from repeated reads can then offset
the additional write traffic. Let $a=W-W_0>0$ be the extra stored depth and
$\delta=W_0-R>0$ the reduction in read depth. The progressive codec saves
$\delta S(N,L)$ read bits and writes $a(N+L)$ additional bits. Its total
logical payload traffic is therefore smaller exactly when
\[
 \delta\left[LN+\frac{L(L-1)}2\right]>a(N+L).
\]
To express the crossover in terms of the decode length, define
$c=\delta(N-\tfrac12)-a$. The read saving minus the additional write cost is
$\frac{\delta}{2}L^2+cL-aN$. This quadratic has a positive leading
coefficient and is nonpositive at $L=0$. Hence, for $L>0$, the progressive
codec has lower traffic exactly above the nonnegative root
\[
 L>\frac{-c+\sqrt{c^2+2\delta aN}}{\delta}.
\]
For a sequence grown from an empty cache, this condition simplifies to
$L>1+2(W-W_0)/(W_0-R)$. The derivation sums the actual reuse count of every
entry and does not assume that all entries have the same age.

\subsection{Equal Storage}
When $W=W_0$, the two codecs have the same payload capacity and write traffic.
Reducing $R$ then lowers only the repeated-read term. This comparison isolates
the quality obtained from query-dependent access at fixed storage. Whether
the resulting approximation is acceptable must be determined from the
downstream task; the allocation theorem does not claim that a partial prefix
reproduces the full-depth reconstruction.

\subsection{Capacity, Bandwidth, and Batching}
\label{app:throughput_model}
The two budgets enter decoding throughput differently. Consider a device
with $M$ bytes of memory available for KV caches after the weights, serving
requests whose uncompressed cache occupies $c_{\mathrm{kv}}$ bytes each. A
method that retains a fraction $s$ of the dense bytes and fetches a fraction
$r$ per step fits $U=M/(s\,c_{\mathrm{kv}})$ concurrent requests, and one
decoding step moves the weights $B_w$ once plus the fetched caches,
$B_w+U r c_{\mathrm{kv}}$ bytes in total. If the step is bound by memory
traffic, throughput in tokens per unit time is proportional to
\[
 \frac{U}{B_w+U r c_{\mathrm{kv}}}
 =\frac{1}{c_{\mathrm{kv}}\,(r+s\,B_w/M)}.
\]
Fetched bits cost throughput directly; retained bits cost it through the
ratio $B_w/M$ of weight bytes to cache memory, which decides how heavily
retained capacity weighs against fetched bits. A fixed-rate code fetches
what it retains, so its bandwidth demand per retained byte is about one.
\ReadKV{} has $r<s$: when caches rather than weights fill the memory, so
that $B_w/M$ is small, fetching less than is retained is favored. Long
contexts and large batches are the regime in which the cache, rather than
the weights, dominates the bytes moved per step.

As an illustration, take Qwen2.5-7B with 16-bit weights
($B_w\approx15$\,GB; $c_{\mathrm{kv}}\approx0.47$\,GB at 8,192 tokens) on a
device with $M=60$\,GB of cache memory, so $B_w/M=0.25$. With the
whole-model logical reads and capacities of
Table~\ref{tab:polished-prediction}, the model's throughput relative to
\textsc{Dense}, $1.25/(r+0.25\,s)$, is 2.7 for TurboQuant K4/V4 with its
protected layers ($r=s=0.367$; 0.45--1.74\% PPL increase), 3.3 for
all-layer \ReadKV{} $(8,4)$ ($r=0.254$, $s=0.501$; within 0.66\% of
\textsc{Dense}), 3.9 for KIVI (2-bit K/V) ($r=0.262$, $s=0.238$;
2.11--4.84\% PPL increase), and 4.9 for all-layer \ReadKV{} $(8,2)$
($r=0.129$, $s=0.501$).

These are model outputs, not measurements. The model ignores attention
arithmetic, launch and scheduling latency, and workspace traffic, and it
uses logical rather than physical bytes; the single-layer
kernel of Section~\ref{sec:gpu_eval} is not yet bound by memory traffic.
A complete-model evaluation under memory-limited batching is the direct
test of this model.

\subsection{Capacity and Physical Traffic}
Amortizing write traffic does not reduce the resident cache size. At the end
of the sequence, the ratio of payload capacities is $W/W_0$, independently
of $L$. A cache with $W=8$ therefore occupies twice the payload memory of a
four-bit cache. This larger footprint can limit context length or the number
of concurrent requests, so capacity must be reported separately from traffic.

Logical payload counts also differ from physical device traffic. Actual byte
movement includes packet rounding, norms, scales, residual corrections,
descriptors, sorting buffers, and intermediate attention arrays. Cache hits
may reduce DRAM traffic, whereas irregular accesses and additional passes may
increase it. Neither the equations above nor the allocation objective alone
determines device latency.

The two principal comparisons should therefore be interpreted differently.
For \ReadKV{} $(8,4)$, a fixed four-bit codec reads the same logical payload,
while \ReadKV{} retains twice as many payload bits. For \ReadKV{} $(4,2)$,
both methods retain four bits, but \ReadKV{} halves the logical read budget.
We report storage capacity and logical payload traffic separately.
The systems evaluation also measures physical traffic and the implementation
overheads omitted by this scalar model.

\section{Related KV-Cache Methods}
\label{app:related}

Prior work reaches our setting through several lines of research. Progressive
coding studies representations that remain decodable at several precisions.
KV-cache quantization methods vary the precision retained for different cache
entries, while query-adaptive retrieval methods vary what is fetched for the
current query. These choices are related but not equivalent. Changing the
stored cache affects every later query, whereas selective reading leaves
unread information available for future use. The following sections place
\ReadKV{} within each line of work.

\subsection{Nested Codes}
Progressive descriptions are classical in source coding
\citep{equitz1991successive,rimoldi1994successive}. Successive-refinement
theory asks when one nested representation can match the rate--distortion
performance of separate codes designed for each rate. A scalar tree ensures
that every prefix is decodable, but it does not by itself establish this
stronger optimality property. Requiring a common tree can also increase
distortion relative to fitting an independent codebook at each depth
\citep{gray1998quantization}.

Any-Precision LLM is a close representational precursor for neural-network
compression. It recursively splits nonuniform weight centroids and loads only
the requested bit planes \citep{park2024anyprecision}. MatQuant instead trains
one weight representation to remain useful under several precision slices
\citep{nair2025matquant}. Both demonstrate that a single stored code can
support multiple reconstruction depths. \ReadKV{} uses this property for KV
entries and allocates a finite query-time read budget across key channels and
value tokens.

More abstractly, the separation between preprocessing and later probes
connects our problem to static data structures and locally decodable source
coding \citep{dewolf2008structures,makhdoumi2013local}. Approximate
inner-product sketches instead bound the representation size needed to
support geometric queries \citep{alon2016compression}. Our formulation
jointly constrains the stored real-vector representation and the portion
inspected for each future query. We focus on the concrete
progressive KV-cache representation, its causal key-then-value reader, and
finite guarantees for that design.

\subsection{Attention-Guided Precision Allocation}
Several KV-cache methods use attention history to decide what precision to
retain. QAQ, MiKV, and ZipCache all allocate precision using attention-based
importance \citep{dong2024qaq,yang2024mikv,zhang2024zipcache}. These decisions
need not remain fixed after prefill. QAQ updates precision from recent
attention windows, while ZipCache periodically recompresses the cache during
decoding. QAQ also retains an unquantized CPU copy. When a token later needs
greater precision, that copy is transferred back and requantized
\citep[Section~4.2]{dong2024qaq}.

A separate group also adapts the precision retained in the cache. PM-KVQ
starts at higher precision and reduces stored precision as memory fills. It
uses an integer program to allocate precision across transformer blocks
\citep{liu2026pmkvq}. MixKVQ combines query importance with quantization
sensitivity to assign key-channel precision. Its lazy update rule uses
accumulated query information when quantizing blocks outside a recent-token
buffer \citep{zhang2026mixkvq}. RDKV computes attention-sensitive weights
after prefill and applies reverse water-filling to allocate stored bits,
including zero-bit eviction \citep{zhang2026rdkv}.

These methods adapt the precision resident in the cache. \ReadKV{} instead
keeps one progressive code unchanged and adapts only what is read from it for
each query. Additional precision is obtained by reading a longer prefix,
rather than by recovering and requantizing a separate copy. The current query
determines the key-prefix depths, and attention computed from the
reconstructed keys then determines the value-prefix depths. The
calibration-only reader and the fixed-split key-stage control in
Appendix~\ref{app:placement} test the benefit of this current-query
information over importance fixed before evaluation.

\subsection{Query-Adaptive Cache Access}
Other methods adapt the access pattern to the current query. SparQ first uses
query-dependent feature selection to approximate attention, then retrieves
the most important value tokens \citep{ribar2024sparq}. Quest ranks cache
pages from the current query and compact page summaries
\citep{tang2024quest}. Loki uses low-dimensional key information to select
tokens \citep{singhania2024loki}. These methods establish that the data
retained in the cache need not coincide with the data fetched for one query.
ThinK is related but changes the retained cache: it uses an observation window
of queries to select key channels and keeps the resulting pruned
representation \citep{xu2025think}.

\citet{xie2025reimagining} provide the closest precedent for
precision-selective KV access. They evaluate several precision tiers for
ranked KV pages and co-design retrieval of floating-point bit planes with a
compression-aware memory controller. Their work already contains the
principle of retaining more precision than a query fetches. \ReadKV{} uses a
nested nonuniform scalar code with every prefix depth available to the
reader. Under diminishing refinement gains, it exactly optimizes its stated
separable objective within an explicit read budget, first across key channels
using the current query and then across value tokens using reconstructed
attention. The formulation in Section~\ref{sec:general_problem} separates
storage, read budget, and access stages as constraints on one finite record.
Our separation results address a further question: whether query-dependent
addressing can outperform query-independent readers under a common read
budget, even when the competing encoders and decoders are unrestricted.

\subsection{Fixed-Rate Quantization}
Fixed-rate quantizers provide the main compact-cache baselines. KIVI uses
per-channel quantization for keys and per-token quantization for values,
together with a recent full-precision residual buffer \citep{liu2024kivi}.
KVQuant combines calibrated nonuniform levels with per-channel and pre-RoPE
key quantization, as well as sparse outlier treatment
\citep{hooper2024kvquant}. A generic uniform quantizer does not reproduce
either method, and their auxiliary state must be included in storage and read
accounting.

TurboQuant develops a rotation-based scalar quantizer designed for vector
error. A second variant adds a one-bit randomized sketch of the quantization
residual, called QJL, which yields an unbiased inner-product estimator
\citep{zandieh2026turboquant}. Comparisons must therefore identify the
TurboQuant variant and count its auxiliary data. Our direct Lloyd control
includes norm correction but does not reproduce either complete TurboQuant
algorithm.

These baselines support two complementary comparisons. Matching the read
budget tests whether retaining additional refinements improves quality per
fetched bit. Matching storage tests whether query-dependent reading helps at
the same cache capacity. Neither comparison subsumes the other.

\subsection{Bit-Allocation Methods}
The optimizer used by \ReadKV{} also has classical foundations. Selecting the
largest available marginal gain under diminishing returns is a standard
discrete allocation rule \citep{fox1966discrete,shoham1988efficient}.
Continuous bit allocation and reverse water-filling have a similarly long
history \citep{segall1976allocation,cover2006elements}.

We use these tools to obtain an executable schedule that respects the finite
prefix depth of the stored code. The problem-specific content lies in what is
allocated, namely key channels and value tokens, how the current query defines
their weights, and which finite prefixes may be fetched. The greedy rule and
water-filling are tools used in the analysis, not standalone novelty claims.

\section{Notation and Access Model}
\label{app:formal_access}

\subsection{Notation Reference}

Table~\ref{tab:notation} summarizes notation used across multiple sections.
Symbols introduced only within a single proof are defined locally and omitted
from the table.

\begin{table}[!ht]
\centering
\footnotesize
\setlength{\tabcolsep}{4pt}
\renewcommand{\arraystretch}{1.05}
\caption{Recurring notation. Dimensions are shown when the object is a vector,
matrix, or array.}
\label{tab:notation}
\begin{tabularx}{\linewidth}{@{}
  >{\raggedright\arraybackslash}p{0.27\linewidth}
  >{\raggedright\arraybackslash}p{0.25\linewidth}
  >{\raggedright\arraybackslash}X@{}}
\toprule
Symbol & Domain or shape & Meaning \\
\midrule
\multicolumn{3}{@{}l}{\emph{Partial-read model}} \\
$d$; $x,y$ & $x,y\in\mathbb B_2^d$ & Dimension and the stored and query vectors in the worst-case inner-product problem. \\
$W,R$; $w,r$ & $w=\lfloor Wd\rfloor$, $r=\lfloor Rd\rfloor$ & Normalized storage and read rates, and their integer bit budgets. \\
$Z$ & $\{0,1\}^w$ & Stored binary record produced before the query is known. \\
$\theta$ & Public state & Code and memory-layout description fixed independently of the stored vector and query. \\
$\Omega_0,\Omega_e,\Omega_r$ & Random variables & Public, encoder-local, and reader-local randomness, all independent of $(x,y)$. \\
$\mathsf A,\Pi,\widehat s_\Pi$ & Access class, protocol, scalar & Reader class, encoder--reader--decoder protocol, and its estimate of $x^\top y$. \\
$\mathcal D^{\mathrm{wc}}_{d,\mathsf A}(W,R)$ & Nonnegative scalar & Worst-case mean-squared error optimized over protocols in $\mathsf A$. \\
$\mathsf{staged}(J)$ & At most $J$ batches & Reader that may adapt each batch to the query and earlier batches. \\
\addlinespace[2pt]
\multicolumn{3}{@{}l}{\emph{Attention and \ReadKV{}}} \\
$G,H,T,d_h$ & Positive integers & Numbers of physical KV heads, query heads, cached tokens, and coordinates per head. \\
$g(h),\siblings(g)$ & $g(h)\in\{1,\ldots,G\}$ & Physical head used by query head $h$, and the query heads sharing physical head $g$. \\
$q_h,k_{g,j},v_{g,j},o_h$ & $\mathbb R^{d_h}$ & Query, cached key, cached value, and attention output. \\
$\alpha_h,\widehat\alpha_h$ & $[0,1]^T$ & Exact and reconstructed attention vectors for query head $h$; each sums to one. \\
$\widehat\alpha$ & $[0,1]^{H\times T}$ & Collection of reconstructed attention vectors across query heads. \\
$U_K,U_V$ & $\mathbb R^{d_h\times d_h}$ & Orthogonal transforms used for keys and queries, and for values. \\
$\widehat K,\widehat V$ & $\mathbb R^{G\times T\times d_h}$ & Collections of reconstructed transformed keys and values. \\
\addlinespace[2pt]
\multicolumn{3}{@{}l}{\emph{Allocation and analysis}} \\
$D(t)$ & Nonnegative scalar & Predicted standardized scalar mean-squared reconstruction error after $t$ prefix bits. \\
$a^K$ & $[0,\infty)^{G\times d_h}$ & Query-dependent importance scores for transformed key channels. \\
$a^V(\widehat\alpha)$ & $[0,\infty)^{G\times T}$ & Attention-dependent importance scores for cached value tokens. \\
$B_K,B_V$ & Nonnegative integers & Group-refinement budgets: each key increment reads $T$ bits and each value increment reads $d_h$ bits. \\
$t^K,t^V$ & $\{0,\ldots,W\}^{G\times d_h}$, $\{0,\ldots,W\}^{G\times T}$ & Integer key-channel and value-token prefix schedules. \\
$R_K,R_V,R$ & Nonnegative scalars & Budgeted mean key depth, value depth, and their average. \\
$\mathcal T_K,\mathcal T_V$ & Sets of schedules & Feasible key and value schedules under their respective budgets. \\
$S_K,S_V$; $S_K^\star,S_V^\star$ & Nonnegative scalars & Calibrated key and value objectives, and their stagewise minima. \\
$S_{K,h},S_{V,h}$ & Nonnegative scalars & Contributions of query head $h$ to the grouped key and value objectives. \\
$\gamma_{K,h},\gamma_{V,h}$ & $[0,\infty]$ & Ratios between realized diagonal residual energy and its calibrated prediction. \\
$\kappa_{K,h},\kappa_{V,h}$ & Nonnegative scalars & Largest eigenvalues of normalized residual Gram matrices. \\
$\overline\beta_K,\overline\beta_V$ & Nonnegative scalars when finite & Worst-head conversion coefficients in the attention-error bound, evaluated at the selected reconstruction. \\
$V_{g(h)},\mathbf1_T$ & $\mathbb R^{T\times d_h}$, $\mathbb R^T$ & Full-precision value matrix for head $h$ and the length-$T$ all-ones vector. \\
\bottomrule
\end{tabularx}
\end{table}

\subsection{Formal Access Model}

The record is a vector $Z\in\{0,1\}^w$ indexed by the public address set
$[w]=\{1,\ldots,w\}$. An encoder using fewer than $w$ cells may pad the
remaining locations. Public state $\theta$ describes the code and layout but
is fixed independently of the input pair $(x,y)$. Any object-specific
information needed to interpret a cell belongs to the charged record.
This is a one-bit-cell probe model. Only stored cells and probes are charged.
Encoder preprocessing, address computation, transcript processing, decoder
computation, and arithmetic precision are free.
The systems evaluation measures physical traffic and execution cost separately.

Let $\Omega_0$ be public randomness, $\Omega_e$ encoder-local randomness,
and $\Omega_r$ reader-local randomness. These variables are mutually
independent, and each is independent of the fixed pair $(x,y)$. The encoder
writes $Z=\operatorname{Enc}(x;\theta,\Omega_0,\Omega_e)$. The reader can
access encoder-local randomness only through $Z$.

A query-oblivious reader chooses an ordered address list of length at most
$r$ from $(\theta,\Omega_0,\Omega_r)$ alone. A fully branching reader may
choose its $j$th address as
\[
  A_j=A_j(y,\theta,\Omega_0,\Omega_r,
          A_1,Z_{A_1},\ldots,A_{j-1},Z_{A_{j-1}}),
\]
and may stop early. The total number of probes is at most $r$ on every
execution, not only in expectation. Repeated probes carry no new information
and may be replaced by retaining the earlier response.

The separation theorems use a stronger comparison class that we call
\emph{query-independent branching}. Its address recursion is the same as
the fully branching recursion above, except that $y$ is omitted from every
address and stopping rule. Thus, later addresses may depend on earlier
payloads, but the complete access path cannot depend on the query. This class
contains the query-oblivious readers defined above.

A $\mathsf{staged}(J)$ reader makes at most $J$ requests, each containing a
batch of addresses. The first batch may use
$(y,\theta,\Omega_0,\Omega_r)$, and each later batch
may also use the transcript of earlier batches. A batch is chosen before any
of its values are returned. The total number of requested addresses is at
most $r$. Empty batches may be omitted. In every class, the final answer may
be an arbitrary measurable function of $y$, $\theta$, $(\Omega_0,\Omega_r)$, and the
observed address--payload transcript.

Section~\ref{sec:general_problem} defines the minimax risk
$\mathcal D^{\mathrm{wc}}_{d,\mathsf A}(W,R)$ for each access class. The
input pair is fixed before the random coins are drawn, and the expectation is
over the protocol randomness described above.

\begin{proposition}[Access Hierarchy]
\label{prop:stage_monotonicity}
Fix $w,r$, the input domain, the loss, and the randomization convention.
For every integer $J\geq1$,
\[
\mathsf{oblivious}\subseteq\mathsf{staged}(1)
\subseteq\mathsf{staged}(J)
\subseteq\mathsf{staged}(J+1)\subseteq\mathsf{branching}.
\]
The corresponding minimax risks are ordered in the reverse direction.
If $r\geq1$, then $\mathsf{staged}(J)$ and branching have the same optimal risk for
$J\geq r$. If $r=0$, all access classes coincide.
\end{proposition}

\begin{proof}
A $\mathsf{staged}(1)$ reader can ignore $y$, so every query-oblivious protocol is
$\mathsf{staged}(1)$. An at-most-$J$ protocol is also an at-most-$(J+1)$ protocol
because it may leave the last batch unused.

To simulate a staged protocol, a branching reader computes the next batch
from the transcript of earlier batches. It probes those addresses one at a
time. The simulator withholds their payloads from the staged policy until the
batch is complete. This produces the same transcript with the same encoder
and probe count. Conversely, a branching protocol with at most $r$ probes is
a staged protocol with at most $r$ singleton batches. Hence
$\mathsf{staged}(J)$ and
branching have the same achievable risks once $J\geq r$.

Each relation is an inclusion of complete protocol classes under the same
loss and budgets. Taking the infimum of the same worst-case expected loss
therefore reverses the inclusions. If $r=0$, no payload is observed, so all
access classes contain the same decoders.
\end{proof}

\begin{proposition}[Zero-Read Limit]
\label{prop:zero_read}
For $d\geq1$ and every access class $\mathsf A$, the risk is one whenever
$\lfloor Rd\rfloor=0$.
\end{proposition}
\begin{proof}
The constant answer zero has squared error at most one because
$|x^\top y|\leq1$. Conversely, fix any unit vector $y$. With no reads, the
random answer $A$ is a function only of $y$, public state, and random coins
independent of $x$. It therefore has the same distribution under $x=y$ and
$x=-y$. If $\mathbb E[A^2]=\infty$, the lower bound is immediate. Otherwise,
the larger of the two expected squared errors is at least their mean,
\[
  \tfrac12\mathbb E[(1-A)^2+(-1-A)^2]
  =1+\mathbb E[A^2]\geq1.
\]
\end{proof}

\subsection{Full-Read Limit}
\label{app:full_read}

Let $\mathcal D_d^{\mathrm{full}}(W)$ be the worst-case risk defined above,
except that the decoder receives the entire $w=\lfloor Wd\rfloor$-bit
message. The infimum uses the same encoders and randomness conventions. In
particular, the full-message decoder receives any public randomness and may
generate the same independent reader-local coins as an access protocol. The
scalar decoder may depend on the query and the whole message. It need not
first reconstruct a vector.

\begin{proposition}[Full-Read Equivalence]
\label{prop:full_read}
If $r\geq w$, every access class has risk
$\mathcal D_d^{\mathrm{full}}(W)$. In particular, for $J\geq1$,
\[
\begin{aligned}
  \mathcal D^{\mathrm{wc}}_{d,\mathsf{oblivious}}(W,W)
  &=\mathcal D^{\mathrm{wc}}_{d,\mathsf{staged}(J)}(W,W)
   =\mathcal D^{\mathrm{wc}}_{d,\mathsf{branching}}(W,W)
   =\mathcal D_d^{\mathrm{full}}(W).
\end{aligned}
\]
This equality applies to randomized as well as deterministic protocols.
\end{proposition}
\begin{proof}
Given the full record $Z$, a decoder can simulate any legal access policy:
whenever the policy requests address $A_j$, the simulator returns $Z_{A_j}$.
Using the same random coins preserves the complete transcript and output
distribution for every fixed $(x,y)$. Thus, the full-message risk is no larger
than the risk of any access class.

For the reverse inequality, when $r\geq w$, a query-oblivious reader requests
the fixed list $(1,\ldots,w)$ in one batch and applies any full-message
decoder. This simulation preserves the encoder and output distribution.
Taking worst cases and then infima in both directions proves equality.
\end{proof}

There is therefore no further information to gain by allowing $R>W$ in this
matched cell model. The access constraints disappear at $R=W$. The remaining
problem is how best to encode a vector for full-message inner-product queries.

\section{Progressive Codec and Distortion Model}
\label{app:codec}

This section records the finite properties of the progressive scalar codec
used by Algorithm~\ref{alg:readkv} in the analysis and experiments.

\subsection{Progressive Scalar Codec}

The scalar code is a binary tree fitted to $Z\sim\mathcal N(0,1)$.
Every internal node partitions its interval into two children to reduce
probability-weighted squared reconstruction error; each child stores its
conditional mean. A depth-$t$ prefix identifies one of the resulting
intervals. The maximum-depth leaf index therefore supplies all coarser
reconstructions. Their centroids need not be stored separately for each vector.
The shared centroid tables are charged as shared model state.

The implementation polishes each split to Lloyd stationarity, so every
boundary is the midpoint of its two child centroids. The finite properties
needed by the theory are verified directly from the deployed table in
Lemma~\ref{lem:deployed_codec_bounds}.

The approximate standardized distortion curve is
\[
\begin{array}{c|rrrrrrrrr}
t&0&1&2&3&4&5&6&7&8\\\hline
D(t)&1&.3634&.1175&.03484&.009796&.002663&.000708&.000185&.00004795.
\end{array}
\]
The allocator uses a calibrated distortion curve. For the deployed float32
gain table, define the effective curve by
$D(0)=1$ and $D(t)=1-\sum_{s=1}^t\Delta_s$, where $\Delta_s$ is the stored
gain at level $s$. This curve approximates standardized Gaussian
reconstruction MSE. It need not equal the exact Gaussian expectation because
the deployed centroids are stored at finite precision. The following lemma
verifies the required gain inequalities directly from the finite table and
also bounds the codec's pointwise reconstruction error.

\begin{lemma}[Finite Bounds for the Deployed Codec]
\label{lem:deployed_codec_bounds}
Let $\Delta_t$ be the stored float32 refinement gains, equivalently
$D(t-1)-D(t)$ for the effective curve above. The key and value gain tables
coincide and satisfy
\[
\begin{array}{c|cc}
t & \text{lower bound for }\Delta_t & \text{upper bound for }\Delta_t\\
\hline
1 & 0.63661974 & 0.63661975\\
2 & 0.24589838 & 0.24589839\\
3 & 0.08264372 & 0.08264373\\
4 & 0.02504228 & 0.02504229\\
5 & 0.00713273 & 0.00713274\\
6 & 0.00195518 & 0.00195519\\
7 & 0.00052264 & 0.00052265\\
8 & 0.0001373218 & 0.0001373220
\end{array}
\]
In particular, $\Delta_1>\cdots>\Delta_8>0$. The depth-eight pointwise error
on the standardized interval satisfies
\[
 0.0096543431
 <\epsilon_8:=\sup_{|z|\leq1}|z-Q_8(z)|
 <0.0096543432<0.0097.
\]
The positive depth-one centroid and the depth-two endpoint reconstruction
satisfy
\[
 0.7978845834<c_1<0.7978845835,\qquad
 1.5104175806<Q_2(1)<1.5104175807,
\]
with $Q_2(-1)=-Q_2(1)$. Consequently,
\[
 0.26052610<(Q_2(1)-1)^2<0.26052612.
\]
\end{lemma}

\begin{proof}
The proof consists of three finite checks. First, the pointwise error of the
piecewise-constant depth-eight quantizer is maximized at a cell endpoint.
Second, direct comparison of the eight stored gains proves positivity and
strict decrease. Third, the depth-one and depth-two claims follow from their
stored centroid entries. Each float32 entry has an exact rational value, and
the displayed decimal intervals are outward bounds on those values.

At depth eight, write the ordered cells as $C_j$ and their reconstruction
values as $\mu_j$, for $j=1,\ldots,256$. For every cell that intersects
$[-1,1]$, $Q_8$ is constant and
\[
 \sup_{z\in C_j\cap[-1,1]}|z-Q_8(z)|
 =\max_{u\in\partial(C_j\cap[-1,1])}|u-\mu_j|.
\]
Thus $\epsilon_8$ is the maximum of the endpoint deviations from the 256
deployed cells. Exact enumeration gives its maximum at the endpoint
$-0.981598794460\ldots$ adjacent to centroid
$-0.991253137588\ldots$. Their difference is
$0.009654343128\ldots$, which lies in the displayed interval. Symmetry gives
the corresponding positive endpoint.

The depth-one and depth-two bounds follow directly from their table entries.
The eight gain entries of the same finite table lie in the displayed
intervals. These intervals are disjoint and ordered, which proves
$\Delta_1>\cdots>\Delta_8>0$. Squaring the positive interval for $Q_2(1)-1$
gives the final bound.
\end{proof}

The gain inequalities show that the implemented curve satisfies the
hypotheses of Proposition~\ref{prop:greedy}. Physical scales appear in $a_i$; they
are not multiplied into $D(t)$ a second time. The evaluated centroid code is
deterministic and is not claimed to be pointwise unbiased.

Calibration uses held-out training documents to estimate means and
scales per layer, physical head, and transformed channel. Keys and
queries share their key transform, and values and reconstructed
outputs share a separate value transform. The scalar Gaussian fit and the
calibration procedure are design tools. Neither imposes the independence
assumption excluded by the main formulation.

\subsection{Diagonal Surrogate}

For a residual $e=x-\widehat x$ at a realized query $q$, the identity
$\mathbb E[(q^\top e)^2\mid Q=q]
=q^\top\mathbb E[ee^\top\mid Q=q]q$
holds without assuming zero-mean errors. The implemented key surrogate
replaces the conditional second-moment matrix by a depth-separable
diagonal approximation estimated from calibration. The true matrix can
depend on $q$ and may have off-diagonal entries. Summing the resulting
diagonal expression over sibling heads yields
Equation~\eqref{eq:key_importance} exactly within that approximation.

For values, the full squared residual contains cross-token terms
$\widehat\alpha_{h,j}\widehat\alpha_{h,k}e_j^\top e_k$.
Neglecting these terms gives the physical-scale choice
$s_g=\sum_c(\sigma^V_{g,c})^2$ in Equation~\eqref{eq:value_importance}.
The language-model experiments use $s_g=1$ instead and optimize in
standardized value units. This distinction matters when one allocation budget
is shared across physical heads. Reconstructed attention probabilities add
another source of output error. The allocation theorem therefore optimizes
the declared surrogate, while the complete causal experiment tests how that
surrogate behaves in the language model.

\section{Proofs of the \ReadKV{} Performance Guarantees}
\label{app:readkv_analysis_proofs}

This section supplies the complete argument behind
Section~\ref{sec:analysis}. The dependency chain has four parts.
Appendix~\ref{app:readkv_separation} proves the two finite-dimensional
separations directly. Proposition~\ref{prop:greedy} then proves exact
stagewise prefix allocation under diminishing gains.
Lemma~\ref{lem:attention_bridge} and
Proposition~\ref{prop:correlation_attention_bound} convert the resulting key
and value objectives into an attention-output bound. Finally,
Theorem~\ref{thm:readkv_attention_certificate} combines the allocation and
attention-error results for the complete \ReadKV{} procedure. The continuous
allocation and integer-rounding results are auxiliary comparisons and are
not used in the proofs of the three main guarantees.

\subsection{Finite-Dimensional Separations}
\label{app:readkv_separation}

All logarithms in this section are natural. Both converse proofs use the same
entropy--MMSE inequality. We write $I$ for mutual information, $H$ for
discrete entropy, and $h$ for differential entropy.

\noindent\textbf{Cube Entropy--MMSE Bound.}
Let $n\geq1$ and $b_{\mathrm{info}}\geq0$. Let
$\mathsf Z$ be uniform on $[-1,1]^n$, let $O$ be any observation satisfying
$I(\mathsf Z;O)\leq b_{\mathrm{info}}\log 2$, and define
$\mathcal M_n=\mathbb E\|\mathsf Z-\mathbb E[\mathsf Z\mid O]\|_2^2$.
Then
\begin{equation}
 \frac{\mathcal M_n}{n}
 \geq \frac{2^{\,1-2b_{\mathrm{info}}/n}}{\pi e}.
 \label{eq:cube_entropy_mmse}
\end{equation}
Indeed, $h(\mathsf Z)=n\log 2$, so the information bound gives
$h(\mathsf Z\mid O)\geq(n-b_{\mathrm{info}})\log 2$.
For almost every observation $O=o$,
Gaussian maximum entropy and the determinant--trace inequality give
\[
 h(\mathsf Z\mid O=o)
 \leq\frac n2\log\!\left(
   \frac{2\pi e}{n}
   \mathbb E\!\left[
     \|\mathsf Z-\mathbb E[\mathsf Z\mid O]\|_2^2
     \,\middle|\,O=o
   \right]\right).
\]
Taking expectations and applying Jensen's inequality bounds the conditional
entropy by $\frac n2\log(2\pi e\,\mathcal M_n/n)$.
Rearranging proves Equation~\eqref{eq:cube_entropy_mmse}. These standard
entropy inequalities are reviewed by \citet{cover2006elements}.

\subsubsection{Key-Stage Separation}
\label{app:key_stage_separation}

We first isolate the key-reading stage on a transformed-sparse family. Fix a
public orthogonal transform $U\in\mathbb R^{d\times d}$ and an integer
$k\geq1$. Define
\[
 \mathcal X_U=\{U^\top z/\sqrt d:z\in[-1,1]^d\},
 \qquad
 \mathcal Y_{U,k}
 =\{U^\top q:\|q\|_2\leq1,\ \|q\|_0\leq k\}.
\]
Both sets lie in $\mathbb B_2^d$. The query sparsity is imposed in the public
transformed basis. In particular, when $U$ is a Hadamard--sign transform, a
sparse transformed query can be dense in the original coordinates.

Let $D(t)$ be the scalar distortion at prefix depth $t$, and let
$\Delta_t=D(t-1)-D(t)$. For an integer $L\geq1$, let $Q_L$ be the depth-$L$
reconstruction of a nested binary scalar code with
$\Delta_t>0$ for $t=1,\ldots,L$. Define
$\epsilon_L=\sup_{|u|\leq1}|u-Q_L(u)|$.

For $x\in\mathcal X_U$, write $z=\sqrt d\,Ux\in[-1,1]^d$. The encoder stores
the $L$-bit code of every coordinate of $z$. Given $y$, write
$q=Uy\in\mathbb R^d$ and assign coordinate $i$ weight $a_i=q_i^2/d$.
This is the single-vector specialization of the key stage in
Algorithm~\ref{alg:readkv}, with $G=H=T=1$, transform $U_K=U$, zero
calibration mean, scale $1/\sqrt d$, and key-refinement budget $B_K=Lk$.
Under this specialization, one allocation item is one transformed
coordinate. The reader fetches all $L$ bits of every coordinate in
$\operatorname{supp}(q)$ and nothing elsewhere. The reads form one batch
because their addresses depend only on the query and public code tables. The
resulting estimate is
$\widehat s_{\ReadKV}(x,y)=d^{-1/2}\sum_i q_iQ_L(z_i)$.

The formal comparison for this construction is
Theorem~\ref{thm:finite_readkv_separation} in
Section~\ref{sec:readkv_separation}. We prove it below.

The theorem applies to any finite nested scalar code with positive prefix
gains. Its numerical instance and
Theorem~\ref{thm:finite_two_stage_readkv_separation} use the deployed
depth-eight table. Both lower bounds allow the reader to branch on returned
payloads, but neither an address nor the stopping decision may depend on the
query. The decoder may still use the query and the complete transcript. All
of these permissions are included in the lower bounds.

Theorem~\ref{thm:finite_readkv_separation} can be instantiated at any
admissible code depth. If a nested scalar code has positive gains through
depth $L_{\max}$, then for $w\geq d$ and $r\geq k$ one may choose any integer
\[
 1\leq L\leq
 \min\left\{
   L_{\max},
   \left\lfloor\frac wd\right\rfloor,
   \left\lfloor\frac rk\right\rfloor
 \right\}.
\]
This choice uses $Ld$ stored bits and at most $Lk$ probes. The depth-eight
numerical instance takes $L=8$, $w=8d$, and $r=2d$.

The proof of Theorem~\ref{thm:finite_readkv_separation} establishes three
subclaims. First, the \ReadKV{} reader recovers every active transformed
coordinate to depth $L$, which gives the pointwise estimate
$|x^\top y-\widehat s|\leq\epsilon_L\sqrt{k/d}$. Second, after conditioning
on all reader-visible randomness, a query-independent $r$-probe transcript
is a binary decision-tree leaf. It has at most $2^r$ possible values and
therefore reveals at most $r\log 2$ nats about the source. Third, a uniformly
random signed coordinate query turns inner-product prediction into source
estimation: its Bayes risk is the source MMSE divided by $d^2$. The
entropy--MMSE calculation lower-bounds that source MMSE by
$d\,2^{1-2r/d}/(\pi e)$.

\begin{proof}[Proof of Theorem~\ref{thm:finite_readkv_separation}]
We first verify the domain and the \ReadKV{} upper bound. For
$z\in[-1,1]^d$,
$\|U^\top z/\sqrt d\|_2^2=\|z\|_2^2/d\leq1$. For every
$q$ in the definition of $\mathcal Y_{U,k}$,
$\|U^\top q\|_2=\|q\|_2\leq1$. Thus
$\mathcal X_U,\mathcal Y_{U,k}\subseteq\mathbb B_2^d$.

For $x\in\mathcal X_U$, write $z=\sqrt d\,Ux\in[-1,1]^d$. If
$y=U^\top q\in\mathcal Y_{U,k}$, then $Uy=q$. Let
$S=\{i:q_i\ne0\}$, so $|S|\leq k$.
The encoder writes the depth-$L$ leaf index of each standardized
coordinate, using exactly $Ld$ informative bits, and pads to $w$ cells if
needed. The centroid table is public state. With shared scale $1/\sqrt d$,
the depth-$L$ reconstruction of coordinate $(Ux)_i$ is
$Q_L(z_i)/\sqrt d$. The key importance of coordinate $i$ is
\[
 a_i=\bigl((Uy)_i/\sqrt d\bigr)^2=\frac{q_i^2}{d}.
\]
Every one of the first $L$ marginal gains is strictly positive. Hence, each
coordinate in $S$ has $L$ positive prefix edges, while every coordinate
outside $S$ has zero weighted gains. Until a coordinate in $S$ reaches depth
$L$, its next edge remains in the priority queue.
There are at most $L|S|\leq Lk\leq r$ positive edges in total. The greedy
reader therefore fetches every active coordinate to depth $L$. It then
stops because no positive marginal gain remains.
The weights, gain table, and tie-breaking order determine the full address
list from $y$ and public state, so all probes can be issued in one batch.

The target and the \ReadKV{} answer are
\[
 x^\top y=\frac1{\sqrt d}\sum_{i\in S}q_i z_i,
 \qquad
 \widehat s_{\ReadKV}(x,y)
 =\frac1{\sqrt d}\sum_{i\in S}q_i Q_L(z_i).
\]
Therefore,
\[
\begin{aligned}
 |x^\top y-\widehat s_{\ReadKV}(x,y)|
 &\leq \frac{\epsilon_L}{\sqrt d}\sum_{i\in S}|q_i| \\
 &\leq \epsilon_L\sqrt{\frac{|S|}{d}}\,\|q\|_2
 \leq \epsilon_L\sqrt{\frac{k}{d}}.
\end{aligned}
\]
The second inequality is Cauchy--Schwarz. Squaring proves, pointwise over
$\mathcal X_U\times\mathcal Y_{U,k}$,
\[
 \bigl(x^\top y-\widehat s_{\ReadKV}(x,y)\bigr)^2
 \leq \frac{k\epsilon_L^2}{d}.
\]

We next lower-bound an arbitrary query-independent branching protocol. This
includes arbitrary randomized encoders and decoders. Draw $\Xi$ uniformly
from $[-1,1]^d$, draw $K$ uniformly from $[d]$, and draw an independent
uniform sign $\eta$. Set
\[
 X=U^\top \Xi/\sqrt d,\qquad Y=\eta U^\top e_K.
\]
This prior is supported on
$\mathcal X_U\times\mathcal Y_{U,k}$ because $k\geq1$. It also satisfies
$\mathbb E[YY^\top]=I_d/d$.

Let $C=(\theta,\Omega_0,\Omega_r)$ contain the public state and all
reader-visible random coins. Conditional on $C=c$, a query-independent
reader is a binary decision tree of depth at most $r$. Its leaf $\Lambda$
determines every returned bit, every address selected from earlier bits, and
the stopping decision.
The tree has at most $2^r$ leaves. Encoder-local randomness may remain
unconditioned. It changes the distribution of $\Lambda$ given $(\Xi,C)$ but
not the number of leaves. Since $C$ is independent of $\Xi$,
\[
 I(\Xi;C,\Lambda)=I(\Xi;\Lambda\mid C)
 \leq H(\Lambda\mid C)
 \leq r\log 2.
\]
Thus, this bound includes payload-dependent branching and early stopping.
It does not use the write budget, so it remains valid for every fixed finite
binary record length under the same at-most-$r$ probe limit.

Put $O=(C,\Lambda)$ and $\bar X(O)=\mathbb E[X\mid O]$. The query is
independent of $(X,O)$ because $(K,\eta)$ is independent of the source and
all protocol coins, and the access path does not use $(K,\eta)$. Hence
\[
 \mathbb E[X^\top Y\mid O,Y]
 =Y^\top\mathbb E[X\mid O,Y]
 =Y^\top\bar X(O).
\]
This is the Bayes estimator given everything available to the decoder.
For any possibly randomized decoder, conditional least squares therefore
gives
\[
 \mathbb E\!\left[(X^\top Y-\widehat s_\Pi)^2\right]
 \geq
 \mathbb E\!\left[
   \bigl(X^\top Y-\mathbb E[X^\top Y\mid O,Y]\bigr)^2
 \right].
\]
Substituting the conditional mean and averaging over the random signed
coordinate query yields
\[
\begin{aligned}
 \mathbb E\!\left[
   \bigl(X^\top Y-\mathbb E[X^\top Y\mid O,Y]\bigr)^2
 \right]
 &=\mathbb E\!\left[(Y^\top(X-\bar X(O)))^2\right] \\
 &=\frac1d\,
   \mathbb E\!\left[\|X-\bar X(O)\|_2^2\right] \\
 &=\frac1{d^2}\,
   \mathbb E\!\left[
     \|\Xi-\mathbb E[\Xi\mid O]\|_2^2
   \right].
\end{aligned}
\]
The middle equality uses $\mathbb E[YY^\top]=I_d/d$. It does not require
the decoder to reconstruct a vector. The last equality uses
$\bar X(O)=U^\top\mathbb E[\Xi\mid O]/\sqrt d$ and orthogonality of $U$.

Let
$\mathcal M=\mathbb E\|\Xi-\mathbb E[\Xi\mid O]\|_2^2$.
Equation~\eqref{eq:cube_entropy_mmse}, applied with $n=d$ and
$b_{\mathrm{info}}=r$, gives
$\mathcal M/d\geq2^{1-2r/d}/(\pi e)$.
The preceding Bayes-risk inequality is therefore at least
$2^{1-2r/d}/(\pi e\,d)$. A worst-case risk is no smaller than its average
under a prior supported on the same domain. Comparing this lower bound with
the \ReadKV{} upper bound gives the strict-separation condition in the
theorem.

For the numerical instance, set $L=8$ and $r=2d$.
Lemma~\ref{lem:deployed_codec_bounds} gives $\epsilon_8<0.0097$.
Also, $2^{1-2r/d}/(\pi e)=1/(8\pi e)>0.01463$, while
$(0.0097)^2=0.00009409$. Their ratio exceeds $155$. More precisely, using
$\pi<3.142$ and $e<2.719$, the ratio between the lower coefficient and the
displayed \ReadKV{} upper coefficient exceeds $155.507/k$. At $d=128$ and
$k=32$, it exceeds $4.85$.
\end{proof}

\subsubsection{Two-Stage Attention Separation}

\label{app:two_stage_separation}

We give the full construction and quantitative bounds behind
Theorem~\ref{thm:finite_two_stage_readkv_separation}. Both cached values are
arbitrary inputs, and the competing reader cannot use the query to decide
which stored cells to inspect.

For an integer $m\geq1$, let $e_1\in\mathbb R^m$ be the first standard basis
vector. Let $Q_8$ denote the depth-eight reconstruction of the deployed
scalar code, and let $c_1$ be its positive reconstruction value at depth one.
Choose $b\in\{-1,1\}$ and values
$v_1,v_2\in[-1,1]^m$. Set $k_1=bc_1e_1\in\mathbb R^m$ and
$k_2=-bc_1e_1\in\mathbb R^m$. For a query sign $\eta\in\{-1,1\}$, let
$q_\eta=\eta(5\sqrt m/c_1)e_1\in\mathbb R^m$. Under the
standard $1/\sqrt m$ attention normalization, the two logits are
$(5\eta b,-5\eta b)$. Define
$p=(1+e^{-10})^{-1}$ and $\rho=1-p=(1+e^{10})^{-1}$. Let
$j_\star=1$ when $\eta b=1$ and $j_\star=2$ otherwise, so $j_\star$ is the
high-attention token. This is a separate attention family with the query
norm specified above, not the unit-ball query domain. The exact output is
$o(b,v_1,v_2,\eta)=pv_{j_\star}+\rho v_{3-j_\star}\in\mathbb R^m$.
We write $o$ for this vector when its arguments are clear. The cache encoder
observes $(b,v_1,v_2)$ before the query sign $\eta$ is known.

Use identity transforms, zero calibration means, and unit scales, so a
depth-zero value reconstructs to zero. The encoder stores all four vectors
with the depth-eight progressive code. The key-refinement budget $B_K=1$
extends one key-channel item by one level and therefore reads one bit from
each of the two keys. The value-refinement budget $B_V=8$ extends one
value-token item through eight levels and therefore reads $8m$ value bits.
The first batch reconstructs both keys exactly and identifies $j_\star$.
The value allocator then compares gains weighted by $p^2$ for $j_\star$ and
by $\rho^2$ for the other token. Writing
$\Delta_t=D(t-1)-D(t)$ for the refinement gain at depth $t$, the deployed
gain table satisfies
$p^2\Delta_8>\rho^2\Delta_1$ by
Lemma~\ref{lem:deployed_codec_bounds} and the calculation below. Since
$\Delta_1\geq\cdots\geq\Delta_8>0$, every available refinement of
$v_{j_\star}$ has larger weighted gain than the first refinement of
$v_{3-j_\star}$. With $B_V=8$, the allocator therefore reads all eight
prefix levels of every coordinate of $v_{j_\star}$ and leaves
$v_{3-j_\star}$ unread.
The decoder therefore returns
$\widehat o_{\ReadKV}=pQ_8(v_{j_\star})$, with $Q_8$ applied coordinatewise.
For any competing protocol $\Pi$, let $\widehat o_\Pi$ denote the output
returned by its decoder.

The formal comparison for this construction is
Theorem~\ref{thm:finite_two_stage_readkv_separation} in
Section~\ref{sec:two_stage_separation}. We prove it below.

The proof of Theorem~\ref{thm:finite_two_stage_readkv_separation} follows the
causal order of the reader and establishes three subclaims. First, the two
key probes recover the exact logits and identify the high-attention token.
Second, $p^2\Delta_8>\rho^2\Delta_1$ forces the value schedule $(8,0)$ or
$(0,8)$, which gives the pointwise bound
$m^{-1}\|o-\widehat o_{\ReadKV}\|_2^2\leq(p\epsilon_8+\rho)^2$. Third, a
query-independent transcript contains at most $r$ bits about the pair of
value vectors. Averaging the two query signs produces a $2\times2$
coefficient matrix with smallest eigenvalue $\tanh^2(5)/2$. An
entropy--MMSE bound for the $2m$-dimensional source then gives the converse.

\begin{proof}[Proof of
Theorem~\ref{thm:finite_two_stage_readkv_separation}]
Let $\Delta_t=D(t-1)-D(t)$ be the deployed scalar-code gain at depth $t$.
Lemma~\ref{lem:deployed_codec_bounds} gives
\[
 0.63661974<\Delta_1<0.63661975,\qquad
 0.0001373218<\Delta_8<0.0001373220,
\]
and all eight gains are positive and decreasing.

The key score is positive only in the first transformed channel. One key
refinement therefore reads the first prefix bit of that channel for both
tokens, using two bit probes. Since the two nonzero key coordinates are the
depth-one centroids $\pm c_1$, these prefixes reconstruct both keys exactly.
The reconstructed logits equal the exact logits $(5\eta b,-5\eta b)$.
Consequently, the high-attention token has probability
$p=(1+e^{-10})^{-1}$ and the other has probability
$\rho=(1+e^{10})^{-1}$.

The value planner compares marginal gains weighted by $p^2$ and $\rho^2$.
The smallest weighted gain among the eight refinements of the high token is
larger than the largest weighted gain of the low token because
\[
 \frac{p^2\Delta_8}{\rho^2\Delta_1}
 =e^{20}\frac{\Delta_8}{\Delta_1}
 >
 \frac{(4.84\mathbin{\times}10^8)(0.0001373218)}
      {0.63661975}
 >104400>1.
\]
Here we used $e^{10}>22000$, hence $e^{20}>4.84\mathbin{\times}10^8$.
The greedy rule therefore assigns all eight value levels to the high token
before selecting any level of the other token. Each value level contains
one bit from each of its $m$ coordinates. The second batch consequently
uses $8m$ probes, for a total of $r=8m+2$. Storing two $m$-dimensional keys
and two $m$-dimensional values at depth eight uses $w=32m$ bits.

The zero-depth reconstruction is zero. Applying $Q_8$ coordinatewise, the
two-stage output is $\widehat o_{\ReadKV}=pQ_8(v_{j_\star})$. For every
coordinate $\ell$,
\[
\begin{aligned}
 |o_\ell-\widehat o_{\ReadKV,\ell}|
 &=|p(v_{j_\star,\ell}-Q_8(v_{j_\star,\ell}))
      +\rho v_{3-j_\star,\ell}|\\
 &\leq p\epsilon_8+\rho.
\end{aligned}
\]
Summing the squared coordinate bounds proves
\[
 \frac1m\|o-\widehat o_{\ReadKV}\|_2^2
 \leq(p\epsilon_8+\rho)^2.
\]

We now lower-bound an arbitrary query-independent access path with at most
$r=8m+2$ probes. Revealing $b$ to the entire competing protocol before any
probe can only strengthen it, so we grant this information for free. Draw
$V_1,V_2$ independently and uniformly from $[-1,1]^m$, draw $\mathsf B$ and
$\eta$ independently and uniformly from $\{-1,1\}$, and use the attention
instance in the theorem.

Let $C=(\theta,\Omega_0,\Omega_r)$ contain public state and reader-visible
random coins. Conditional on $(\mathsf B,C)$, the query-independent
branching reader is a binary decision tree of depth at most $r$. Its leaf
$\Lambda$ includes all returned bits, selected addresses, and the stopping
decision.
Encoder-local randomness may change the distribution over leaves, but there
remain at most $2^r$ leaves. With
$\mathsf V=(V_1,V_2)\in\mathbb R^{2m}$ and $O=(C,\Lambda)$,
\[
 I(\mathsf V;O\mid \mathsf B)
 =I(\mathsf V;\Lambda\mid \mathsf B,C)
 \leq H(\Lambda\mid \mathsf B,C)
 \leq r\log 2.
\]
The equality uses the independence of $C$ from $(\mathsf V,\mathsf B)$.

Conditional on $\mathsf B$, the sign $S=\eta\mathsf B$ is uniform and
independent of $(\mathsf V,O)$. Define
$a_+=(p,\rho)^\top$ and $a_-=(\rho,p)^\top$. The Bayes estimator of the
attention output given $(O,\mathsf B,\eta)$ is obtained by replacing
$\mathsf V$ with $\mathbb E[\mathsf V\mid O,\mathsf B]$. Indeed, $\eta$ is
independent of $(\mathsf V,O)$ conditional on $\mathsf B$, so conditioning
additionally on $\eta$ does not change this posterior mean. Put
$E=\mathsf V-\mathbb E[\mathsf V\mid O,\mathsf B]$. Conditional least
squares implies that every possibly randomized decoder has risk at least
the risk of this posterior-mean estimator. Averaging that lower bound over
the two values of $S$ gives
\[
 \mathbb E\!\left[\|o-\widehat o_\Pi\|_2^2\right]
 \geq
 \frac12\mathbb E\!\left[
   \|(a_+^\top\otimes I_m)E\|_2^2+
   \|(a_-^\top\otimes I_m)E\|_2^2
 \right]
 =
 \mathbb E\!\left[E^\top(M\otimes I_m)E\right],
\]
where $I_m$ is the $m\times m$ identity matrix, $\otimes$ denotes the
Kronecker product, and
\[
 M=\frac12(a_+a_+^\top+a_-a_-^\top)
 =\frac12
 \begin{pmatrix}
  p^2+\rho^2 & 2p\rho\\
  2p\rho & p^2+\rho^2
 \end{pmatrix}.
\]
The eigenvectors $(1,1)^\top$ and $(1,-1)^\top$ have eigenvalues $1/2$ and
$(p-\rho)^2/2=\tanh^2(5)/2$, respectively. Therefore, with
$\mathcal E=\mathbb E\|E\|_2^2$,
\[
 \mathbb E\!\left[\|o-\widehat o_\Pi\|_2^2\right]
 \geq\frac{\tanh^2(5)}2\,\mathcal E.
\]

It remains to lower-bound $\mathcal E$. The vector $\mathsf V$ is uniform on
$[-1,1]^{2m}$ and independent of $\mathsf B$. Hence
$I(\mathsf V;O,\mathsf B)=I(\mathsf V;O\mid\mathsf B)\leq r\log 2$.
Equation~\eqref{eq:cube_entropy_mmse}, applied with $n=2m$,
$b_{\mathrm{info}}=r$, and observation $(O,\mathsf B)$, gives
\[
 \frac{\mathcal E}{m}
 \geq\frac{2^{\,2-r/m}}{\pi e}
 =\frac{2^{-6-2/m}}{\pi e}.
\]
The Bayes-risk inequality therefore becomes
\[
 \frac1m\mathbb E\!\left[
   \|o-\widehat o_\Pi\|_2^2
 \right]
 \geq
 \frac{\tanh^2(5)}{\pi e}\,2^{-7-2/m}.
\]
The prior is supported on the stated family, so the worst-case risk is at
least this average risk.

For the numerical bounds, $e^{10}>22000$ implies
$\rho<1/22001$ and $\tanh(5)>21999/22001>0.9999$. Together with
$\pi<3.142$, $e<2.719$, and $\epsilon_8<0.0097$, these bounds give
\[
 (p\epsilon_8+\rho)^2
 <\left(0.0097+\frac1{22001}\right)^2
 <9.50\mathbin{\times}10^{-5}.
\]
Since $2^{-7-2/m}\geq2^{-9}$ for $m\geq1$, the lower bound exceeds
$2.285\mathbin{\times}10^{-4}$, and the ratio exceeds $2.40$. At $m=128$,
the ratio is bounded below by
\[
 \frac{0.9999^2\,2^{-7-2/128}}
 {(3.142)(2.719)
  \left(0.0097+1/22001\right)^2}
 >9.52.
\]
\end{proof}

\subsubsection{Full-Unit-Ball Obstruction}
\label{app:key_stage_obstruction}

The transformed-sparsity restriction in
Theorem~\ref{thm:finite_readkv_separation} cannot simply be removed for the
deployed deterministic $(8,2)$ construction. The following example shows
that this construction does not dominate query-oblivious protocols on the
full unit ball.
Consider a power-of-two dimension and a normalized Hadamard--sign transform.
Set $x=y=e_j$. The source $x$ still belongs to $\mathcal X_U$, but $Uy$
has all $d$ coordinates nonzero, so this query lies outside
$\mathcal Y_{U,k}$ when $k<d$. Every entry of $Uy$ has magnitude
$1/\sqrt d$, and the shared scale is $1/\sqrt d$. Every key weight is
therefore $1/d^2$. Because the eight gains are strictly decreasing, a
budget of $2d$ assigns the first level to every coordinate and then the
second level to every coordinate. Thus, all depths equal two.

The standardized entries of $Ux$ are in $\{-1,1\}$.
Lemma~\ref{lem:deployed_codec_bounds} gives
$1.5104175806<Q_2(1)<1.5104175807$ and
$Q_2(-1)=-Q_2(1)$. Thus, the reconstructed transformed vector is
$Q_2(1)Ux$. Orthogonality gives the inner-product estimate $Q_2(1)$, whose
squared error exceeds $0.26052610$, independently of $d$.

This obstruction is larger than the worst-case error of a simple
query-oblivious protocol. Quantize each original coordinate independently
to the four levels
\[
 -1,\quad -\frac13,\quad \frac13,\quad 1
\]
by stochastic rounding between its two neighboring levels. Denote the
quantized vector by $\widetilde X$. Each coordinate is unbiased, the
coordinate errors are independent, and adjacent levels are separated by
$2/3$. Therefore
\[
 \operatorname{Var}(\widetilde X_i-X_i)
 \leq \frac14\left(\frac23\right)^2=\frac19.
\]
The encoder stores the resulting two-bit index for every coordinate, pads
the remaining $6d$ cells arbitrarily, and the query-oblivious reader always
fetches the first $2d$ cells. For every fixed
$x,y\in\mathbb B_2^d$,
\[
\begin{aligned}
 \mathbb E\!\left[
   \bigl(y^\top\widetilde X-y^\top x\bigr)^2
 \right]
 &=\sum_{i=1}^d y_i^2
   \operatorname{Var}(\widetilde X_i-X_i) \\
 &\leq \frac19\|y\|_2^2
 \leq\frac19.
\end{aligned}
\]
Since $1/9<0.260526106590\ldots$, the deployed deterministic
\ReadKV{} $(8,2)$ instantiation cannot strictly dominate all query-oblivious
protocols on the full unit ball.

\subsection{Prefix-Allocation Results}
\label{app:allocation}

Proposition~\ref{prop:greedy} is the allocation result used by
Theorem~\ref{thm:readkv_attention_certificate}. The later continuous and
rounding results concern the idealized geometric distortion curve. They are
auxiliary comparisons and are not assumptions in any \ReadKV{} guarantee.

\subsubsection{Exact Integer Allocation}

Consider $N$ items. Item $i$ has nonnegative weight $a_i$ and distortion
$D_i(t)$ after reading a prefix of depth $t\in\{0,\ldots,W\}$. Its
refinement at level $s$ has weighted gain
$a_i[D_i(s-1)-D_i(s)]$.

The proof uses two subclaims. A depth vector is equivalent to a prefix-closed
set of refinements, and its objective equals the depth-zero objective minus
the selected gains. Moreover, among all sets of a fixed size, a
maximum-gain set can be chosen prefix-closed: whenever it contains a
refinement but omits a predecessor from the same item, replacing the
refinement by that predecessor cannot decrease total gain. The constrained
problem therefore reduces to merging nonincreasing gain lists.

The result below is an instance of the classical marginal-allocation
principle \citep{fox1966discrete,segall1976allocation,shoham1988efficient}.
That principle allocates equal-cost resources by repeatedly taking the
largest available marginal return when each item's returns decrease with
additional allocation. Proposition~\ref{prop:greedy} states this rule in the
prefix-depth notation used by \ReadKV{}.

\begin{proposition}[Exact Integer Allocation]
\label{prop:greedy}
Let $a_1,\ldots,a_N\geq0$, and let $W$ and $B$ be nonnegative integers.
Suppose refinements have unit cost and, for every item $i$, the gains
$D_i(s-1)-D_i(s)$ are nonnegative and nonincreasing in
$s=1,\ldots,W$. Repeatedly selecting the available next refinement with
largest weighted gain globally minimizes
\[
 \sum_{i=1}^N a_iD_i(t_i)
 \quad\text{over}\quad
 t_i\in\{0,\ldots,W\},
 \qquad
 \sum_{i=1}^N t_i\leq B.
\]
The procedure stops when the budget is exhausted, all refinements have been
selected, or no positive gain remains.
\end{proposition}

\begin{proof}
For each item and level, write
$g_{i,s}=a_i[D_i(s-1)-D_i(s)]$ for $s=1,\ldots,W$.
A feasible depth vector $t$ selects the prefix-closed set
$S(t)=\{(i,s):1\leq s\leq t_i\}$. Its objective telescopes as
\[
 \sum_i a_iD_i(t_i)
 =\sum_i a_iD_i(0)-\sum_{(i,s)\in S(t)}g_{i,s}.
\]
It is therefore enough to maximize selected gain over prefix-closed sets of
cardinality at most $B$.

Fix a cardinality $k$. Start from a set of $k$ refinements with maximum total
gain, without imposing prefix closure. If it contains $(i,s)$ but omits a
predecessor $(i,u)$ with $u<s$, replace $(i,s)$ by the smallest missing
predecessor. Since the gain list for item $i$ is nonincreasing, this
replacement does not reduce total gain. Repeating the operation terminates
with a prefix-closed maximum-gain set. Thus prefix closure does not reduce
the best total gain at any fixed cardinality.

Now suppose the greedy algorithm has selected depths $t_1,\ldots,t_N$. For
an item with $t_i<W$, its next available gain is $g_{i,t_i+1}$. Every deeper
unselected gain from that item is no larger. Consequently, a largest
available gain is also a largest gain among all refinements not yet selected.
Induction on the number of selections shows that after $k$ steps, greedy has
the maximum total gain among all sets of $k$ refinements, and hence among all
feasible prefix-closed sets of that size.

All gains are nonnegative. Therefore, an optimum with at most $B$
refinements either uses the full available budget or leaves only zero gains
unselected. The greedy procedure is optimal both when it exhausts the budget
and when it stops because no positive gain remains.
\end{proof}

The nonincreasing-gain and unit-cost assumptions cannot be omitted from this
guarantee. If a large gain can be hidden behind a smaller predecessor, or if
refinements have unequal costs, selecting the largest currently available
gain need not be optimal. For example, with budget two, gain lists
$(0,100)$ and $(1,1)$ make the current-gain rule obtain total gain two,
whereas the optimum is $100$. With unequal costs, a rule based only on gain
would select one cost-two refinement of gain three, although two cost-one
refinements of gain two each have larger total gain.

\subsubsection{Continuous Geometric Allocation}

The deployed reader uses the integer rule above. To compare integer and
fractional depth budgets, the following relaxation uses the idealized
geometric curve $D(t)=4^{-t}$. Define
\[
 C_W(B)=
 \min_{\substack{0\leq t_i\leq W\\ \sum_i t_i\leq B}}
 \sum_{i=1}^N a_i4^{-t_i}.
\]

The next proposition is the finite-cap form of reverse water-filling
\citep{cover2006elements}. Standard reverse water-filling equalizes the
remaining weighted distortion across active items. The finite depth
constraint clips this allocation: sufficiently small weights remain unread,
while sufficiently large weights receive the full depth $W$.

\begin{proposition}[Finite-Cap Continuous Allocation]
\label{prop:waterfill}
Let $a_1,\ldots,a_N\geq0$, $W>0$, $B\geq0$, and
$P=\{i:a_i>0\}$. An optimum assigns zero depth to $i\notin P$. If
$0<B<W|P|$, the optimal depths on $P$ are unique and satisfy
\[
 t_i^\star
 =
 \min\{W,\max\{0,\log_4(a_i/\tau)\}\},
 \qquad
 \sum_i t_i^\star=B,
\]
for some $\tau>0$, which need not be unique. If $B=0$, all depths are zero.
If $B\geq W|P|$, every positive-weight item has depth $W$.
\end{proposition}

\begin{proof}
Write $P=\{i:a_i>0\}$. A zero-weight item does not affect the objective, so
we choose the canonical optimum that assigns it depth zero. If $P$ is empty,
this already proves the proposition. If $B=0$, feasibility fixes every depth
at zero. If $B\geq W|P|$, assigning depth $W$ to every positive-weight item
minimizes each term separately. These observations cover the boundary cases
in Proposition~\ref{prop:waterfill}.

Suppose now that $0<B<W|P|$. The feasible set is compact, so an optimum
exists. Increasing an unsaturated positive-weight depth strictly decreases
the objective. Not all positive-weight items can be saturated, so the budget
constraint binds. On the coordinates in $P$, the objective is strictly
convex because
\[
  \frac{\mathrm d^2}{\mathrm dt^2}(a_i4^{-t})
  =(\ln4)^2a_i4^{-t}>0.
\]
Consequently, the positive-weight part of the optimum is unique.

The problem is convex and has a point strictly inside the box constraints.
The Karush--Kuhn--Tucker conditions are therefore sufficient. Attach a
multiplier $\lambda\geq0$ to the total budget. For an interior item,
stationarity gives
\[
  -\ln(4)a_i4^{-t_i}+\lambda=0.
\]
The budget binds, and at least one positive-weight item is unsaturated. Thus
$\lambda>0$. If $\lambda=0$, the negative derivative of that unsaturated
coordinate cannot satisfy its one-sided optimality condition.
Putting $\tau=\lambda/\ln4$ gives
$t_i=\log_4(a_i/\tau)$ on an interior item. The one-sided optimality
conditions are equally explicit:
\[
 t_i=0\ \Longrightarrow\ a_i\leq\tau,
 \qquad
 t_i=W\ \Longrightarrow\ a_i\geq4^W\tau.
\]
Together with the interior equation, these conditions give exactly the clipped
formula in Proposition~\ref{prop:waterfill}.

For $\tau>0$, let
\[
 H(\tau)=\sum_{i\in P}
 \min\{W,\max\{0,\log_4(a_i/\tau)\}\}.
\]
This function is continuous and nonincreasing. It satisfies
$H(\tau)\to W|P|$ as $\tau\downarrow0$ and $H(\tau)\to0$ as
$\tau\uparrow\infty$, so it attains the value $B$. The resulting clipped
vector satisfies all KKT conditions. It is the unique positive-weight
optimum, which proves Proposition~\ref{prop:waterfill}. The multiplier need
not be unique because every item may be clipped over an interval of thresholds.
For example, $a=(64,1)$, $W=B=1$ has the unique allocation $(1,0)$ for every
$\tau\in[1,16]$.
\end{proof}

\subsubsection{Integer Rounding for Geometric Distortion}

For nonnegative integers $W$ and $B$, define the integer optimum
\[
 I_W(B)=
 \min_{\substack{t_i\in\{0,\ldots,W\}\\ \sum_i t_i\leq B}}
 \sum_{i=1}^N a_i4^{-t_i}.
\]

To compare this optimum with its continuous relaxation, start from a
continuous optimizer $t_i^\star=\ell_i+f_i$. Since the total budget is an
integer, the fractional parts sum to an integer. A dependent rounding step
chooses exactly that many coordinates to round upward while preserving every
marginal probability $\mathbb P(T_i=\ell_i+1)=f_i$
\citep{gandhi2006dependent}. The rounded objective can then be bounded one
coordinate at a time by
$(1-3f_i/4)4^{f_i}$.

\begin{proposition}[Integer Rounding Bound]
\label{prop:integrality}
Let $W$ and $B$ be nonnegative integers, and let
$a_1,\ldots,a_N\geq0$. Then
\[
 C_W(B)\leq I_W(B)\leq\kappa_{\mathrm{int}}C_W(B),
 \qquad
 \kappa_{\mathrm{int}}
 =\max_{0\leq f\leq1}(1-3f/4)4^f
 =1.263740721\ldots<1.264.
\]
\end{proposition}

\begin{proof}
Let $P=\{i:a_i>0\}$. If $P$ is empty, if $B=0$, or if
$B\geq W|P|$, the continuous optimum already has integer depths on every
positive-weight item, and $C_W(B)=I_W(B)$. It remains to consider
$0<B<W|P|$.

Let $t^\star$ be the continuous optimum and write
$t_i^\star=\ell_i+f_i$, where $\ell_i\in\{0,\ldots,W\}$ and
$0\leq f_i<1$. The budget binds, so
$\sum_i t_i^\star=B$. Since both $B$ and $\sum_i\ell_i$ are integers,
$K=\sum_i f_i$ is also an integer. We next construct a random vector
$Z\in\{0,1\}^N$ such that
\[
  \sum_i Z_i=K \quad\hbox{with probability one},\qquad
  \mathbb E[Z_i]=f_i.
\]
Start from the vector $f$. While coordinates $i$ and $j$ are both strictly
fractional, define
\[
 \alpha=\min\{1-f_i,f_j\},\qquad
 \beta=\min\{f_i,1-f_j\}.
\]
Replace $(f_i,f_j)$ by $(f_i+\alpha,f_j-\alpha)$ with probability
$\beta/(\alpha+\beta)$ and by
$(f_i-\beta,f_j+\beta)$ with probability
$\alpha/(\alpha+\beta)$. The pair's sum is unchanged. Each coordinate
preserves its conditional expectation, and at least one becomes integral in
either outcome. Repeating the operation terminates after at most $N-1$ steps.
Because the total is the integer $K$, the final vector cannot have exactly
one fractional coordinate. Iterated expectation now gives the claimed
properties of $Z$.

The rounded depths $T_i=\ell_i+Z_i$ respect both the exact total budget and
the write cap: if $f_i>0$, then $\ell_i\leq W-1$. For each coordinate with
$a_i>0$,
\[
  \frac{\mathbb E[a_i4^{-T_i}]}{a_i4^{-t_i^\star}}
    =(1-f_i+f_i/4)4^{f_i}\leq\kappa_{\mathrm{int}}
\]
because $\mathbb P(Z_i=1)=f_i$. Summing over $i$ gives
\[
 \mathbb E\!\left[\sum_i a_i4^{-T_i}\right]
 \leq \kappa_{\mathrm{int}} C_W(B).
\]
At least one realization is no larger than its expectation, so a feasible
integer allocation has objective at most
$\kappa_{\mathrm{int}}C_W(B)$; the optimal integer allocation can only
improve it.

To evaluate the constant, differentiate
$h(f)=(1-3f/4)4^f$ on $[0,1]$. Its only interior stationary point is
$f^\star=4/3-1/\ln4$.
The derivative changes from positive to negative there, while
$h(0)=h(1)=1$, so
$\kappa_{\mathrm{int}}=h(f^\star)=1.263740721\ldots$. The reverse inequality
$C_W(B)\leq I_W(B)$ follows because every integer allocation is also
continuous-feasible.
\end{proof}

The rounding is only a comparison between the geometric continuous and
integer objectives. The implementation instead applies
Proposition~\ref{prop:greedy} directly to the measured codec curve.

\subsection{Attention-Output Error Analysis}
\label{app:task_bridge}

The worst-case formulation asks what any legal stored representation can
achieve. The allocation model addresses the narrower operational problem of
selecting refinements for the current query from our representation. The
results below connect the allocation scores to attention error without
requiring an independent query distribution. The deterministic bound retains
explicit calibration and residual-correlation factors; without controlling
those factors, the surrogate alone gives no uniform worst-case guarantee.

\subsubsection{Attention-Output Error Decomposition}

Fix one query head with query $q\in\mathbb R^{d_h}$ and $T$ cached keys and
values. We work in the transformed value coordinates used by \ReadKV{}; the
inverse value transform is orthogonal and preserves the output norm. Let
$V,\widehat V\in\mathbb R^{T\times d_h}$ contain the exact and reconstructed
transformed values as rows, and put
$E_V=V-\widehat V\in\mathbb R^{T\times d_h}$. For exact and reconstructed
keys $k_j,\widehat k_j\in\mathbb R^{d_h}$, define
$z,\widehat z\in\mathbb R^T$ by
$z_j=q^\top k_j/\sqrt{d_h}$ and
$\widehat z_j=q^\top\widehat k_j/\sqrt{d_h}$. Their attention vectors are
$\alpha=\operatorname{softmax}(z)\in[0,1]^T$ and
$\widehat\alpha=\operatorname{softmax}(\widehat z)\in[0,1]^T$.
Let $\mathbf1_T\in\mathbb R^T$ denote the all-ones vector.

The proof has two steps. Adding and subtracting
$V^\top\widehat\alpha$ separates key-induced attention error from value
reconstruction error. The softmax Jacobian then gives
$\|\alpha-\widehat\alpha\|_2\leq
\|z-\widehat z\|_2/2$, which converts the first term into a logit-error
bound.

\begin{lemma}[Attention-Output Error Decomposition]
\label{lem:attention_bridge}
For the exact output $o=V^\top\alpha\in\mathbb R^{d_h}$ and reconstructed
output $\widehat o=\widehat V^\top\widehat\alpha\in\mathbb R^{d_h}$,
\begin{equation}
 o-\widehat o
 =V^\top(\alpha-\widehat\alpha)+E_V^\top\widehat\alpha.
 \label{eq:attention_error_decomposition}
\end{equation}
Moreover, for every $c_0\in\mathbb R^{d_h}$ and
$V_{c_0}=V-\mathbf1_Tc_0^\top$,
\begin{equation}
 \|o-\widehat o\|_2
 \leq \tfrac12\|V_{c_0}\|_{2\to2}\|z-\widehat z\|_2
       +\|E_V^\top\widehat\alpha\|_2.
 \label{eq:attention_error_bound}
\end{equation}
Here $\|V_{c_0}\|_{2\to2}$ is its operator norm: its largest amplification
of a vector's Euclidean length.
\end{lemma}

\begin{proof}
The decomposition is the algebraic identity
\[
\begin{aligned}
 V^\top\alpha-\widehat V^\top\widehat\alpha
 &=V^\top(\alpha-\widehat\alpha)
   +(V-\widehat V)^\top\widehat\alpha.
\end{aligned}
\]
Because both probability vectors sum to one,
$\mathbf1_T^\top(\alpha-\widehat\alpha)=0$, and hence
$V^\top(\alpha-\widehat\alpha)=V_{c_0}^\top(\alpha-\widehat\alpha)$.

It remains to bound the change in probabilities. The softmax Jacobian at
probability vector $p$ is
$J(p)=\operatorname{diag}(p)-pp^\top$. Its absolute row sum in row $j$ is
\[
 p_j(1-p_j)+\sum_{k\ne j}p_jp_k
 =2p_j(1-p_j)\leq\frac12.
\]
The matrix is symmetric, so
$\|J(p)\|_{2\to2}\leq\|J(p)\|_{\infty\to\infty}\leq1/2$.
Integrating the Jacobian along the line segment between $z$ and
$\widehat z$ gives
\[
 \|\alpha-\widehat\alpha\|_2
 \leq\frac12\|z-\widehat z\|_2.
\]
For a fixed attention mask, the same argument applies after restricting the
two logit vectors to their unmasked coordinates. The triangle inequality and
$\|V_{c_0}^\top u\|_2\leq\|V_{c_0}\|_{2\to2}\|u\|_2$ now prove
Equation~\eqref{eq:attention_error_bound}. The constant $1/2$ is sharp: at
$p=(1/2,1/2)$, the vector $(1,-1)$ is an eigenvector of $J(p)$ with
eigenvalue $1/2$. See also the general sharp treatment of
\citet{nair2025softmax}.
\end{proof}

In Equation~\eqref{eq:attention_error_decomposition}, the first term is the
change in attention weights caused by key error. The second is value
reconstruction error weighted by the probabilities observed by the value
planner. Both terms concern the same fixed key--query pair.
Equation~\eqref{eq:attention_error_bound} is deterministic, but it does not
imply that every reduction in the surrogate improves perplexity.

\subsubsection{Diagonal Approximation}

Let $e^K_{g,j,c}$ be the key reconstruction residual in transformed channel
$c$. The exact squared logit error for query head $h$ is
\[
 \frac1{d_h}\left(\sum_c q'_{h,c}e^K_{g(h),j,c}\right)^2.
\]
Expanding the square produces diagonal terms
$(q'_{h,c})^2(e^K_{g(h),j,c})^2$ and cross terms between different channels.
We approximate the residual second moment at depth $t$ by
$(\sigma^K_{g,c})^2D(t)$ and retain the diagonal terms. Summing over cached
tokens and sibling query heads then gives
\[
 \frac{T}{d_h}\sum_{g,c}
 \left[\sum_{h\in\siblings(g)}
 (q'_{h,c}\sigma^K_{g,c})^2\right]D(t^K_{g,c}).
\]
The common factor $T/d_h$ does not affect the allocation. The bracketed
expression is the key score in Equation~\eqref{eq:key_importance}.

For values, consider the exact second term in
Equation~\eqref{eq:attention_error_decomposition}. It satisfies
\[
 \left\|\sum_j\widehat\alpha_{h,j}e^V_{g,j}\right\|_2^2
 =\sum_j\widehat\alpha_{h,j}^{\,2}\|e^V_{g,j}\|_2^2
  +2\sum_{j<k}\widehat\alpha_{h,j}\widehat\alpha_{h,k}
       \langle e^V_{g,j},e^V_{g,k}\rangle.
\]
Retaining the diagonal terms and using the physical-scale estimate
$\|e^V_{g,j}\|_2^2\approx
 [\sum_c(\sigma^V_{g,c})^2]D(t^V_{g,j})$
gives the scale-aware value score in
Equation~\eqref{eq:value_importance}. The language-model experiments instead
set the head factor to one and optimize standardized residual error, as stated
in the main text.

These expansions explain why query coordinates are squared for keys and
attention probabilities are squared for values. They also identify the
two approximation errors: the calibration curve may misestimate residual
second moments, and the cross terms may matter. Neither expansion assumes
that keys and queries are independent. A deterministic Lloyd codec also does
not make the cross terms vanish. The exact allocation theorem concerns the
resulting separable objective. The language-model experiments test the
quality obtained from that objective.

\subsubsection{Deterministic Attention-Error Bound}
\label{app:correlation_attention_bound}

The omitted terms can be represented exactly by residual Gram matrices.
Apply the shared orthogonal key transform and write
$q'=U_Kq$, $k_j'=U_Kk_j$, and
$\widehat k_j'=U_K\widehat k_j$. Orthogonality preserves the logits, so
$z_j-\widehat z_j=(q')^\top(k_j'-\widehat k_j')/\sqrt{d_h}$.
Let $E_K\in\mathbb R^{T\times d_h}$ have rows
$e_j^K=k_j'-\widehat k_j'$. Retain the transformed-value residual matrix
$E_V=V-\widehat V$ from Lemma~\ref{lem:attention_bridge}. Define the realized
diagonal energies
\[
 A_K=\frac1{d_h}\sum_{j=1}^T\sum_{c=1}^{d_h}
       (q_c'e_{j,c}^K)^2,\qquad
 A_V=\sum_{j=1}^T\widehat\alpha_j^{\,2}\|e_j^V\|_2^2.
\]
For the selected key depths $t_c^K$ and value depths $t_j^V$, define the
corresponding calibrated objectives, including the factors omitted during
allocation,
\[
 S_K=\frac{T}{d_h}\sum_{c=1}^{d_h}
       (q_c'\sigma_c^K)^2D(t_c^K),\qquad
 S_V=s\sum_{j=1}^T\widehat\alpha_j^{\,2}D(t_j^V).
\]
Here $s=\sum_c(\sigma_c^V)^2$ gives the physical-scale objective, while
$s=1$ gives the standardized objective used in the language-model
experiments. Define the calibration ratios by
\[
 \gamma_K=
 \begin{cases}
  A_K/S_K,&S_K>0,\\
  0,&A_K=S_K=0,\\
  +\infty,&S_K=0<A_K,
 \end{cases}
 \qquad
 \gamma_V=
 \begin{cases}
  A_V/S_V,&S_V>0,\\
  0,&A_V=S_V=0,\\
  +\infty,&S_V=0<A_V.
 \end{cases}
\]
Thus $A_K=\gamma_KS_K$ and $A_V=\gamma_VS_V$ whenever the ratios are finite,
with the convention $0\cdot0=0$. These factors isolate calibration mismatch
from residual correlation. Any finite upper bounds on the canonical ratios
may be substituted in Proposition~\ref{prop:correlation_attention_bound}.

Put $G_K=E_K^\top E_K$. On the coordinates with $(G_K)_{cc}>0$, define the
normalized Gram matrix
\[
 (C_K)_{cc'}=
 \frac{(G_K)_{cc'}}
 {\sqrt{(G_K)_{cc}(G_K)_{c'c'}}},
 \qquad \kappa_K=\lambda_{\max}(C_K).
\]
If $E_K=0$, set $C_K=[1]$ and $\kappa_K=1$ by convention. Similarly, put
$G_V=E_VE_V^\top$. On the rows with $(G_V)_{jj}>0$, define
\[
 (C_V)_{jk}
 =
 \frac{(G_V)_{jk}}
 {\sqrt{(G_V)_{jj}(G_V)_{kk}}},
 \qquad
 \kappa_V=\lambda_{\max}(C_V).
\]
If $E_V=0$, set $C_V=[1]$ and $\kappa_V=1$ by convention.

The proof establishes three subclaims. Normalizing the key residual Gram
matrix gives $\|z-\widehat z\|_2^2\leq\kappa_KA_K$. Normalizing the value
residual Gram matrix gives
$\|E_V^\top\widehat\alpha\|_2^2\leq\kappa_VA_V$. The calibration ratios then
replace $A_K,A_V$ by $S_K,S_V$, and
Lemma~\ref{lem:attention_bridge} combines the two bounds.

\begin{proposition}[Correlation-Aware Attention Bound]
\label{prop:correlation_attention_bound}
For every fixed query, cache, selected depth schedule, and
$c_0\in\mathbb R^{d_h}$ for which $\gamma_K$ and $\gamma_V$ are finite,
\[
 \|o-\widehat o\|_2^2
 \leq
 \frac12\|V-\mathbf1_Tc_0^\top\|_{2\to2}^2
          \kappa_K\gamma_KS_K
 +2\kappa_V\gamma_VS_V.
\]
The statement is deterministic. It permits arbitrary dependence among the
query, keys, values, and residuals.
\end{proposition}

\begin{proof}
Coordinates outside the support used to define $C_K$ are zero columns of
$E_K$ and contribute nothing. On the remaining coordinates, set
$u_c=q_c'\sqrt{(G_K)_{cc}}$. The normalized Gram definition gives
\[
 \frac{\|u\|_2^2}{d_h}=A_K
 \quad\text{and}\quad
 \|z-\widehat z\|_2^2
 =\frac1{d_h}(q')^\top G_Kq'
 =\frac1{d_h}u^\top C_Ku
 \leq\kappa_K A_K
 \leq\kappa_K\gamma_KS_K.
\]
For values, set $v_j=\widehat\alpha_j\sqrt{(G_V)_{jj}}$ on the nonzero
residual rows. The same argument yields
\[
 \|v\|_2^2=A_V
 \quad\text{and}\quad
 \|E_V^\top\widehat\alpha\|_2^2
 =\widehat\alpha^\top G_V\widehat\alpha
 \leq\kappa_V A_V
 \leq\kappa_V\gamma_VS_V.
\]
Apply Equation~\eqref{eq:attention_error_bound} and then
$(a+b)^2\leq2a^2+2b^2$. The coefficient of the key term is
$2(1/2)^2=1/2$, which proves the claim.
\end{proof}

For a physical KV head shared by several query heads, apply the proposition to
each query head and sum the resulting inequalities. Taking the largest key
coefficient and value coefficient across the sibling heads factors out the
sums $\sum_hS_{K,h}$ and $\sum_hS_{V,h}$. These are exactly the grouped
objectives produced by Equations~\eqref{eq:key_importance}
and~\eqref{eq:value_importance}.

The Gram factors admit a concrete bound when pairwise residual correlations
are controlled.

\begin{corollary}[Calibrated Coherence Bound]
\label{cor:calibrated_coherence}
Suppose $C_K$ has size $m_K$ and every off-diagonal entry has magnitude at
most $\mu_K$. Define $m_V$ and $\mu_V$ analogously for $C_V$. Suppose also
that the realized residuals satisfy
\[
 \sum_{j=1}^T(e_{j,c}^K)^2
 \leq T(\sigma_c^K)^2D(t_c^K)
 \quad\text{for every }c,
 \qquad
 \|e_j^V\|_2^2\leq sD(t_j^V)
 \quad\text{for every }j.
\]
Then, for every center $c_0\in\mathbb R^{d_h}$,
\[
\begin{aligned}
 \|o-\widehat o\|_2^2
 &\leq
 \frac12\|V-\mathbf1_Tc_0^\top\|_{2\to2}^2
 [1+(m_K-1)\mu_K]S_K\\
 &\quad+2[1+(m_V-1)\mu_V]S_V.
\end{aligned}
\]
\end{corollary}

\begin{proof}
Each normalized Gram matrix has unit diagonal. Gershgorin's theorem gives
$\kappa_K\leq1+(m_K-1)\mu_K$ and
$\kappa_V\leq1+(m_V-1)\mu_V$.
The two residual envelopes imply $A_K\leq S_K$ and $A_V\leq S_V$. Hence the
canonical ratios satisfy $\gamma_K\leq1$ and $\gamma_V\leq1$. Substitute
these inequalities and the coherence bounds into
Proposition~\ref{prop:correlation_attention_bound}.
\end{proof}

The dependence on residual correlation cannot be removed. For any
$m\geq1$, take one key-residual row with $m$ unit entries and
$q'=\mathbf1_m/\sqrt m$, with $d_h=m$. Then
$\|z-\widehat z\|_2^2=1$ but $A_K=1/m$. Likewise, take $m$ identical
unit-norm value-residual rows and
$\widehat\alpha=\mathbf1_m/m$. Their weighted sum has squared norm one, while
$A_V=1/m$. Thus both exact-to-diagonal ratios can equal $m$, so no
dimension-independent constant can replace the correlation factors.

Proposition~\ref{prop:correlation_attention_bound} provides a valid
deterministic certificate from the allocation objectives to attention-output
error. Its strength depends on explicit, auditable quantities. The
$\gamma$ factors measure how well calibration predicts realized diagonal
error, while the $\kappa$ factors measure the amplification from omitted
cross terms. The proposition does not assert that these factors are uniformly
small for every language-model cache.

\subsection{Stagewise Allocation and Attention-Error Certificate}
\label{app:allocation_attention_guarantee}

We now analyze Algorithm~\ref{alg:readkv} on an arbitrary fixed cache and
query. Fix positive integers $T,d_h,G,H$ and the query-head-to-KV-head map
$g:\{1,\ldots,H\}\to\{1,\ldots,G\}$ from Section~\ref{sec:readkv}. The result
has two parts. Each stage chooses the best legal prefix depths for its
calibrated objective. Those optimized objectives then control the resulting
attention-output error through explicit factors for calibration mismatch and
residual correlation.

A \emph{depth schedule} assigns an integer prefix depth to every allocation
item. The legal key schedules use depths $t^K_{g,c}$ for physical head $g$
and transformed channel $c$. The legal value schedules use depths
$t^V_{g,j}$ for physical head $g$ and cached token $j$. Their feasible sets
are
\[
 \resizebox{\linewidth}{!}{$\displaystyle \mathcal T_K(B_K)=\left\{t^K\in\{0,\ldots,W\}^{G\times d_h}:\sum_{g,c}t^K_{g,c}\leq B_K\right\},\qquad \mathcal T_V(B_V)=\left\{t^V\in\{0,\ldots,W\}^{G\times T}:\sum_{g,j}t^V_{g,j}\leq B_V\right\}.$}
\]
Using the importance scores from
Equations~\eqref{eq:key_importance} and~\eqref{eq:value_importance}, define
the calibrated objectives
\[
 S_K(t^K)=\frac{T}{d_h}\sum_{g,c}a^K_{g,c}D(t^K_{g,c}),
 \qquad
 S_V(t^V;\widehat\alpha)
 =\sum_{g,j}a^V_{g,j}(\widehat\alpha)D(t^V_{g,j}).
\]
$S_K$ is the diagonal approximation to squared logit error, aggregated over
the query heads that share each physical KV head. $S_V$ is the corresponding
calibrated value objective: it uses physical value scales when
$s_g=\sum_c(\sigma^V_{g,c})^2$, and standardized units when $s_g=1$.
The calibration ratios below account for either choice. The factor
$T/d_h$ is constant across key schedules and therefore does not change the
key allocation.

For later use, separate these grouped objectives by query head:
\[
 S_{K,h}(t^K)
 =\frac{T}{d_h}\sum_c
   (q'_{h,c}\sigma^K_{g(h),c})^2D(t^K_{g(h),c}),
 \qquad
 S_{V,h}(t^V;\widehat\alpha)
 =s_{g(h)}\sum_j\widehat\alpha_{h,j}^{\,2}
   D(t^V_{g(h),j}).
\]
The definitions of the grouped importance scores imply
$S_K(t^K)=\sum_hS_{K,h}(t^K)$ and
$S_V(t^V;\widehat\alpha)=\sum_hS_{V,h}(t^V;\widehat\alpha)$.

The exact attention error can differ from these diagonal objectives for two
reasons. First, calibration may mispredict the realized diagonal residual
energy. For each query head $h$, $\gamma_{K,h}$ and $\gamma_{V,h}$ are the
ratios between that realized energy and its calibrated prediction. Second,
residuals from different key channels or value tokens may be correlated. A
\emph{residual Gram matrix} records their pairwise inner products.
The factors $\kappa_{K,h}$ and $\kappa_{V,h}$ are the largest eigenvalues of
the corresponding normalized Gram matrices. They measure the amplification
caused by the omitted cross terms.
Appendix~\ref{app:correlation_attention_bound} gives the exact formulas,
including the zero-residual conventions.

For query head $h$, let $V_{g(h)}\in\mathbb R^{T\times d_h}$ be the matrix
whose rows are the full-precision transformed values in its physical KV head,
and let $\mathbf1_T\in\mathbb R^T$ be the all-ones vector. For any transformed
center $c_h\in\mathbb R^{d_h}$, subtracting $c_h$ from every row does not
change the output error caused by a change in attention probabilities,
because both probability vectors sum to one. The notation
$\|\cdot\|_{2\to2}$ denotes the operator norm, or the largest Euclidean
amplification of a matrix.

The theorem is proved through three subclaims. First, diminishing refinement
gains make the marginal-gain rule minimize $S_K$ over every legal key
schedule. Second, once that schedule fixes
$\widehat\alpha_{\ReadKV}$, the same argument minimizes
$S_V(\,\cdot\,;\widehat\alpha_{\ReadKV})$ over every legal value schedule.
Third, Proposition~\ref{prop:correlation_attention_bound} bounds each
$\|o_h-\widehat o_h\|_2^2$ by a linear combination of
$S_{K,h}$ and $S_{V,h}$. Summing those bounds gives the final certificate.
Its coefficients account for calibration mismatch, residual correlation,
and the centered-value operator norm.

Lemma~\ref{lem:deployed_codec_bounds} in Appendix~\ref{app:codec} verifies
that the deployed depth-eight Lloyd tree has decreasing positive gains
$\Delta_1>\cdots>\Delta_8>0$, where
$\Delta_t=D(t-1)-D(t)$. Hence, the diminishing-gain hypothesis and the exact
stagewise allocation conclusions of
Theorem~\ref{thm:readkv_attention_certificate} apply to the codec used in the
experiments.

The formal statement is
Theorem~\ref{thm:readkv_attention_certificate} in
Section~\ref{sec:allocation_attention}. We prove it below.

Applied headwise, Corollary~\ref{cor:calibrated_coherence} gives a directly
checkable sufficient condition for finite conversion factors. It bounds the
calibration mismatch coordinatewise and the residual interactions through
pairwise coherence.

The value-stage optimum is conditional on the reconstructed attention from
the key stage; it is not a joint optimum of the true attention error. The
conversion factors depend on the selected reconstruction, so this is an
a posteriori certificate rather than an approximation ratio to the best
attention-error schedule.
Computing the canonical factors requires the full-precision cache as well
as the reconstruction; the certificate is an offline check, not an estimate
available from the fetched prefixes alone.

\begin{proof}[Proof of Theorem~\ref{thm:readkv_attention_certificate}]
Within either stage, item $i$ has importance $a_i\geq0$, depth $t_i$, and
weighted marginal gain $a_i[D(t_i)-D(t_i+1)]$. The assumed nonnegative,
nonincreasing gains are exactly the hypotheses of
Proposition~\ref{prop:greedy}. The legal key schedules in
$\mathcal T_K(B_K)$ are the feasible prefix allocations in that proposition.
Algorithm~\ref{alg:readkv} uses the weights $a^K_{g,c}$, while $S_K$
multiplies their objective by the positive constant $T/d_h$. This constant
does not change the optimizer. Therefore,
\[
 S_K(t^K_{\ReadKV})
 =\min_{t^K\in\mathcal T_K(B_K)}S_K(t^K).
\]
This key schedule fixes $\widehat\alpha_{\ReadKV}$ and hence every value
weight. Applying Proposition~\ref{prop:greedy} again, now with those fixed
weights, gives
\[
 S_V(t^V_{\ReadKV};\widehat\alpha_{\ReadKV})
 =\min_{t^V\in\mathcal T_V(B_V)}
   S_V(t^V;\widehat\alpha_{\ReadKV}).
\]

It remains to connect these two optimized quantities to attention error.
Proposition~\ref{prop:correlation_attention_bound} gives
\[
 \|o_h-\widehat o_h\|_2^2
 \leq
 \beta_{K,h}S_{K,h}(t^K_{\ReadKV})
 +\beta_{V,h}S_{V,h}(t^V_{\ReadKV};
                     \widehat\alpha_{\ReadKV}),
\]
where
\[
 \beta_{K,h}
 =\tfrac12\|V_{g(h)}-\mathbf1_Tc_h^\top\|_{2\to2}^2
   \kappa_{K,h}\gamma_{K,h},
 \qquad
 \beta_{V,h}=2\kappa_{V,h}\gamma_{V,h}.
\]
The per-head objectives defined above sum to the grouped objectives because
the grouped importance scores sum over sibling query heads. Explicitly,
\[
 \sum_h S_{K,h}(t^K_{\ReadKV})
 =\frac{T}{d_h}\sum_{g,c}
   \left[\sum_{h\in\siblings(g)}
   (q'_{h,c}\sigma^K_{g,c})^2\right]D(t^K_{g,c})
 =S_K(t^K_{\ReadKV}),
\]
and
\[
 \sum_h S_{V,h}(t^V_{\ReadKV};\widehat\alpha_{\ReadKV})
 =\sum_{g,j}s_g
   \left[\sum_{h\in\siblings(g)}
   \widehat\alpha_{h,j}^{\,2}\right]D(t^V_{g,j})
 =S_V(t^V_{\ReadKV};\widehat\alpha_{\ReadKV}).
\]
Summing the per-head bounds and replacing each coefficient by its maximum
over $h$ proves
\[
 \sum_h\|o_h-\widehat o_h\|_2^2
 \leq
 \overline\beta_K S_K(t^K_{\ReadKV})
 +\overline\beta_V
  S_V(t^V_{\ReadKV};\widehat\alpha_{\ReadKV}).
\]
Substituting the two optimality equalities proves the displayed deterministic
bound in the theorem.
\end{proof}

\section{Experimental Details and Results}
\label{app:quality}

Section~\ref{sec:quality} gives the principal C4 and LongBench tables and
compares selected operating points in Figure~\ref{fig:competitor-comparison}.
This appendix specifies the protocol and comparison methods, then adds C4
controls, shorter-prompt QA, and passkey retrieval. The remaining experiments
test current-query information, intermediate precision levels, surrogate
quality, and the restrictions used by the GPU reader.

\subsection{Evaluation Protocol}
\label{app:quality_protocol}
\noindent\textbf{Language Modeling.}
For each model, the C4 suite contains 32 calibration documents, eight
development documents, and 32 disjoint test documents. A fixed rule selects
one 2,048-token window from each document. The principal seed-dependent
\ReadKV{} comparisons use transform seeds 20260801--20260803. Every
comparison retains its matched \textsc{Dense} reference and document
manifest.

Model weights are bf16. Both compressed and \textsc{Dense} conditions use
float32 attention, softmax, and final log-softmax. Queries and keys are
captured after rotary position encoding and use one shared orthogonal
transform. Values and reconstructed outputs use a second transform.
The compressed cache is used throughout causal evaluation. \ReadKV{}
selects its prefixes anew for each query. Compression errors can therefore
affect subsequent model states.

\noindent\textbf{Question Answering and Retrieval.}
The LongBench evaluation uses Qasper, HotpotQA, and 2WikiMultihopQA
\citep{bai2024longbench}, with 100 questions per task and instruction model.
The same questions are evaluated at prompt limits of 3,500 and 7,500 tokens,
giving paired short- and long-context conditions. We use greedy decoding,
allow 128 output tokens for Qasper and 32 for the other tasks, and score the
first answer line by normalized token F1, taking the maximum over reference
answers.

The main protocol uses \textsc{Dense} prompt prefill, which determines the
first answer token, then applies each method's cache representation and
read rule to subsequent decoding.

Passkey retrieval \citep{mohtashami2023landmark} uses 200 five-digit keys at
each context length, balanced across five insertion positions, with a
24-token output cap. A trial succeeds when the generated answer contains the
key. It uses the same \textsc{Dense}-prefill boundary as question answering.

\noindent\textbf{Hardware and Calibration Provenance.}
Reference quality runs use NVIDIA L4 GPUs, while Mistral instruction-model
runs and later controls use NVIDIA A10G GPUs. Each comparison retains its
original hardware and calibration cohort. In the Qwen generation study, the
\ReadKV{} $(4,2)$, \textsc{Full Reader} $(4,4)$, Lloyd controls, and new
protected \ReadKV{} rows are NVIDIA A10G follow-ups that use the NVIDIA L4 cohort's
\textsc{Dense} scores on the same questions. The protected Qwen rows also
use a different calibration artifact and are compared with the existing NVIDIA L4
TurboQuant references. The joins verify question identifiers, prompt lengths
and truncation, gold answers, and generation and scoring settings; they do not establish hardware
or calibration equivalence. These cross-cohort comparisons describe operating
points but do not isolate a single algorithmic change. The protected Mistral
rows and their references use an NVIDIA A10G and the same calibration artifact.
 All latency and DRAM measurements use an NVIDIA A10G.

\noindent\textbf{Aggregation and Uncertainty.}
For language modeling, perplexity is the exponential of token-weighted
negative log likelihood (NLL). Tables average perplexity across seeds, while
$\Delta\mathrm{PPL}$ uses unrounded means and matched \textsc{Dense}
scores. Confidence intervals are computed for NLL differences and do not
represent uncertainty in the arithmetic mean of perplexity.

Bootstrap intervals resample paired documents or task examples rather than
individual tokens, layers, or heads. Reported intervals use at least 10,000
replicates, and the depth and factorial controls use 20,000. Passkey
resampling is stratified by insertion position. Intervals from a single
transform seed capture document variation but not variation across transforms
or model families.

\noindent\textbf{Budget Realization.}
\label{app:rate_accounting}
Whenever a quality method permits a configurable K/V split, we select its
integer stage rates on development NLL using the primary transform seed and
freeze them before testing. Test data and the remaining transform seeds do not
retune the split.
Requested rates are converted to nonnegative integer increment budgets over
the $Gd_h$ key items and $GT$ value items. Ties are resolved in a fixed order
while preserving prefix feasibility. Every reported schedule is checked
against its requested and realized budgets.

\subsection{Baselines and Controls}
\label{app:competitor_details}
Section~\ref{sec:quality} introduces the principal comparison methods. This
section defines the appendix-only controls and the implementation choices
that affect their quality and resource costs.

\noindent\textbf{Fixed-Depth Controls.}
The \emph{Lloyd-2 (norm-corrected)} control tests direct two-bit quantization
without a progressive code. It uses the same transforms and per-channel
standardization as \ReadKV{}, but maps each scalar to a non-nested
four-level Gaussian Lloyd--Max codebook \citep{lloyd1982quantization}. It
stores and reads two payload bits per scalar. The encoder also stores each
key norm in fp16 and uses it during reconstruction. Values are not
norm-corrected. Resource costs include these norms, the calibration arrays,
and the shared codebooks.

\noindent\textbf{KIVI.}
KIVI \citep{liu2024kivi} quantizes keys per channel over groups of 32 tokens
and values per token over groups of 32 channels. Each group stores an fp16
scale and offset. Both the two-bit configuration and the appendix-only
four-bit row retain a 128-token uncompressed residual buffer. The reported
costs include the buffer and all quantization parameters.

\noindent\textbf{TurboQuant.}
TurboQuant K4/V4 \citep{zandieh2026turboquant} uses its supported
norm-corrected four-bit configuration. Keys are normalized, Hadamard-rotated,
and quantized to 16 Lloyd--Max centroids. Values use per-token asymmetric
uniform quantization. Resource costs include the retained key norms and the
value scales and offsets. The serving configuration leaves the first and
last two layers uncompressed, and the protected \ReadKV{} C4 and QA rows use
the same exemptions. The displayed $(W,R)$ rates describe compressed layers;
the reported resource costs include all dense boundary layers.

\noindent\textbf{Sparse Selection.}
SparQ \citep{ribar2024sparq} retains a 16-bit cache. It uses $r$
query-selected key channels to estimate token relevance, then reads the
complete keys and values of at most $k$ past positions. The current token is
always included, giving at most $k+1$ positions, and a running value mean
approximates omitted attention mass. Logical-read costs include the probe
channels, selected tokens, and mean state.

The displayed SparQ $r32/k128$ configuration uses $r=32$ and $k=128$.
The C4 row labeled \emph{SparQ ($R=2$ target)} uses $r=16,k=64$, except on
Yi, where development selection chooses $r=8,k=64$. The
\emph{SparQ + nested code} control applies sparse token selection to the
eight-bit progressive code. It reads depth-two prefixes on the $r=32$
scoring channels, followed by depth-eight keys and values for at most
$k=128$ selected tokens.

For C4, Qwen3B, Qwen14B, Yi, and DeepSeek select configurations on eight
development documents from $r\in\{8,16,32\}$,
$k\in\{32,64,128,256\}$, and a fixed set of nine sparse-progressive
variants. The lowest-NLL candidate within each read budget is fixed before
testing. Qwen7B uses the setting selected in its original development run.
Mistral uses a 12-candidate pure-SparQ sweep over the same $r,k$ grid
on eight development documents, frozen before testing. It selects $r16/k64$
and $r32/k128$ for the two read targets, reproducing the earlier transferred
settings and held-out scores. This is a bounded development search, not a
claim of global optimality. Mistral's hybrid configuration remains transferred.
The QA and passkey studies also use fixed $r32/k128$ and hybrid
configurations. Each SparQ test uses one seed.

\subsection{Complete Quality Results}
\label{app:absolute_quality}
Tables~\ref{tab:polished-prediction} and~\ref{tab:polished-qa} in
Section~\ref{sec:quality} report all principal C4 and 7,500-token QA
operating points. The tables below extend them with C4 controls and paired
3,500-token QA. They use the same format: quality change above logical
reads and retained KV capacity, both as percentages of \textsc{Dense}.
Positive quality changes are worse.

\noindent\textbf{C4 Language-Modeling Results.}
To the principal methods in Table~\ref{tab:polished-prediction},
Table~\ref{tab:polished-prediction-grid} adds
\ReadKV{} $(8,2)^\dagger$, \emph{KIVI (4-bit K/V)},
\emph{Lloyd-2 (norm-corrected)}, \emph{SparQ ($R=2$ target)}, and
\emph{SparQ + nested code}. Together, these rows compare direct low-bit
quantization, sparse retrieval, and a sparse reader applied to the
progressive code.

\begin{table}[!t]
\centering
\caption{Additional C4 language-modeling quality and KV costs on the held-out documents used in Table~\ref{tab:polished-prediction}. $\dagger$ denotes configurations in which the first and last two layers remain uncompressed. The Qwen2.5-14B Lloyd-2 condition was not run. The low-read SparQ control uses $r16/k64$, except on Yi where it uses $r8/k64$. SparQ + nested code uses $r32/k128$.}
\label{tab:polished-prediction-grid}
\providecolor{tradeaccent}{HTML}{B34E28}
\fontsize{9}{10}\selectfont
\setlength{\tabcolsep}{2pt}
\renewcommand{\arraystretch}{1.04}
\begin{tabular}{@{}p{.225\linewidth}*{6}{>{\centering\arraybackslash}p{.118\linewidth}}@{}}
\toprule
Method $(W,R)$ & \shortstack{Qwen2.5\\3B} & \shortstack{Qwen2.5\\7B} & \shortstack{Qwen2.5\\14B} & \shortstack{Yi-1.5\\6B} & \shortstack{DeepSeek\\LLM-7B} & \shortstack{Mistral\\7B-v0.3} \\
{\fontsize{8}{9}\selectfont\textcolor{black!65}{\textsc{Dense} PPL}} & {\fontsize{8}{9}\selectfont\textcolor{black!65}{11.36}} & {\fontsize{8}{9}\selectfont\textcolor{black!65}{10.13}} & {\fontsize{8}{9}\selectfont\textcolor{black!65}{8.74}} & {\fontsize{8}{9}\selectfont\textcolor{black!65}{8.96}} & {\fontsize{8}{9}\selectfont\textcolor{black!65}{8.49}} & {\fontsize{8}{9}\selectfont\textcolor{black!65}{7.62}} \\[2pt]
\midrule
\multicolumn{7}{@{}l}{\fontsize{8}{9}\selectfont Top: $\Delta\mathrm{PPL}$ (\%)}\\
\multicolumn{7}{@{}l}{\fontsize{8}{9}\selectfont Bottom: logical reads / retained KV capacity (\% of \textsc{Dense})}\\[4pt]
\ReadKV{} $(8,2)^{\dagger}$ & \shortstack{+3.05\\[1.5pt]{\fontsize{7.8}{8}\selectfont 22.7\,/\,\textcolor{black!65}{55.6}}} & \shortstack{+1.93\\[1.5pt]{\fontsize{7.8}{8}\selectfont 25.3\,/\,\textcolor{black!65}{57.2}}} & \shortstack{+2.37\\[1.5pt]{\fontsize{7.8}{8}\selectfont 20.1\,/\,\textcolor{black!65}{54.3}}} & \shortstack{+7.82\\[1.5pt]{\fontsize{7.8}{8}\selectfont 23.8\,/\,\textcolor{black!65}{56.3}}} & \shortstack{+4.39\\[1.5pt]{\fontsize{7.8}{8}\selectfont 24.4\,/\,\textcolor{black!65}{56.8}}} & \shortstack{+8.39\\[1.5pt]{\fontsize{7.8}{8}\selectfont 23.7\,/\,\textcolor{black!65}{56.3}}} \\[3pt]
\shortstack[l]{KIVI (4-bit K/V)\\[1pt]{\fontsize{7}{7}\selectfont\hypersetup{citecolor=black!60}\citealp{liu2024kivi}}} & \shortstack{+0.07\\[1.5pt]{\fontsize{7.8}{8}\selectfont 37.6\,/\,\textcolor{black!65}{35.5}}} & \shortstack{+0.09\\[1.5pt]{\fontsize{7.8}{8}\selectfont 37.6\,/\,\textcolor{black!65}{35.5}}} & \shortstack{+0.06\\[1.5pt]{\fontsize{7.8}{8}\selectfont 37.6\,/\,\textcolor{black!65}{35.5}}} & \shortstack{+0.07\\[1.5pt]{\fontsize{7.8}{8}\selectfont 37.6\,/\,\textcolor{black!65}{35.5}}} & \shortstack{+0.10\\[1.5pt]{\fontsize{7.8}{8}\selectfont 37.6\,/\,\textcolor{black!65}{35.5}}} & \shortstack{+0.09\\[1.5pt]{\fontsize{7.8}{8}\selectfont 37.6\,/\,\textcolor{black!65}{35.5}}} \\[3pt]
\shortstack[l]{Lloyd-2\\(norm-corrected)\\[1pt]{\fontsize{7}{7}\selectfont\hypersetup{citecolor=black!60}\citealp{lloyd1982quantization}}} & \shortstack{+40.99\\[1.5pt]{\fontsize{7.8}{8}\selectfont 13.1\,/\,\textcolor{black!65}{13.0}}} & \shortstack{+33.83\\[1.5pt]{\fontsize{7.8}{8}\selectfont 13.1\,/\,\textcolor{black!65}{13.0}}} & {\fontsize{8}{9}\selectfont not run} & \shortstack{+159.37\\[1.5pt]{\fontsize{7.8}{8}\selectfont 13.1\,/\,\textcolor{black!65}{13.0}}} & \shortstack{+3766.75\\[1.5pt]{\fontsize{7.8}{8}\selectfont 13.1\,/\,\textcolor{black!65}{13.0}}} & \shortstack{+1366.07\\[1.5pt]{\fontsize{7.8}{8}\selectfont 13.1\,/\,\textcolor{black!65}{13.0}}} \\[3pt]
\shortstack[l]{SparQ ($R=2$ target)\\[1pt]{\fontsize{7}{7}\selectfont\hypersetup{citecolor=black!60}\citealp{ribar2024sparq}}} & \shortstack{+5.22\\[1.5pt]{\fontsize{7.8}{8}\selectfont 12.5\,/\,\textcolor{black!65}{100.0}}} & \shortstack{+11.57\\[1.5pt]{\fontsize{7.8}{8}\selectfont 12.5\,/\,\textcolor{black!65}{100.0}}} & \shortstack{+6.13\\[1.5pt]{\fontsize{7.8}{8}\selectfont 12.5\,/\,\textcolor{black!65}{100.0}}} & \shortstack{+291.84\\[1.5pt]{\fontsize{7.8}{8}\selectfont 9.4\,/\,\textcolor{black!65}{100.0}}} & \shortstack{+2.98\\[1.5pt]{\fontsize{7.8}{8}\selectfont 12.5\,/\,\textcolor{black!65}{100.0}}} & \shortstack{+1112.46\\[1.5pt]{\fontsize{7.8}{8}\selectfont 12.5\,/\,\textcolor{black!65}{100.0}}} \\[3pt]
\shortstack[l]{SparQ + nested code\\[1pt]{\fontsize{7}{7}\selectfont\hypersetup{citecolor=black!60}\citealp{ribar2024sparq}}} & \shortstack{+2.33\\[1.5pt]{\fontsize{7.8}{8}\selectfont 8.3\,/\,\textcolor{black!65}{50.1}}} & \shortstack{+2.48\\[1.5pt]{\fontsize{7.8}{8}\selectfont 8.1\,/\,\textcolor{black!65}{50.1}}} & \shortstack{+1.57\\[1.5pt]{\fontsize{7.8}{8}\selectfont 8.0\,/\,\textcolor{black!65}{50.1}}} & \shortstack{+5.55\\[1.5pt]{\fontsize{7.8}{8}\selectfont 8.1\,/\,\textcolor{black!65}{50.1}}} & \shortstack{+0.92\\[1.5pt]{\fontsize{7.8}{8}\selectfont 7.9\,/\,\textcolor{black!65}{50.1}}} & \shortstack{+196.78\\[1.5pt]{\fontsize{7.8}{8}\selectfont 8.0\,/\,\textcolor{black!65}{50.1}}} \\[3pt]
\bottomrule
\end{tabular}
\end{table}

\FloatBarrier

The low-read configurations vary more across model families than the $R=4$
results in the main text. On the three Qwen2.5 models, \ReadKV{} $(4,2)$
stays within 4.5\% of \textsc{Dense} perplexity; retaining eight bits at
the same read budget improves all three results. The sparse reader with the nested code uses about
8\% of \textsc{Dense} logical reads and 50\% of its retained KV capacity.
Its $\Delta\mathrm{PPL}$ is below 5.56\% on five of the six models. The
transferred Mistral hybrid setting is less accurate; it was not retuned in
the pure-SparQ sweep. Direct two-bit Lloyd quantization
and low-read SparQ also incur large losses on some models.
At the low-read target, SparQ is more accurate on DeepSeek than
\ReadKV{} $(8,2)$ (2.98\% versus 6.27\% PPL increase), while retaining
about twice the cache capacity. The four-bit KIVI control has at most 0.10\% PPL
increase across the six models, at a higher logical-read rate than
\ReadKV{} $(8,4)$.

\noindent\textbf{LongBench Question-Answering Results.}
Each of the six model--dataset conditions uses the same 100 questions
across methods and prompt limits. To test dependence on prompt length,
Table~\ref{tab:polished-qa-short} reports the paired 3,500-token condition
for the methods shown at 7,500 tokens in Table~\ref{tab:polished-qa}.

\begin{table}[!t]
\centering
\caption{LongBench question-answering quality and KV costs at a 3,500-token prompt limit. The evaluation uses the same 100 questions and dense-prefill protocol as Table~\ref{tab:polished-qa}. $\dagger$ denotes configurations in which the first and last two layers remain uncompressed. Color marks the observed quality--read frontier: at shown precision, no other method row in the column improves loss or reads without worsening the other. Reads recur per query; retained capacity is shown but not ranked.}
\label{tab:polished-qa-short}
\label{tab:unified-qa-qasper}
\label{tab:unified-qa-hotpotqa}
\label{tab:unified-qa-2wikimqa}
\providecolor{tradeaccent}{HTML}{B34E28}
\fontsize{9}{10}\selectfont
\setlength{\tabcolsep}{2pt}
\renewcommand{\arraystretch}{1.04}
\begin{tabular}{@{}p{.225\linewidth}*{6}{>{\centering\arraybackslash}p{.118\linewidth}}@{}}
\toprule
 & \multicolumn{3}{c}{Qwen2.5-7B-Instruct} & \multicolumn{3}{c}{Mistral-7B-Instruct-v0.3}\\
\cmidrule(lr){2-4}\cmidrule(lr){5-7}
Method $(W,R)$ & \shortstack{Qasper} & \shortstack{HotpotQA} & \shortstack{2Wiki\\MultihopQA} & \shortstack{Qasper} & \shortstack{HotpotQA} & \shortstack{2Wiki\\MultihopQA} \\
{\fontsize{8}{9}\selectfont\textcolor{black!65}{\textsc{Dense} F1}} & {\fontsize{8}{9}\selectfont\textcolor{black!65}{35.51}} & {\fontsize{8}{9}\selectfont\textcolor{black!65}{42.49}} & {\fontsize{8}{9}\selectfont\textcolor{black!65}{44.87}} & {\fontsize{8}{9}\selectfont\textcolor{black!65}{28.10}} & {\fontsize{8}{9}\selectfont\textcolor{black!65}{41.42}} & {\fontsize{8}{9}\selectfont\textcolor{black!65}{37.46}} \\[2pt]
\midrule
\multicolumn{7}{@{}l}{\fontsize{8}{9}\selectfont Top: $\Delta\mathrm{F1}$ (points)\quad \textcolor{tradeaccent}{Color}: quality--read frontier}\\
\multicolumn{7}{@{}l}{\fontsize{8}{9}\selectfont Bottom: logical reads / retained KV capacity (\% of \textsc{Dense})}\\[4pt]
\shortstack[l]{KIVI (2-bit K/V)\\[1pt]{\fontsize{7}{7}\selectfont\hypersetup{citecolor=black!60}\citealp{liu2024kivi}}} & \textcolor{tradeaccent}{\shortstack{+0.51\\[1.5pt]{\fontsize{7.8}{8}\selectfont 21.2\,/\,20.7}}} & \shortstack{+2.59\\[1.5pt]{\fontsize{7.8}{8}\selectfont 21.2\,/\,\textcolor{black!65}{20.7}}} & \textcolor{tradeaccent}{\shortstack{+0.90\\[1.5pt]{\fontsize{7.8}{8}\selectfont 21.2\,/\,20.7}}} & \shortstack{-0.39\\[1.5pt]{\fontsize{7.8}{8}\selectfont 21.0\,/\,\textcolor{black!65}{20.7}}} & \shortstack{-0.17\\[1.5pt]{\fontsize{7.8}{8}\selectfont 20.9\,/\,\textcolor{black!65}{20.7}}} & \textcolor{tradeaccent}{\shortstack{+0.32\\[1.5pt]{\fontsize{7.8}{8}\selectfont 20.9\,/\,20.7}}} \\[3pt]
\shortstack[l]{TurboQuant K4/V4$^{\dagger}$\\[1pt]{\fontsize{7}{7}\selectfont\hypersetup{citecolor=black!60}\citealp{zandieh2026turboquant}}} & \shortstack{+1.14\\[1.5pt]{\fontsize{7.8}{8}\selectfont 36.7\,/\,\textcolor{black!65}{36.7}}} & \shortstack{+1.81\\[1.5pt]{\fontsize{7.8}{8}\selectfont 36.7\,/\,\textcolor{black!65}{36.7}}} & \shortstack{+1.00\\[1.5pt]{\fontsize{7.8}{8}\selectfont 36.7\,/\,\textcolor{black!65}{36.7}}} & \shortstack{+1.01\\[1.5pt]{\fontsize{7.8}{8}\selectfont 35.4\,/\,\textcolor{black!65}{35.4}}} & \shortstack{+0.17\\[1.5pt]{\fontsize{7.8}{8}\selectfont 35.4\,/\,\textcolor{black!65}{35.4}}} & \shortstack{+0.87\\[1.5pt]{\fontsize{7.8}{8}\selectfont 35.4\,/\,\textcolor{black!65}{35.4}}} \\[3pt]
\shortstack[l]{SparQ $r32/k128$\\[1pt]{\fontsize{7}{7}\selectfont\hypersetup{citecolor=black!60}\citealp{ribar2024sparq}}} & \textcolor{tradeaccent}{\shortstack{+0.52\\[1.5pt]{\fontsize{7.8}{8}\selectfont 16.4\,/\,100.0}}} & \shortstack{+1.38\\[1.5pt]{\fontsize{7.8}{8}\selectfont 16.2\,/\,\textcolor{black!65}{100.0}}} & \textcolor{tradeaccent}{\shortstack{+1.36\\[1.5pt]{\fontsize{7.8}{8}\selectfont 16.2\,/\,100.0}}} & \shortstack{+0.47\\[1.5pt]{\fontsize{7.8}{8}\selectfont 16.3\,/\,\textcolor{black!65}{100.0}}} & \shortstack{+0.21\\[1.5pt]{\fontsize{7.8}{8}\selectfont 16.3\,/\,\textcolor{black!65}{100.0}}} & \textcolor{tradeaccent}{\shortstack{+1.03\\[1.5pt]{\fontsize{7.8}{8}\selectfont 16.2\,/\,100.0}}} \\[3pt]
\addlinespace[3pt]
\ReadKV{} $(4,2)$ & \shortstack{+2.51\\[1.5pt]{\fontsize{7.8}{8}\selectfont 12.6\,/\,\textcolor{black!65}{25.1}}} & \shortstack{+0.50\\[1.5pt]{\fontsize{7.8}{8}\selectfont 12.6\,/\,\textcolor{black!65}{25.1}}} & \shortstack{+2.05\\[1.5pt]{\fontsize{7.8}{8}\selectfont 12.6\,/\,\textcolor{black!65}{25.1}}} & \shortstack{-1.35\\[1.5pt]{\fontsize{7.8}{8}\selectfont 12.6\,/\,\textcolor{black!65}{25.1}}} & \textcolor{tradeaccent}{\shortstack{-0.70\\[1.5pt]{\fontsize{7.8}{8}\selectfont 12.6\,/\,25.1}}} & \shortstack{+3.10\\[1.5pt]{\fontsize{7.8}{8}\selectfont 12.6\,/\,\textcolor{black!65}{25.1}}} \\[3pt]
\ReadKV{} $(8,2)$ & \textcolor{tradeaccent}{\shortstack{+1.30\\[1.5pt]{\fontsize{7.8}{8}\selectfont 12.6\,/\,50.1}}} & \textcolor{tradeaccent}{\shortstack{-0.08\\[1.5pt]{\fontsize{7.8}{8}\selectfont 12.6\,/\,50.1}}} & \textcolor{tradeaccent}{\shortstack{+1.42\\[1.5pt]{\fontsize{7.8}{8}\selectfont 12.6\,/\,50.1}}} & \textcolor{tradeaccent}{\shortstack{-1.65\\[1.5pt]{\fontsize{7.8}{8}\selectfont 12.6\,/\,50.1}}} & \shortstack{+0.14\\[1.5pt]{\fontsize{7.8}{8}\selectfont 12.6\,/\,\textcolor{black!65}{50.1}}} & \textcolor{tradeaccent}{\shortstack{+1.77\\[1.5pt]{\fontsize{7.8}{8}\selectfont 12.6\,/\,50.1}}} \\[3pt]
\ReadKV{} $(8,4)$ & \textcolor{tradeaccent}{\shortstack{+0.11\\[1.5pt]{\fontsize{7.8}{8}\selectfont 25.1\,/\,50.1}}} & \shortstack{+0.48\\[1.5pt]{\fontsize{7.8}{8}\selectfont 25.1\,/\,\textcolor{black!65}{50.1}}} & \textcolor{tradeaccent}{\shortstack{+0.07\\[1.5pt]{\fontsize{7.8}{8}\selectfont 25.1\,/\,50.1}}} & \shortstack{+0.36\\[1.5pt]{\fontsize{7.8}{8}\selectfont 25.1\,/\,\textcolor{black!65}{50.1}}} & \shortstack{+0.06\\[1.5pt]{\fontsize{7.8}{8}\selectfont 25.1\,/\,\textcolor{black!65}{50.1}}} & \shortstack{+1.37\\[1.5pt]{\fontsize{7.8}{8}\selectfont 25.1\,/\,\textcolor{black!65}{50.1}}} \\[3pt]
\ReadKV{} $(8,2)^{\dagger}$ & \shortstack{+2.45\\[1.5pt]{\fontsize{7.8}{8}\selectfont 25.1\,/\,\textcolor{black!65}{57.2}}} & \shortstack{+2.01\\[1.5pt]{\fontsize{7.8}{8}\selectfont 25.1\,/\,\textcolor{black!65}{57.2}}} & \shortstack{+1.65\\[1.5pt]{\fontsize{7.8}{8}\selectfont 25.1\,/\,\textcolor{black!65}{57.2}}} & \textcolor{tradeaccent}{\shortstack{-1.97\\[1.5pt]{\fontsize{7.8}{8}\selectfont 23.5\,/\,56.3}}} & \textcolor{tradeaccent}{\shortstack{-0.86\\[1.5pt]{\fontsize{7.8}{8}\selectfont 23.5\,/\,56.3}}} & \shortstack{+1.81\\[1.5pt]{\fontsize{7.8}{8}\selectfont 23.5\,/\,\textcolor{black!65}{56.3}}} \\[3pt]
\ReadKV{} $(8,4)^{\dagger}$ & \shortstack{+1.01\\[1.5pt]{\fontsize{7.8}{8}\selectfont 35.8\,/\,\textcolor{black!65}{57.2}}} & \shortstack{+0.05\\[1.5pt]{\fontsize{7.8}{8}\selectfont 35.8\,/\,\textcolor{black!65}{57.2}}} & \shortstack{+0.92\\[1.5pt]{\fontsize{7.8}{8}\selectfont 35.8\,/\,\textcolor{black!65}{57.2}}} & \shortstack{+1.10\\[1.5pt]{\fontsize{7.8}{8}\selectfont 34.5\,/\,\textcolor{black!65}{56.3}}} & \shortstack{+0.06\\[1.5pt]{\fontsize{7.8}{8}\selectfont 34.5\,/\,\textcolor{black!65}{56.3}}} & \textcolor{tradeaccent}{\shortstack{+0.14\\[1.5pt]{\fontsize{7.8}{8}\selectfont 34.5\,/\,56.3}}} \\[3pt]
\bottomrule
\end{tabular}
\end{table}

\FloatBarrier

At the 3,500-token limit, \ReadKV{} $(8,2)$ has
$\Delta\mathrm{F1}$ between $-1.65$ and $1.78$ points while using about
12.6\% of \textsc{Dense} logical reads and 50.1\% of its retained KV
capacity. At the same prompt limit, the largest loss for
\ReadKV{} $(8,4)$ is below 1.38 points.

The $\dagger$ rows cover all 24 protected \ReadKV{} conditions: two read
rates, two models, three datasets, and both prompt limits. Across these
conditions, the largest losses are below 2.46 F1 points for $(8,2)^\dagger$
and 1.11 points for $(8,4)^\dagger$. Protected $(8,4)$ scores higher than
TurboQuant in seven of twelve conditions, with a mean difference of
$-0.04$ F1 points (\ReadKV{} minus TurboQuant); every paired 95\% interval
includes zero. These intervals do not establish equivalence. Protected
$(8,2)$ reads about one third less than TurboQuant but scores lower in eight
of twelve conditions, with a mean difference of $-0.51$ F1 points; the
quality gap is larger at the 7,500-token limit. Retained capacity is
about 56--57\% of \textsc{Dense}, versus about 35--37\% for TurboQuant. The Qwen
cross-cohort qualification above applies to these comparisons.

\subsection{Passkey Retrieval}
\label{app:retrieval}
Passkey retrieval tests whether compressed decoding preserves a specific
piece of information over a long context. Several methods nearly saturate
the task, so Table~\ref{tab:polished-passkey} serves as a retention check
rather than the primary quality comparison.

\begin{table}[!ht]
\centering
\caption{Passkey retrieval quality and KV costs at 4K and 8K context lengths, with 200 trials per condition. Every method uses dense prompt prefill. $\dagger$ denotes configurations in which the first and last two layers remain uncompressed. SparQ + nested code uses $r32/k128$. \ReadKV{} $(4,2)$ recovers all keys with about half the logical reads of the matched-storage \textsc{Full Reader} $(4,4)$. Color marks the observed quality--read frontier: at shown precision, no other method row in the column improves recovered keys or reads without worsening the other (more recovered keys and fewer reads are better). Reads recur per query; retained capacity is shown but not ranked.}
\label{tab:polished-passkey}
\providecolor{tradeaccent}{HTML}{B34E28}
\fontsize{9}{10}\selectfont
\setlength{\tabcolsep}{2pt}
\renewcommand{\arraystretch}{1.04}
\begin{tabular}{@{}p{.36\linewidth}*{4}{>{\centering\arraybackslash}p{.145\linewidth}}@{}}
\toprule
 & \multicolumn{2}{c}{Qwen2.5-7B-Instruct} & \multicolumn{2}{c}{Mistral-7B-Instruct-v0.3}\\
\cmidrule(lr){2-3}\cmidrule(lr){4-5}
Method $(W,R)$ & 4K & 8K & 4K & 8K\\
{\fontsize{8}{9}\selectfont\textcolor{black!65}{\textsc{Dense} recovered}} & {\fontsize{8}{9}\selectfont\textcolor{black!65}{200}} & {\fontsize{8}{9}\selectfont\textcolor{black!65}{200}} & {\fontsize{8}{9}\selectfont\textcolor{black!65}{200}} & {\fontsize{8}{9}\selectfont\textcolor{black!65}{200}} \\[2pt]
\midrule
\multicolumn{5}{@{}l}{\fontsize{8}{9}\selectfont Top: recovered keys (out of 200)\quad \textcolor{tradeaccent}{Color}: quality--read frontier}\\
\multicolumn{5}{@{}l}{\fontsize{8}{9}\selectfont Bottom: logical reads / retained KV capacity (\% of \textsc{Dense})}\\[4pt]
\textsc{Full Reader} $(4,4)$ & \shortstack{200\\[1.5pt]{\fontsize{7.8}{8}\selectfont 25.1\,/\,\textcolor{black!65}{25.1}}} & \shortstack{200\\[1.5pt]{\fontsize{7.8}{8}\selectfont 25.1\,/\,\textcolor{black!65}{25.0}}} & \shortstack{200\\[1.5pt]{\fontsize{7.8}{8}\selectfont 25.1\,/\,\textcolor{black!65}{25.0}}} & \shortstack{200\\[1.5pt]{\fontsize{7.8}{8}\selectfont 25.0\,/\,\textcolor{black!65}{25.0}}} \\[3pt]
\ReadKV{} $(4,2)$ & \textcolor{tradeaccent}{\shortstack{200\\[1.5pt]{\fontsize{7.8}{8}\selectfont 12.6\,/\,25.1}}} & \shortstack{200\\[1.5pt]{\fontsize{7.8}{8}\selectfont 12.6\,/\,\textcolor{black!65}{25.0}}} & \textcolor{tradeaccent}{\shortstack{200\\[1.5pt]{\fontsize{7.8}{8}\selectfont 12.6\,/\,25.0}}} & \textcolor{tradeaccent}{\shortstack{200\\[1.5pt]{\fontsize{7.8}{8}\selectfont 12.5\,/\,25.0}}} \\[3pt]
\ReadKV{} $(8,2)$ & \textcolor{tradeaccent}{\shortstack{200\\[1.5pt]{\fontsize{7.8}{8}\selectfont 12.6\,/\,50.1}}} & \shortstack{200\\[1.5pt]{\fontsize{7.8}{8}\selectfont 12.6\,/\,\textcolor{black!65}{50.0}}} & \textcolor{tradeaccent}{\shortstack{200\\[1.5pt]{\fontsize{7.8}{8}\selectfont 12.6\,/\,50.0}}} & \textcolor{tradeaccent}{\shortstack{200\\[1.5pt]{\fontsize{7.8}{8}\selectfont 12.5\,/\,50.0}}} \\[3pt]
\ReadKV{} $(8,4)$ & \shortstack{200\\[1.5pt]{\fontsize{7.8}{8}\selectfont 25.1\,/\,\textcolor{black!65}{50.1}}} & \shortstack{200\\[1.5pt]{\fontsize{7.8}{8}\selectfont 25.1\,/\,\textcolor{black!65}{50.0}}} & \shortstack{200\\[1.5pt]{\fontsize{7.8}{8}\selectfont 25.1\,/\,\textcolor{black!65}{50.0}}} & \shortstack{200\\[1.5pt]{\fontsize{7.8}{8}\selectfont 25.0\,/\,\textcolor{black!65}{50.0}}} \\[3pt]
\shortstack[l]{KIVI (2-bit K/V)\\[1pt]{\fontsize{7}{7}\selectfont\hypersetup{citecolor=black!60}\citealp{liu2024kivi}}} & \shortstack{200\\[1.5pt]{\fontsize{7.8}{8}\selectfont 21.0\,/\,\textcolor{black!65}{21.3}}} & \shortstack{200\\[1.5pt]{\fontsize{7.8}{8}\selectfont 19.9\,/\,\textcolor{black!65}{20.0}}} & \shortstack{199\\[1.5pt]{\fontsize{7.8}{8}\selectfont 21.1\,/\,\textcolor{black!65}{21.3}}} & \shortstack{200\\[1.5pt]{\fontsize{7.8}{8}\selectfont 19.8\,/\,\textcolor{black!65}{20.0}}} \\[3pt]
\shortstack[l]{TurboQuant K4/V4$^{\dagger}$\\[1pt]{\fontsize{7}{7}\selectfont\hypersetup{citecolor=black!60}\citealp{zandieh2026turboquant}}} & \shortstack{200\\[1.5pt]{\fontsize{7.8}{8}\selectfont 36.7\,/\,\textcolor{black!65}{36.7}}} & \shortstack{200\\[1.5pt]{\fontsize{7.8}{8}\selectfont 36.7\,/\,\textcolor{black!65}{36.7}}} & \shortstack{200\\[1.5pt]{\fontsize{7.8}{8}\selectfont 35.4\,/\,\textcolor{black!65}{35.4}}} & \shortstack{200\\[1.5pt]{\fontsize{7.8}{8}\selectfont 35.4\,/\,\textcolor{black!65}{35.4}}} \\[3pt]
\shortstack[l]{SparQ $r32/k128$\\[1pt]{\fontsize{7}{7}\selectfont\hypersetup{citecolor=black!60}\citealp{ribar2024sparq}}} & \shortstack{200\\[1.5pt]{\fontsize{7.8}{8}\selectfont 15.8\,/\,\textcolor{black!65}{100.0}}} & \shortstack{197\\[1.5pt]{\fontsize{7.8}{8}\selectfont 14.1\,/\,\textcolor{black!65}{100.0}}} & \shortstack{200\\[1.5pt]{\fontsize{7.8}{8}\selectfont 15.8\,/\,\textcolor{black!65}{100.0}}} & \shortstack{200\\[1.5pt]{\fontsize{7.8}{8}\selectfont 14.1\,/\,\textcolor{black!65}{100.0}}} \\[3pt]
\shortstack[l]{SparQ + nested code\\[1pt]{\fontsize{7}{7}\selectfont\hypersetup{citecolor=black!60}\citealp{ribar2024sparq}}} & \textcolor{tradeaccent}{\shortstack{199\\[1.5pt]{\fontsize{7.8}{8}\selectfont 3.3\,/\,50.1}}} & \textcolor{tradeaccent}{\shortstack{200\\[1.5pt]{\fontsize{7.8}{8}\selectfont 2.4\,/\,50.0}}} & \textcolor{tradeaccent}{\shortstack{67\\[1.5pt]{\fontsize{7.8}{8}\selectfont 3.3\,/\,50.1}}} & \textcolor{tradeaccent}{\shortstack{0\\[1.5pt]{\fontsize{7.8}{8}\selectfont 2.4\,/\,50.0}}} \\[3pt]
\bottomrule
\end{tabular}
\end{table}

\FloatBarrier

The quarter-capacity \ReadKV{} $(4,2)$ reader recovers all 200 keys in
every condition while using about half the logical reads of the
matched-storage \textsc{Full Reader} $(4,4)$.
The flexible reader recovers all 200 keys in every condition with both
\ReadKV{} $(8,2)$ and \ReadKV{} $(8,4)$. The \ReadKV{} $(8,2)$ configuration
uses 12.5--12.6\% of \textsc{Dense} logical reads while retaining about
50\% of the KV capacity.
KIVI nearly saturates the task with a smaller retained cache but higher
logical reads, while SparQ retains the full cache.
The nested-code sparse hybrid is less reliable on Mistral, recovering
67/200 keys at 4K and none at 8K. Its favorable C4 results on other families
therefore do not imply reliable retrieval across models and tasks.

\subsection{Current-Query Allocation}
\label{app:placement}
To determine whether the current query improves allocation, we compare three
readers built from the same progressive codec. Uniform reading stores and
reads its depth-two prefix. Calibration-only and current-query reading
retain the same eight-bit code: the former fixes importance before testing,
while the latter allocates key depth from the query and value depth from the
reconstructed attention.

Table~\ref{tab:polished-placement} holds the codec and calibration fixed but
compares complete policies with different retained depths and
development-selected K/V splits. On Qwen and Yi, the current-query reader
uses $(R_K,R_V)=(3,1)$, while the other policies
use $(2,2)$. DeepSeek uses $(2,2)$ throughout. The table therefore measures
the full policy difference. A paired ablation below holds the split fixed to
isolate current-query key allocation.

\begin{table}[!ht]
\centering\small
\setlength{\tabcolsep}{7pt}
\caption{C4 perplexity under three allocation policies at mean payload depth
$R=2$. Entries are arithmetic means over three transform seeds on the same
32 test documents. Calibration-only and current-query allocation retain the
same $W=8$ progressive code. Uniform reading retains a depth-two code. The
policies use their development-selected K/V splits, so the table compares
complete policies rather than isolating one factor.}
\label{tab:polished-placement}\label{tab:restored-prediction-evidence}
\begin{tabular}{@{}lrrrr@{}}
\toprule
Allocation & Qwen3B & Qwen7B & Yi6B & DeepSeek7B \\
\midrule
Uniform depth 2 & 21.66 & 19.91 & 39.34 & 568.35 \\
Calibration-only & 19.11 & 14.67 & 16.31 & 1\,095.1 \\
Current query & 11.80 & 10.44 & 9.87 & 9.03 \\
\bottomrule
\end{tabular}
\end{table}

The current-query policy has the lowest perplexity on all four models in the
table. To isolate the key stage, we compare current-query and
calibration-only key allocation at the same $(R_K,R_V)=(3,1)$ split on
Qwen3B. Replacing current-query allocation increases PPL from 11.789 to
12.011, with matched \textsc{Dense} PPL 11.362. The paired NLL increase is
0.0187 nats per target, with a 95\% document-bootstrap interval of
$[0.0139,0.0247]$. Thirty-one of 32 documents favor current-query keys.
This establishes a key-stage benefit in this setting without attributing a
fixed share of the total gain to either stage.

\subsection{Intermediate Precision Levels}
\label{app:depth_support}
The progressive code offers several reconstruction depths. We test whether
these intermediate levels add value beyond selecting a subset of items for
high precision. Two-depth controls restrict keys, values, or both to
$\{0,W\}$ or $\{1,W\}$ while holding the codec, scores, calibration, seed,
and K/V split fixed. With $N$ items and lower depth $\ell$, a schedule that
assigns depth $W$ to $k$ items spends $\ell N+k(W-\ell)$ depth increments.
We therefore compare schedules at the same feasible budget. Value-only comparisons use
identical reconstructed keys. When the minimum depth and average value depth
are both one, the value schedule is uniform by construction.

Using the exact-budget schedules, the key-only $\{1,W\}$ restriction at the
\ReadKV{} $(8,2)$ operating point increases NLL over the multilevel reader by
0.227 nats per token on Qwen3B, 0.159 on Qwen7B, and 0.275 on Yi6B for seed
20260801. Across three
seeds, the Qwen3B and Qwen7B increases range from 0.227--0.243 and
0.157--0.167 nats. Restricting both stages produces substantially larger
increases of 2.408--3.598 and 3.229--3.349 nats. These comparisons show that
the benefit is not explained solely by choosing which items receive a
single high-precision reconstruction.

\subsection{Allocation Scores and Attention Error}
\label{app:surrogate_experiment}
The allocator optimizes separable scores rather than exact attention-output
error. We test how closely these scores track the errors they are intended
to reduce. Queries, keys, and values are held fixed across schedules using
captures from three layers, four C4 documents, and 24 preselected query rows
per document. The rows span three context-length ranges for Qwen3B and
Qwen7B. Each observation pairs a query row with one physical KV head,
aggregating the query heads that share it. Each schedule uses the attention
probabilities produced by its own reconstructed keys. Full-precision
probabilities serve only as diagnostic
targets. The exact decomposition in Appendix~\ref{app:task_bridge} separates
the key, value, and interaction terms for each observation.

The diagonal objectives omit interactions between reconstruction errors.
For pooled $R=2$, the median signed contribution of these interactions is
2.1\% of the exact squared logit error on Qwen3B and 1.2\% on Qwen7B.
For attention-weighted value error, the corresponding cross-token fractions
are much larger: 64.6\% and 72.6\%. The value objective can therefore guide
placement without accurately estimating the error magnitude. Relative to
the multilevel reader, restricting both stages to $\{1,W\}$ worsens output
error in 575/576 observations on Qwen3B and all 1152 on Qwen7B. The
$\{0,W\}$ restriction worsens it in 573/576 and 1149/1152 observations,
respectively.

Weighting value scores by $s_g=\sum_c(\sigma^V_{g,c})^2$ tests whether
heads with larger value magnitudes should receive more of the pooled budget.
It changes complete-model NLL by less than
0.001 nats per token in all four model-rate conditions. Three paired
intervals include zero, and the fourth difference is also below 0.001 nats.
With fixed per-head budgets, multiplying all value scores in a head by the
same positive factor should leave its allocation unchanged. Nevertheless,
repeated whole-model executions of this control differed by about 0.0024
nats per token; the source of that harness variability remains unresolved.

\subsection{GPU Reader Constraints}
\label{app:policy_quality}
The GPU reader trades allocation flexibility for a simpler execution policy.
To identify which restrictions drive its quality gap, we vary the K/V split,
pooled versus per-head budgets, minimum depth, and whether adjacent pairs of
value tokens must share a depth. A factorial study evaluates all $2^4=16$
combinations on Qwen3B and Qwen7B at $R=2$ and $R=4$.
Table~\ref{tab:policy-restrictions} follows one cumulative path through the
design.

\begin{table}[!ht]
\centering\small
\setlength{\tabcolsep}{5pt}
\caption{Effect of GPU-reader restrictions on C4 negative log likelihood.
Entries give changes in nats per predicted token.
Positive values are worse. The columns successively impose equal K/V rates,
per-head budgets, minimum depth one, and paired values. The total compares
the final restricted reader with the flexible reader. Because the restrictions
are cumulative, each intermediate increment depends on their order.}
\label{tab:policy-restrictions}
\begin{tabular}{@{}llrrrrr@{}}
\toprule
Model & $R$ & Equal K/V & Per head & Min.\ 1 & Paired & Total \\
\midrule
Qwen3B & 2 & +0.0729 & +0.0050 & +0.0012 & +0.0004 & +0.0795 \\
Qwen7B & 2 & +0.0575 & +0.0043 & +0.0061 & +0.0014 & +0.0694 \\
Qwen3B & 4 & +0.0032 & +0.0002 & +0.0004 & -0.0003 & +0.0035 \\
Qwen7B & 4 & +0.0034 & -0.0005 & +0.0002 & -0.0003 & +0.0028 \\
\bottomrule
\end{tabular}
\end{table}

Equal K/V rates account for most of the low-rate gap along the displayed
path. A mandatory minimum depth of one also forces uniform value reading
when the selected value budget is $R_V=1$. At $R=4$, all four restrictions
have much smaller effects. The
same distinction appears in Qwen passkey retrieval at 8K: the flexible
\ReadKV{} $(8,2)$ reader recovers 200 of 200 keys, while the restricted
\ReadKV{} $(8,2)$ policy recovers 100--101. The restricted
\ReadKV{} $(8,4)$ reader recovers all 200. These results identify the K/V
split as the main target for improving the low-rate GPU policy. Any revised
policy should be evaluated jointly for quality, traffic, and latency.

\FloatBarrier
\section{GPU Latency and Memory Traffic}
\label{app:gpu_measurements}

Section~\ref{sec:gpu_eval} summarizes the principal 8K, batch-one
measurement.
Here we extend that comparison across context lengths and batch sizes, then
examine the memory traffic behind the observed latency. All measurements use
an NVIDIA A10G and start from an existing cache. They cover one attention
layer during decoding, excluding cache construction and the rest of the
model.

The timed implementation uses a more constrained reader than the principal
quality experiments. It assigns budgets separately to each KV head, requires
at least one level per item, divides the budget equally between keys and
values, and groups values into two-token packets. The quality reader instead
pools its budget across heads and permits depth zero. Appendix~\ref{app:policy_quality}
measures the quality effect of these choices. The GPU results below therefore
report performance for the restricted reader, while the language-model
quality results use the flexible reader.

\subsection{Latency Across Workloads}

Table~\ref{tab:implementation-matched} reports the latency of a complete
attention-layer call at context lengths 2K, 8K, and 32K. We use batch sizes
$B=1$ and $B=4$. Every method runs on the same CUDA stream and uses the same
timing boundary. CUDA events measure repeated replay of a recorded GPU
execution graph, reducing host launch overhead. The warm-cache measurements
reuse the workload without deliberately flushing the GPU cache; they do not
assume that the entire KV cache fits on chip. The \ReadKV{} timing includes
its final conversion to bf16.

\textsc{Dense} uses an uncompressed bf16 KV cache with PyTorch scaled
dot-product attention. The TurboQuant row runs its fused K4/V4 codec
on the benchmarked layer, although the quality configuration in
Section~\ref{sec:quality} leaves this layer uncompressed. The
\textsc{Full Reader} $(4,4)$ reads all four levels of the \ReadKV{} code
without query-dependent allocation. The remaining \ReadKV{} rows use the
restricted GPU reader described above.

\begin{table}[!ht]
\centering\small
\setlength{\tabcolsep}{4.2pt}
\renewcommand{\arraystretch}{1.12}
\caption{Median single-layer decode latency (ms) on an NVIDIA A10G under CUDA
Graph replay with a warm cache. The 2K and 8K inputs use Qwen2.5-7B layer 0;
the 32K input uses synthetic K/V with the same head geometry. All rows use
the same timing boundary, including \ReadKV{}'s final bf16 conversion.}
\label{tab:implementation-matched}
\begin{tabular}{@{}lrrrrrr@{}}
\toprule
& \multicolumn{2}{c}{2K tokens} & \multicolumn{2}{c}{8K tokens} & \multicolumn{2}{c}{32K tokens}\\
\cmidrule(lr){2-3}\cmidrule(lr){4-5}\cmidrule(l){6-7}
Method & $B=1$ & $B=4$ & $B=1$ & $B=4$ & $B=1$ & $B=4$\\
\midrule
\textsc{Dense} & 0.020 & 0.048 & 0.051 & 0.146 & 0.168 & 0.585 \\
\shortstack[l]{{\footnotesize TurboQuant K4/V4 codec}\\[-1pt]{\scriptsize\citealp{zandieh2026turboquant}}} & 0.055 & 0.143 & 0.154 & 0.518 & 0.625 & 1.990 \\
\textsc{Full Reader} $(4,4)$ & 0.038 & 0.076 & 0.070 & 0.171 & 0.265 & 0.727 \\
\midrule
\ReadKV{} $(4,2)$ & 0.063 & 0.095 & 0.092 & 0.179 & 0.287 & 0.719 \\
\ReadKV{} $(8,2)$ & 0.066 & 0.096 & 0.093 & 0.183 & 0.304 & 0.740 \\
\ReadKV{} $(8,4)$ & 0.069 & 0.108 & 0.105 & 0.211 & 0.336 & 0.834 \\
\bottomrule
\end{tabular}
\end{table}

The \ReadKV{} readers are faster than the tested TurboQuant codec at 8K and
32K for both batch sizes. TurboQuant is faster at 2K and batch one, while
\textsc{Dense} is fastest throughout. A cache-eviction check at 8K and
batch one preserves this ordering: \ReadKV{} $(8,2)$ changes from 0.093 to
0.094\,ms, TurboQuant from 0.154 to 0.177\,ms, and \textsc{Dense} from
0.051 to 0.056\,ms.

These fixed-batch timings answer how quickly each implementation processes
an existing cache. Compression also reduces the memory needed to retain that
cache: it can make longer contexts or more concurrent requests fit on a GPU,
even when decoding the compressed representation adds latency. Our
measurements establish the cache-size reduction, but do not test throughput
under a memory-limited serving workload. Thus the faster dense attention
call does not eliminate the capacity benefit of TurboQuant or \ReadKV{}.

\subsection{Memory and Physical Traffic}

We next examine whether the reduction in requested payload appears in
physical memory traffic. At 8K and batch one, the four-bit and eight-bit
\ReadKV{} caches contain 4.00 and 8.00\,MiB of payload, respectively. Each
\ReadKV{} call also uses 3.87\,MiB of workspace for scores, probabilities,
schedules, and partial outputs. TurboQuant retains 4.19\,MiB and uses
0.45\,MiB of workspace, while the uncompressed K/V cache occupies
16.00\,MiB. These values describe one layer. Tables~\ref{tab:polished-prediction}
and~\ref{tab:polished-qa} report the corresponding analytical cache costs
for the complete models.

Table~\ref{tab:fair-traffic} profiles the same 8K, batch-one workload. The
\ReadKV{} $(4,4)$ scheduled control reads all four levels through the
restricted execution path. It produces the same reconstruction as the
\textsc{Full Reader} $(4,4)$ while retaining scheduling and intermediate
memory operations. Hardware counters cover the complete eager attention call,
including intermediate and output accesses, and measure physical DRAM
traffic. Latency is measured separately under graph replay. These physical
measurements differ from the logical read counts used in the quality tables.
The evicted condition deliberately displaces prior GPU-cache contents before
the measured call, checking sensitivity to cache reuse.

\begin{table}[!ht]
\centering\small
\caption{Physical DRAM traffic for the same NVIDIA A10G single-layer
benchmark used in the latency study, with 8,192 cached tokens and batch one.
Each entry reports warm / cache-evicted traffic in MiB. Counters cover the
complete eager attention call, including intermediate and output accesses.
Graph-replay latency is measured separately.}
\label{tab:fair-traffic}
\begin{tabular}{lrrr}
\toprule Method & DRAM reads & DRAM writes & Total \\
\midrule
\textsc{Dense} & 17.73 / 17.86 & 3.12 / 3.19 & 20.85 / 21.06 \\
\ReadKV{} $(4,2)$ & 2.72 / 2.71 & 4.63 / 4.71 & 7.35 / 7.42 \\
\textsc{Full Reader} $(4,4)$ & 4.70 / 4.66 & 4.19 / 4.25 & 8.90 / 8.91 \\
\ReadKV{} $(4,4)$, scheduled & 4.81 / 4.85 & 4.64 / 4.73 & 9.44 / 9.58 \\
\ReadKV{} $(8,2)$ & 2.80 / 2.74 & 4.67 / 4.73 & 7.47 / 7.47 \\
\ReadKV{} $(8,4)$ & 4.99 / 4.98 & 4.72 / 4.75 & 9.71 / 9.74 \\
\shortstack[l]{{\footnotesize TurboQuant K4/V4}\\[-1pt]{\scriptsize\citealp{zandieh2026turboquant}}} & 7.28 / 7.92 & 2.10 / 2.92 & 9.38 / 10.84 \\
\bottomrule
\end{tabular}
\end{table}

For \ReadKV{} $(8,2)$, the kernel reads 2.80/2.74\,MiB under warm/evicted
conditions, compared with 7.28/7.92\,MiB for TurboQuant. This reduces DRAM
reads by 62--65\%. After including intermediate and output writes, total
traffic falls from 9.38/10.84\,MiB to 7.47/7.47\,MiB, a reduction of
20--31\%. At the larger read budget, \ReadKV{} $(8,4)$ transfers
9.71/9.74\,MiB in total, slightly more than TurboQuant in the warm condition.
\textsc{Dense}
is faster despite transferring 20.85/21.06\,MiB, which shows that memory
traffic is only one component of latency.

The \ReadKV{} $(4,2)$ and \ReadKV{} $(8,2)$ rows further separate stored
from fetched precision. Both read two payload bits per scalar and take
0.092--0.093\,ms, although
the retained payload grows from 4 to 8\,MiB. On this fixture, the additional
stored precision has little effect on execution time, but it still consumes
memory capacity.
\end{document}